\documentclass[a4paper]{article}
\usepackage{graphicx}
\usepackage{epsfig}
\usepackage{graphicx}
\usepackage{amsmath}
\usepackage{amsthm}
\usepackage{amssymb}
\usepackage{algorithm}
\usepackage{algorithmic}
\usepackage{psfrag}
\usepackage{array}

\newcommand{\assign}{:=}
\newcommand{\tmem}[1]{{\em #1\/}}
\newcommand{\tmop}[1]{\ensuremath{\operatorname{#1}}}

\newcommand{\tmtextit}[1]{{\itshape{#1}}}
\newcommand{\mathbbm}{\mathbb}

\newcommand{\nin}{\not\in}

\newcommand{\pot}{d}

\newcommand{\restrict}{\llcorner}

\newtheorem{proposition}{Proposition}
\newtheorem{corollary}{Corollary} 

\graphicspath{{fig/}{Illustrations/}}

\begin{document}

\title{Continuous Multiclass Labeling Approaches and Algorithms}

\author{J.~Lellmann\footnotemark[2] \and C.~Schn{\"o}rr\footnotemark[2]}

\maketitle

\renewcommand{\thefootnote}{\fnsymbol{footnote}}
\footnotetext[2]{
        Image and Pattern Analysis \& HCI,
				Dept. of Mathematics and Computer Science, University of Heidelberg,
				{\tt \{lellmann,schnoerr\}@math.uni-heidelberg.de}}
\renewcommand{\thefootnote}{\arabic{footnote}}

\begin{abstract}
  We study convex relaxations of the image labeling problem on a continuous domain
  with regularizers based on metric interaction potentials. The generic
  framework ensures existence of minimizers and covers a wide
  range of relaxations of the originally combinatorial problem.  
  We focus on two specific relaxations that differ in flexibility
  and simplicity -- one can be used to tightly relax any metric
  interaction potential, while the other one only covers Euclidean
  metrics but requires less computational effort. For solving the
  nonsmooth discretized problem, we propose a globally convergent Douglas-Rachford
  scheme, and show that a sequence of dual iterates can be recovered
  in order to provide a~posteriori optimality bounds. In a
  quantitative comparison to two other first-order methods,
  the approach shows competitive performance on synthetical and real-world images.  
  By combining the method with an improved binarization technique for
  nonstandard potentials, we were able to routinely recover discrete
  solutions within $1\%$--$5\%$ of the global optimum for the
  combinatorial image labeling problem. 
\end{abstract}



\section{Problem Formulation}

The multi-class image labeling problem consists in finding, for each pixel $x$
in the image domain $\Omega \subseteq \mathbbm{R}^d$, a label $\ell (x) \in
\{1, \ldots, l\}$ which assigns one of $l$ class labels to $x$ so that the
labeling function $\ell$ adheres to some local data fidelity as well as nonlocal spatial
coherency constraints.

This problem class occurs in many applications, such as segmentation, multiview reconstruction, stitching, and inpainting \cite{Paragios2006}.
We consider the variational formulation
\begin{equation}
  \inf_{\ell : \Omega \rightarrow \{1, \ldots, l\}} f (\ell), \; f (\ell)
  \assign \underbrace{\int_{\Omega} s (x, \ell (x))
  dx}_{\text{data term}} + \underbrace{J (\ell).}_{\text{regularizer}} \label{eq:problemcomb}
\end{equation}
The {\tmem{data term}} assigns to each possible label $\ell (x)$ a {\tmem{local cost}}
$s (x, \ell (x)) \in \mathbbm{R}$, while the {\tmem{regularizer}} $J$ enforces the desired
spatial coherency. We will in particular be interested in regularizers that
penalize the weighted length of interfaces between regions of constant labeling.
Minimizing $f$ is an inherently combinatorial problem, as there is a discrete
decision to be made for each point in the image.

In the fully discrete setting, the problem can be express\-ed in terms of a Markov Random Field
\cite{Winkler2006} with a discrete state space, where the data and
regularization terms can be thought of as unary and binary potentials,
respectively. For graph-based discretizations of $J$, the resulting objective only contains
terms that depend on the labels at one or two points, and the problem can be approached
with fast graph cut-based methods. Unfortunately, this scheme introduces an anisotropy~\cite{Boykov2003} 
and thus does not represent isotropic regularizers well. Using ternary or
higher-order terms, the metrication error can be reduced,
but graph-based methods then cannot be directly applied.

However, it can be shown that even in the graph-based representation the problem is NP-hard for
relatively simple $J$ \cite{Boykov2001}, so we cannot expect to easily derive fast solvers
for this problem. This is in part caused by the discrete nature of the
feasible set. In the following, we will relax this set. This
allows to solve the problem in a globally optimal way using convex
optimization methods. On the downside, we cannot expect the relaxation to be
exact for any problem instance, i.e. we might get non-discrete (or discrete but suboptimal) solutions.

\begin{figure}
\centering
	\includegraphics[width=.80\columnwidth]{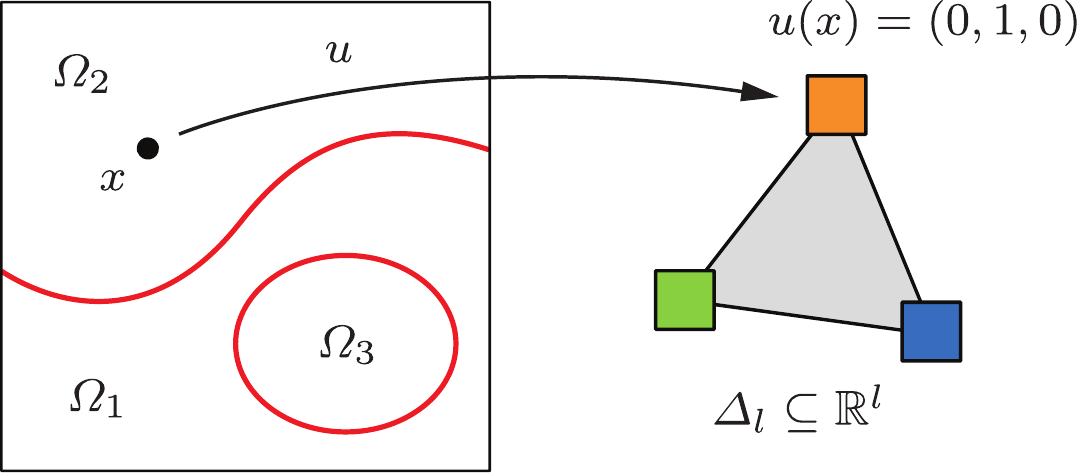}
  \caption{Convex relaxation of the multiclass labeling problem. The assignment of one
  unique label to each point in the image domain $\Omega$ is represented by a
  vector-valued function $u: \Omega \rightarrow \mathbbm{R}^l$. Ideally, $u$
  partitions the image into $l$ sets by assuming one of the unit vectors $\{e^1,\ldots,e^l\}$ everywhere.
  By relaxing this set to the standard (probability) simplex $\Delta_l$, the originally combinatorial problem
  can be treated in a convex framework.
  }
  \label{fig:relaxvis}
\end{figure}%
There are several choices for the relaxation method, of which in our opinion
the following is the most transparent (Fig.~\ref{fig:relaxvis}): Identify label $i$ with
the $i$-th unit vector $e^i \in \mathbbm{R}^l$, set $E \assign \{e^1, \ldots,
e^l \}$, and solve
\begin{equation}
  \inf_{u : \Omega \rightarrow E} f (u)\,, \; f (u)
  \assign \int_{\Omega} \langle u (x), s (x) \rangle d x + J (u)\,.
\end{equation}
The data term is now linear in $u$ and fully described by the vector
\begin{eqnarray}
  s(x) \assign \left(s_1(x),\ldots,s_l(x)\right)^{\top}
  \assign \left(s (x, 1),\ldots, s (x, l)\right)^{\top}.
\end{eqnarray}
Due to the linearization, the local costs $s$ may be arbitrarily
complicated, possibly derived from a probabilistic model, without affecting the
overall problem class. In this form, we \emph{relax} the label set by allowing $u$ to take ``intermediate''
values in the convex hull $\tmop{conv} E$ of the original label set. This is just the standard simplex $\Delta_l$,
\begin{eqnarray}
  \Delta_l & \assign & \tmop{conv} \{e^1, \ldots, e^l \} = \{a \in \mathbbm{R}^l |a \geqslant 0, \sum_{i = 1}^l a_i = 1\}\,.
\end{eqnarray}
The problem is then considered on the relaxed feasible set $\mathcal{C}$,
\begin{eqnarray}
  &  & \mathcal{C} \assign \{u \in \tmop{BV} (\Omega)^l | u (x) \in \Delta_l \text{ for a.e. } x \in \Omega\}\,.  \label{eq:defc}
\end{eqnarray}
The space of functions of bounded variation $\tmop{BV} (\Omega)^l \subset (L^1)^l$ guarantees a
minimal regularity of the discontinuities of $u$, see Sect.~\ref{sec:totalvariationbv}.
Assuming we can extend the regularizer $J$ to the whole relaxed set $\mathcal{C}$,
we get the relaxed problem
\begin{equation}
  \inf_{u \in \mathcal{C}} f (u) \,, \; f (u) \assign \int_{\Omega} \langle u (x), s (x) \rangle  dx + J(u)\,. \label{eq:problemrelaxed}
\end{equation}
If $J$ can be made convex, the overall problem is convex as well, and it may
likely be computationally tractable. In addition, $J$ should
ideally have a closed-form expression, or at least lead to a computationally tractable problem.

Whether these points are satisfied depends on the way a given regularizer is
\emph{extended} to the relaxed set. The prototypical example for such a regularizer is the \emph{total variation},
\begin{equation}
  \tmop{TV}(u) = \int_{\Omega} \| D u \|\,,
\end{equation}
where $\|\cdot\|$ denotes the Frobenius norm, in this case on $\mathbb{R}^{d \times l}$.
Note that $u$ may be discontinuous, so the gradient $D u$ has to be understood in a distributional sense (Sect.~\ref{sec:totalvariationbv}).
Much of the popularity of $\tmop{TV}$ stems from the fact that it allows to include \emph{boundary-length}
terms: The total variation of the indicator function $\chi_{\mathcal{S}}$ of a set $\mathcal{S}$,
\begin{equation}
  \tmop{Per}(\mathcal{S}) \assign \tmop{TV}(\chi_{\mathcal{S}})\,.
\label{eq:tvupers}
\end{equation}
called the \emph{perimeter} of $\mathcal{S}$, is just the classical length of the boundary $\partial \mathcal{S}$.

In this paper, we will in more generality consider ways to construct regularizers which penalize
interfaces between two adjacent regions with labels $i \neq j$ according to
the {\tmem{perimeter}} (i.e. length or area) of the interface weighted by an
{\tmem{interaction potential}} $\pot:\{1,\ldots,l\}^2 \rightarrow \mathbb{R}$ depending on the labels (in a slight abuse of notation the interaction potential is also denoted by $d$, since there is rarely any
ambiguity with respect to the ambient space dimension). The
simplest case is the \emph{Potts} model with the \emph{uniform} metric
$\pot (i, j) = 0$ iff $i = j$ and otherwise $\pot(i,j) = 1$. In this case,
the regularizer penalizes the total interface length, as seen above for the total variation.

As a prime motivation for our work, consider the two-class case $l = 2$ with $J=\tmop{TV}$. As here
the second component of $u$ is given by the first via $u_2 = 1-u_1$, we may pose the relaxed problem in the form
\begin{equation}
\min_{
  \begin{array}{c}
  {\scriptstyle u' \in \tmop{BV}(\Omega),}\\
  {\scriptstyle u'(x) \in [0,1] \text{ for a.e. } x \in \Omega}
  \end{array}
 } \int_{\Omega} u'(x) (s_1(x)-s_2(x)) d x + 2 \tmop{TV}(u')\,, \label{eq:continuouscut-binary}
\end{equation}
where $u'(x)$ is a scalar. This formulation is also referred to as \emph{continuous cut} in analogy to graph cut methods.
It can be shown \cite{Nikolova2006} that while there may be non-discrete solutions of the relaxed problem, a \emph{discrete} -- i.e. $u'(x) \in \{0,1\}$ -- global optimal solution can be recovered from {\tmem{any}} solution of the relaxed problem. We can thus reduce the \emph{combinatorial} problem to a \emph{convex} problem.
While there are reasons to believe that this procedure cannot be
extended to the multi-class case, we may still hope for ``nearly'' discrete solutions.


\subsection{Related Work}
The difficulty of the labeling problem varies greatly with its precise definition.
Formulations of the labeling problem can be categorized based on
\begin{enumerate}
\item whether they tackle the binary (two-class) or the much harder multiclass problem, and
\item whether they rely on a graph representation or are formulated in a spatially continuous framework.
\end{enumerate}
An early analysis of a variant of the \emph{binary} \emph{continuous} cut problem and 
the associated dual \emph{maximal flow} problem was done by Strang {\cite{Strang1983}}.
Chan et~al.~{\cite{Nikolova2006}} pointed out that by thresholding a
nonbinary result of the relaxed, convex problem at almost any threshold
one can generate binary solutions of the original, combinatorial problem
(this can be carried over to \emph{any} threshold and to slightly more general regularizers {\cite{Berkels2009,Zach2009}}. The proof heavily relies on the coarea formula \cite{Fleming1960}, which unfortunately
does not transfer to the multiclass case. The binary continuous case
is also closely related to the thoroughly analyzed Mumford-Shah \cite{Mumford1989} and Rudin-Osher-Fatemi (ROF)
\cite{Rudin1992} problems, where a quadratic data term is used.

For the \emph{graph-based} discretization, the binary case can be formulated as a
minimum-cut problem on a grid graph, which allows to solve the problem exactly
and efficiently for a large class of metrics using graph cuts
{\cite{Kolmogorov2005,Boykov2004}}. Graph cut algorithms have become a working horse in
many applications as they are very fast for medium sized problems. Unfortunately
they offer hardly any potential for parallelization.
As mentioned in the introduction, the graph representation invariably introduces a grid bias ~\cite{Boykov2003}
(Fig.~\ref{fig:grid}). While it is possible to reduce the resulting artifacts by using larger neighborhoods,
or by moving higher-order potentials through factor graphs, this
greatly increases graph size as well as processing time.

Prominent methods to handle the graph-based \emph{multiclass} case rely on finding a local minimum by solving a sequence of
binary graph cuts {\cite{Boykov2001}} (see {\cite{Komodakis2007}} for a recent
generalization). These methods have recently been extended to the continuous
formulation \cite{Trobin2008} with similar theoretical performance \cite{Olsson2009}.
Our results can be seen as a continuous analogon to {\cite{Ishikawa2003}},
where it was shown that potentials of the form $\pot(i,j) = |i-j|$ can be
exactly formulated as a cut on a multi-layered graph. An early analysis can be
found in {\cite{Kleinberg1999}}, where the authors also derive suboptimality
bounds of a linear programming relaxation for metric distances. All these
methods rely on the graph representation with pairwise potentials.

\begin{figure}[t]%
\center
\includegraphics[width=0.80\columnwidth]{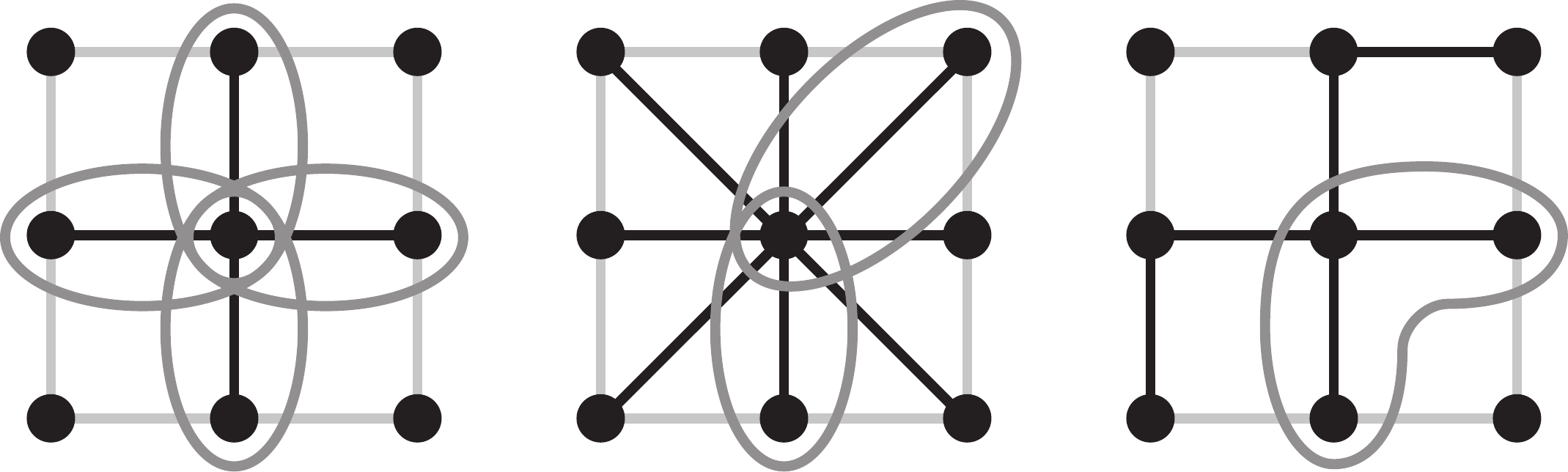}
\caption{Discretization schemes. \textbf{Left to right:} Graph-based with 4- and 8-neighborhood; higher order potentials. In graph-based approaches the regularizer is discretized using terms depending on the labels of at most two neighboring points. This leads to artifacts as isotropic metrics are not approximated well. Using higher-order terms the discrete functional more closely approximates the continuous functional (Fig.~\ref{fig:bias}).}%
\label{fig:grid}%
\end{figure}
\begin{figure}[t]%
\center
\includegraphics[width=0.26\columnwidth]{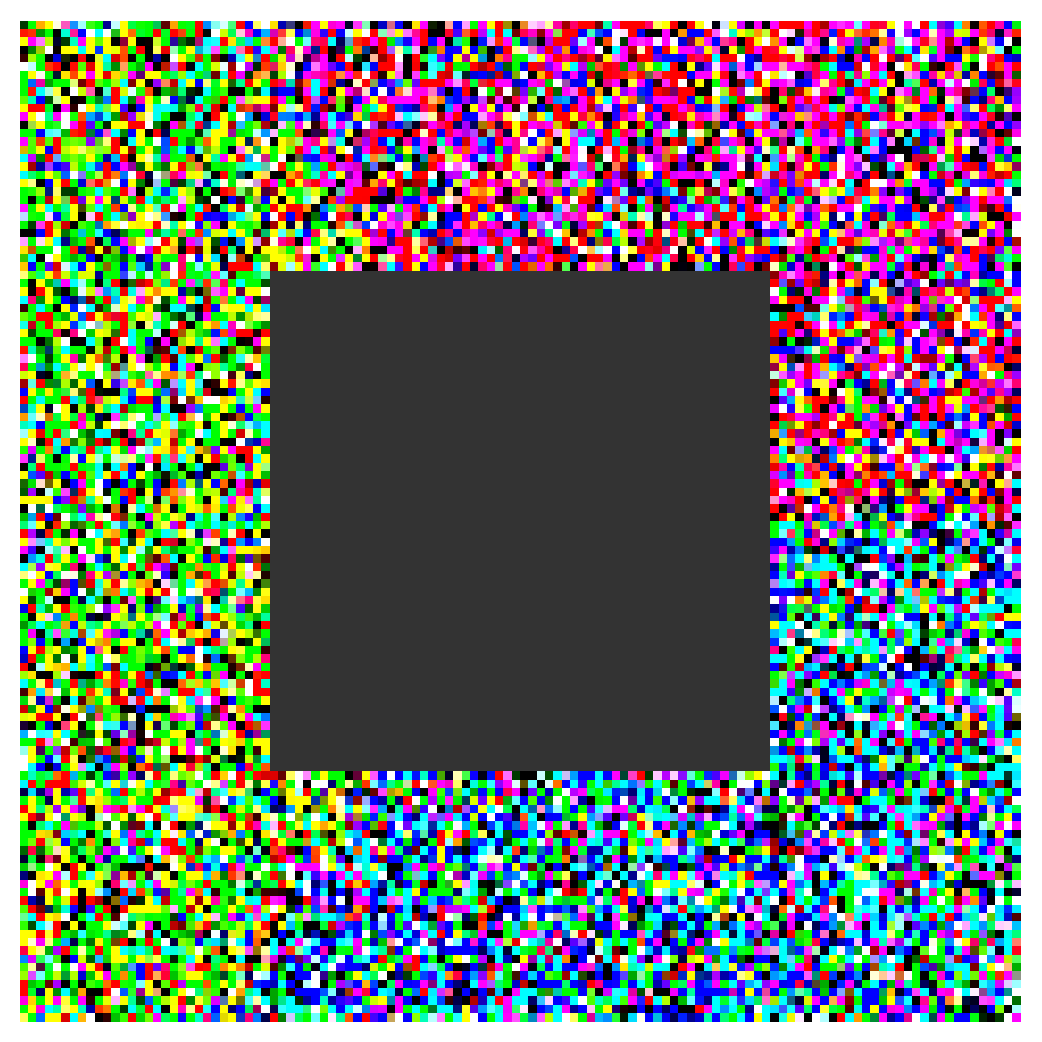}
\includegraphics[width=0.26\columnwidth]{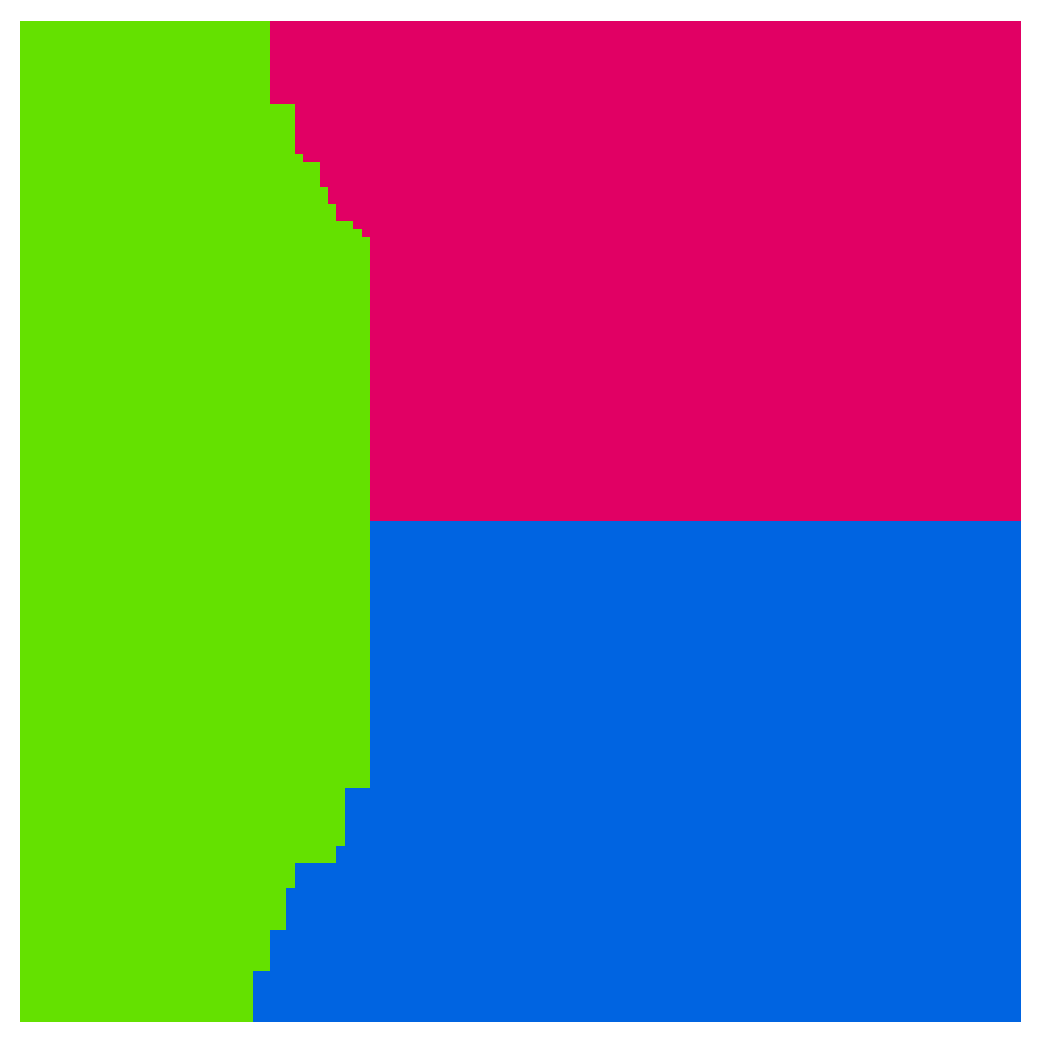}
\includegraphics[width=0.26\columnwidth]{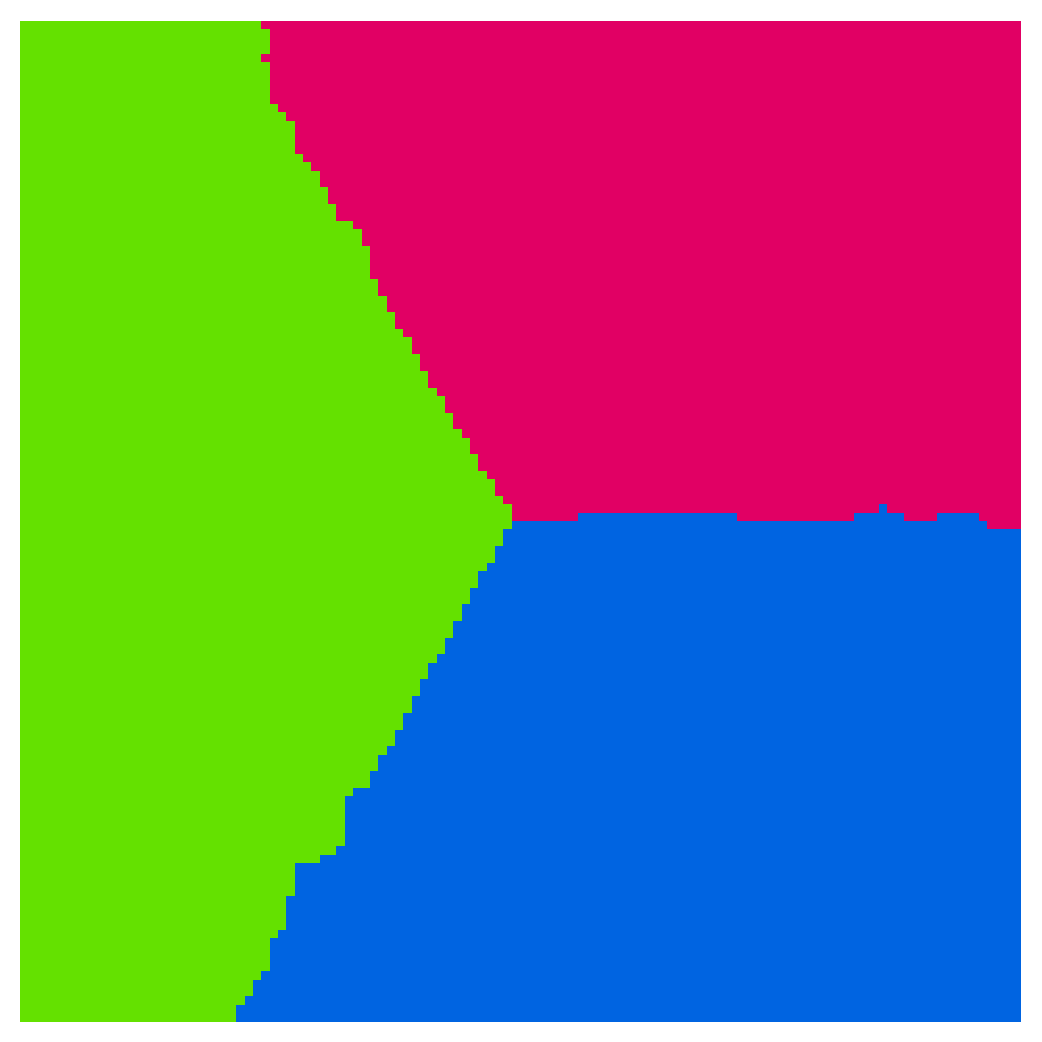}
\caption{Effect of the discretization on an inpainting problem. \textbf{Left to right:} Input image with unknown black region; graph-based method using pairwise potentials ($\alpha$-$\beta$-swap code from \cite{Szeliski2006}); proposed method using higher order potentials. Due to the introduced anisotropy, the graph-based method shows a bias towards axis-parallel edges.}
\label{fig:bias}%
\end{figure}

In this paper, we will focus on the \emph{continuous multiclass} setting with higher order potentials
in the discretization. Closely related to our approach is {\cite{Chambolle2008}}. In contrast to approaches that rely on a linear ordering of the labels {\cite{Ishikawa2003,Bae2009}}, the
authors represent labels in a higher-dimensional space. In a certain sense, \cite{Lie2006} can
be seen as a predecessor of this approach: in this work, the authors represent the label assignment
using a piecewise constant real-valued function, but parametrize this function using a set of
$l$ polynomial basis functions, which enables them to employ the Potts regularizer.

The approach of
{\cite{Chambolle2008}} allows to formulate interaction potentials of the form $\pot(i,j) = \gamma(|i - j|)$ with nondecreasing, positive, concave $\gamma$. The authors provide a thorough analysis of the continuous model and propose a relaxation based on the convex envelope, which gives almost discrete solutions in many cases. We will extend this approach
to the more general class of regularizers where $\pot$ is an arbitrary metric.
The same authors proposed a ``Fast Primal-Dual'' algorithm with proven convergence to solve the associated saddle point problem
{\cite{Pock2009}}. By lifting the objective to a higher dimensional space,
it turns out that the same method can be used to solve many problems also with nonlinear data term~{\cite{Pock2009a}}.

Our approach is a generalization of \cite{Zach2008,Lellmann2009}, where a similar linearization is used with the regularizer restricted to the Potts distance, and with less strong convergence results.
These methods have also been extended to segmentation on manifolds {\cite{Delaunoy2009}}.

Regarding optimization, several authors proposed smoothing of the primal or dual objective together
with gradient descent~\cite{Berkels2009,Bae2010}. In contrast, our approach does not require any a priori smoothing.
Using Nesterov's method \cite{Nesterov2004} for the labeling problem was proposed in \cite{Lellmann2009}. An earlier analysis of the method in the context of
$\ell_1$-norm and TV minimization can be found in {\cite{Weiss2007}}. In \cite{Becker2009} the method
is applied to a class of general $\ell_1$ regularized problems. In {\cite{Goldstein2009}} a predecessor of the
proposed Douglas-Rachford approach was presented in the Split Bregman framework {\cite{Setzer2009}} and restricted to the
two-class case. We provide an extension to the multi-class case, with proof of convergence
and a sound stopping criterion.


\subsection{Contribution}

The paper is organized as follows:
\begin{enumerate}
  \item We formulate natural requirements on the regularizer $J$ and show
  their implications on the choice of the interaction potential $\pot$ (Sect.~\ref{sec:axiomatic}).
  In particular, $\pot$ must necessarily be a metric under these requirements (Prop.~\ref{prop:metricd}).
  
  \item Given such a metric, we study two possible approaches to extend
  regularizers $J$ on the relaxed set (Sect.~\ref{sec:constructive}):
  \begin{itemize}
    \item The ``envelope'' approach, which is a generalization of the method
    recently suggested by Chambolle et al. (Sect.~\ref{sec:envelopemethod}).
    While there is no simple closed form expression, we show that it can be
    used to construct a true extension for {\tmem{any}} metric $\pot$ (Prop.~\ref{prop:supij}).
    
    \item The ``Euclidean distance'' approach (Sect.~\ref{sec:euclmetricmethod}), which
    yields exact extensions for Euclidean metrics $\pot$ only but has a closed form
    expression. We review some methods for the approximation of non-Euclidean metrics.
  \end{itemize}
  We provide a unified continuous framework and show existence of a minimizer.
  
  \item Both approaches lead to non-smooth convex problems, which can be
  studied in a general saddle point formulation (Sect.~\ref{sec:discrete}). Within
  this framework, we propose an improved binarization technique for nonstandard
  potentials to recover approximate solutions for the non-relaxed problem (Sect.~\ref{sec:binarization}).
  
  \item We provide and analyze two different methods that are capable of minimizing the
  specific class of saddle point problems (Sect.~\ref{sec:optimization}):
  \begin{itemize}
    \item A specialization of a method for nonsmooth optimization as suggested
    by Nesterov (Sect.~\ref{sec:nesterov}). The method is virtually
    parameter-free and provides explicit a priori and a posteriori optimality
    bounds.
    
    \item A Douglas-Rachford splitting approach (Sect.~\ref{sec:douglas-rachford}). We show that
    the approach allows to compute a sequence of dual iterates that provide an
    optimality bound and stopping criterion in form of the primal-dual gap.
  \end{itemize}
  Both methods are highly parallelizable and are shown to converge. For reference, we also
  summarize the primal-dual technique from \cite{Pock2009} (Sect.~\ref{sec:fpd}).
  
  \item Finally, we illustrate and compare the above methods under varying conditions and demonstrate the 
applicability of the proposed approaches on real-world problems (Sect.~\ref{sec:experiments}).
\end{enumerate}
In contrast to existing graph-based methods, we provide a continuous and
isotropic formulation, while in comparison with existing continuous
approaches, we provide a unified framework for arbitrary, non-uniform metrics $\pot$.
The Euclidean metric method and Nesterov optimization have been announced in less generality in \cite{Lellmann2009,Lellmann2009a}.


\section{Mathematical Preliminaries}\label{sec:preliminaries}

In the following sections we provide a reference of the notation used, and
a brief introduction to the concept of functions of bounded variation and
corresponding functionals. We aim to provide the reader with the basic ideas. For more detailed expositions we
refer to \cite{Ambrosio2000,Ziemer-89}.


\subsection{Notation}

In the following, superscripts $v^i$ denote a collection of vectors or matrix columns, while subscripts $v_k$ denote vector components, i.e. we denote,
for $A \in \mathbbm{R}^{d \times l}$,
\begin{eqnarray}
  A & = & (a^1 | \ldots |a^l) = (A_{i j}),\quad A_{i j} = \left(a^j\right)_i = a^j_i,\quad 1 \leqslant i \leqslant d, \, 1 \leqslant j \leqslant l.
\end{eqnarray}
An additional bracket $v^{(i)}$ indicates an element of a sequence $(v^{(i)})$.
We will frequently make use of the Kronecker product~\cite{Graham1981}
\begin{equation}
\otimes: \mathbbm{R}^{n_1 \times m_1} \times \mathbbm{R}^{n_2 \times m_2} \rightarrow \mathbbm{R}^{(n_1 n_2)\times(m_1 m_2)}
\end{equation}
in order to formulate all results for arbitrary dimensions.
The standard simplex in $\mathbbm{R}^l$ is denoted by
$\Delta_l \assign \{x \in \mathbbm{R}^l |x \geqslant 0, e^{\top} x = 1\}$, where $e \assign (1, \ldots, 1)^{\top} \in \mathbbm{R}^l$. $I_n$ is
the identity matrix in $\mathbbm{R}^n$ and $\| \cdot \|$ the usual Euclidean
norm for vectors resp. the Frobenius norm for matrices. Analogously, the standard
inner product $\langle \cdot, \cdot \rangle$ extends to pairs of matrices as the sum over
their elementwise product. $\mathcal{B}_r (x)$ denotes the ball of radius $r$ at $x$, and $S^{d-1}$ the set of $x \in \mathbbm{R}^d$ with $\|x\| = 1$. The characteristic function $\chi_{\mathcal{S}}(x)$ of a
set $\mathcal{S}$ is defined as $\chi_{\mathcal{S}} (x) = 1$ iff $x \in S$
and $\chi_{\mathcal{S}} (x) = 0$ otherwise. By $\delta_{\mathcal{S}} (x) = 0$ iff $x \in \mathcal{S}$ and 
$\delta_{\mathcal{S}} (x) = +\infty$ otherwise we denote the corresponding indicator function. For a
convex set $\mathcal{C}$, $\sigma_{\mathcal{C}} (u) \assign \sup_{v \in \mathcal{C}} \langle u, v \rangle$ is
the support function from convex analysis. $\mathcal{J}(u)$ denotes the
classical Jacobian of $u$.

$C^k_c (\Omega)$ is the space of $k$-times continuously differentiable functions on
$\Omega$ with compact support, and $C_0(\Omega)$ the completion of $C^0_c(\Omega)$ under the
supremum norm. As usual, $\mathcal{L}^d$ denotes the $\pot$-dimensional
Lebesgue measure, while $\mathcal{H}^k$ denotes the $k$-dimensional Hausdorff
measure. For some measure $\mu$ and set $M$, $\mu \restrict M$ denotes the
restriction of $\mu$ to $M$, i.e. $(\mu \restrict M)(A) \assign \mu(M \cap A)$.


\subsection{Total Variation and $\tmop{BV}$}\label{sec:totalvariationbv}
The total variation will be our main tool to construct the regularizer $J$.
For a differentiable scalar-valued function $u$, the total variation is simply
the integral over the norm of its gradient:
\begin{eqnarray}
  \tmop{TV} (u) & = & \int_{\Omega} \| \nabla u\|d x\,. \label{eq:tv-smooth-scalar}
\end{eqnarray}
As $u$ is the designated labeling function, which ideally should be
piecewise constant, the differentiability and continuity assumptions have to be dropped. In the
following we will shortly review the general definition of the total variation
and its properties.

We require the image domain $\Omega \subseteq \mathbbm{R}^d$ to be a bounded open domain with compact Lipschitz boundary,
that is the boundary can locally be represented as the graph of a Lipschitz-continuous function.
For simplicity, we will assume in the following that $\Omega = (0, 1)^d$.

We consider general vector-valued functions $u = (u_1,\ldots,u_l): \Omega \rightarrow \mathbbm{R}^l$ which are locally
absolutely integrable, i.e. $u \in L^1_{\tmop{loc}} (\Omega)^l$. As $\Omega$ is bounded
this is equivalent to being absolutely integrable, i.e. $u \in L^1 (\Omega)^l$.
For any such function $u$, its {\tmem{total variation}} $\tmop{TV} (u)$ is
defined in a dual way {\cite[(3.4)]{Ambrosio2000}} as
\begin{eqnarray}
  \tmop{TV} (u) & \assign & \sup_{
  v \in \mathcal{D}^{\tmop{TV}}
  }\;
  \sum_{j = 1}^l \int_{\Omega} u_j \tmop{div} v^j d x
   = \sup_{
  v \in \mathcal{D}^{\tmop{TV}}
  }\;
  \int_{\Omega} \langle u, \tmop{Div} v \rangle d x \,, \label{eq:tvdef-dual}\\
  \mathcal{D}^{\tmop{TV}} & \assign & \{ v \in C^{\infty}_c (\Omega)^{d \times l} | \|v(x)\| \leqslant 1 \; \forall x \in \Omega \}\,,\\
  \tmop{Div} v & \assign & \left(\tmop{div} v^1, \ldots, \tmop{div} v^l\right)^{\top}\,. \nonumber
\end{eqnarray}
This definition can be derived for continuously differentiable $u$
by extending (\ref{eq:tv-smooth-scalar}) to vector-valued $u$,
\begin{eqnarray}
  \tmop{TV} (u) & = & \int_{\Omega} \|\mathcal{J}(u)\| d x\,,
\end{eqnarray}
replacing the norm by its dual formulation and partial integration.
If $u$ has finite total variation, i.e.~$\tmop{TV} (u) < \infty$, $u$ is said to be of {\tmem{bounded variation}}.
The vector space of all such functions is denoted by $\tmop{BV}(\Omega)^l$:
\begin{eqnarray}
  \tmop{BV} (\Omega)^l & = & \left\{ u \in \left( L^1 (\Omega) \right)^l | \tmop{TV} (u) < \infty \right\}\,.
\end{eqnarray}
Equivalently, $u \in \tmop{BV} (\Omega)^l$ iff $u \in L^1(\Omega)^l$ and its distributional derivative corresponds to a finite Radon
measure; i.e. $u_j \in L^1 (\Omega)$ and there exist $\mathbbm{R}^d$-valued
measures $D u_j = (D_1 u_j, \ldots, D_d u_j)$ on the Borel subsets
$\mathcal{B}(\Omega)$ of $\Omega$ such that {\cite[p.118]{Ambrosio2000}}
\begin{eqnarray}
  & \sum_{j = 1}^l \int_{\Omega} u_j \tmop{div} v^j d x = 
   - \sum_{j = 1}^l \sum_{i = 1}^d \int_{\Omega} v_i^j d D_i u_j \,,
     \quad\forall v \in (C^{\infty}_c (\Omega))^{d \times l}\,. &
\end{eqnarray}
These measures form the distributional gradient $D u = (D u_1 | \ldots | D u_l)$, which is again a measure
that vanishes on any $\mathcal{H}^{(d-1)}$-negligible set. If $u \in \tmop{BV} (\Omega)$ then $|D u| (\Omega) = \tmop{TV} (u)$, where $|D u|$ is the total variation of the measure $D u$ in the measure-theoretic sense {\cite[3.6]{Ambrosio2000}}.
The total variation of characteristic functions has an intuitive geometrical
interpretation: For a Lebesgue-measurable subset $\mathcal{S} \subseteq \mathbbm{R}^d$, its
{\tmem{perimeter}} is defined as the total variation of its characteristic function,
\begin{eqnarray}
  \tmop{Per} (\mathcal{S}) & \assign & \tmop{TV} (\chi_\mathcal{S})\,.
\end{eqnarray}
Assuming the boundary $\partial \mathcal{S}$ is sufficiently regular,
$\tmop{Per}(\mathcal{S})$ is just the classical length ($d = 2$) or area ($d = 3$) of the boundary.


\subsection{Properties of $\tmop{TV}$ and Compactness}

We review the most important ingredients for proving existence of
minimizers for variational problems on involving $\tmop{BV}$ involving $\tmop{TV}$.

\noindent{\emph{Convexity.}}
As $\tmop{TV}$ is the pointwise supremum of a family
of linear functions, it is {\tmem{convex}} and {\tmem{positively
homogeneous}}, i.e. $\tmop{TV} (\alpha u) = \alpha \tmop{TV} (u)$ for $\alpha
> 0$.

\noindent{\emph{Lower Semicontinuity.}} A functional $J$ is said to be {\tmem{lower
semicontinuous}} with respect to some topology, if for any sequence $(u^{(k)})$ converging to $u$,
\begin{eqnarray}
  \lim \inf_{k \rightarrow \infty} J (u^{(k)}) & \geqslant & J (u)\,.
\end{eqnarray}
It can be shown that for fixed $\Omega$, the total variation $\tmop{TV}$ is
well-defined on $L^1_{\tmop{loc}} (\Omega)^l$ and lower semicontinuous in
$\tmop{BV} (\Omega)^l$ w.r.t. the $L^1_{\tmop{loc}} (\Omega)^l$ topology
{\cite[3.5,3.6]{Ambrosio2000}}; hence also in $L^1 (\Omega)^l$ due to the
boundedness of $\Omega$.

\noindent{\emph{Compactness.}} For $\tmop{BV}$, instead of the norm topology induced by
\begin{eqnarray}
  \|u\|_{\tmop{BV}} & \assign & \int_{\Omega} \|u\| d x + \tmop{TV} (u)\,,
\end{eqnarray}
which makes $\tmop{BV} (\Omega)^l$ a Banach space but is often too strong,
one frequently uses the weak* convergence: Define $u^{(k)} \rightarrow u$ \emph{weakly*} iff
\begin{enumerate}
  \item $u, u^{(k)} \in \tmop{BV} (\Omega)^l \; \forall k \in \mathbbm{N}$,
  
  \item $u^{(k)} \rightarrow u$ in $L^1 (\Omega)$ and
  
  \item $D u^{(k)} \rightarrow D u$ weakly* in measure, i.e.
  \begin{eqnarray}
    & & \forall v \in C_0 (\Omega) : \lim_{k \rightarrow \infty}
    \int_{\Omega} v \, d D u^{(k)} = \int_{\Omega} v \, d D u\,. 
  \end{eqnarray}
\end{enumerate}
For $u, u^{(k)} \in \tmop{BV} (\Omega)^l$ this is equivalent to $u^{(k)} \rightarrow
u$ in $L^1 (\Omega)^l$, and $(u^{(k)})$ being uniformly bounded in $\tmop{BV} (\Omega)^l$, i.e.
there exists a constant $C < \infty$ such that $\|u^{(k)} \|_{\tmop{BV}} \leqslant C\,\forall k \in \mathbbm{N}$
{\cite[3.13]{Ambrosio2000}}. For the weak* topology in $\tmop{BV}$,
a compactness result holds {\cite[3.23]{Ambrosio2000}}:
If $(u^{(k)}) \subset \tmop{BV} (\Omega)^l$ and $(u^{(k)})$ is uniformly bounded in $\tmop{BV}(\Omega)^l$,
then $(u^{(k)})$ contains a subsequence weakly*-con\-verg\-ing to some $u \in \tmop{BV} (\Omega)^l$.


\subsection{General Functionals on $\tmop{BV}$}

We will now review how general functionals depending on the distributional
gradient $D u$ can be defined. Recall that for any $u \in \tmop{BV} (\Omega)^l$ the distributional gradient $D u$ is a measure.
Moreover, it can be uniquely decomposed into three mutually singular measures
\begin{eqnarray}
  D u & = & D^a u + D^j u + D^c u\,,
\end{eqnarray}
that is: 
An {\tmem{absolutely continuous}} part $D^a$, \ the {\tmem{jump}} part $D^j$, and
the {\tmem{Cantor}} part $D^c$. Mutual singularity refers to the fact that
$\Omega$ can be partitioned into three subsets, such that each of the measures is concentrated
on exactly one of the sets.
We will give a short intuitive explanation, see {\cite[3.91]{Ambrosio2000}} for the full definitions.

The $D^a$ part is absolutely continuous with respect to the $\pot$-dimensional Lebesgue measure $\mathcal{L}^d$,
i.e. it vanishes on any $\mathcal{L}^d$-negligible set. It captures the
``smooth'' variations of $u$: in any neighborhood where $u$ has a (possibly weak) Jacobian $\mathcal{J}(u)$,
the jump and Cantor parts vanish and
\begin{equation}
  D u = D^a u = \mathcal{J}(u) \mathcal{L}^d\,.
\end{equation}

The jump part $D^j$ is concentrated on the set of points where locally $u$ jumps between two values
$u^-$ and $u^+$ along a $(d-1)$-dimensional hypersurface with normal $\nu_u \in S^{d-1}$ (unique up to a change of sign).
In fact, there exists a \emph{jump set} $J_u$ of discontinuities of $u$ and Borel functions $u^+, u^- : J_u \rightarrow \mathbbm{R}^l$ and $\nu_u : J_u \rightarrow S^{d - 1}$ such that {\cite[3.78, 3.90]{Ambrosio2000}}
\begin{equation}
    D^j u = D u \restrict J_u = \nu_u (u^+ - u^-)^{\top} \mathcal{H}^{d - 1} \restrict J_u\,,
\end{equation}
where $\mathcal{H}^{d - 1} \restrict J_u$ denotes the restriction of the $(d-1)$-di\-men\-sional Hausdorff measure
on the jump set $J_u$, i.e. $(\mathcal{H}^{d - 1} \restrict J_u)(A) = \mathcal{H}^{d - 1}(J_u \cap A)$ for measurable
sets $A$. The Cantor part $D^c$ captures anything that is left.

As an important consequence of the mutual singularity, the total variation decomposes
into $|D u| = |D^a u| + |D^j u| + |D^c u|$. Using this idea, one can define functionals depending on the distributional gradient $D u$ \cite[Prop.~2.34]{Ambrosio2000}. For $u \in \tmop{BV} (\Omega)^l$ and some convex, lower semi-continuous $\Psi : \mathbbm{R}^{d \times l} \rightarrow \mathbbm{R}$, define
\begin{eqnarray}
  J (u) & \assign & \int_{\Omega} \Psi (Du) \assign \int_{\Omega} \Psi (\mathcal{J}(u) (x)) d x + \ldots \nonumber\\
  & & \int_{J_u} \Psi_{\infty} \left( \nu_u (x) \left( u^+ (x) - u^- (x) \right)^{\top} \right) d\mathcal{H}^{d - 1} + \ldots  \nonumber\\
  & & \int_{\Omega} \Psi_{\infty} \left( \frac{D^c u}{|D^c u|} \right) d|D^c u|\,.\label{eq:generalju}
\end{eqnarray}
Here $\Psi_{\infty}$ is the recession function $\Psi_{\infty} (x) = \lim_{t
\rightarrow \infty} \frac{\Psi (tx)}{t}$ of $\Psi$,
and $D^c u / |D^c u|$ denotes the \emph{polar decomposition} of $D^c u$, which is
the density of $D^c u$ with respect to its total variation measure $|D^c u|$.
If $\Psi$ is positively homogeneous, $\Psi^{\infty} = \Psi$ holds, and
\begin{eqnarray}
  J(u) & = & \int_{\Omega} \Psi \left( \frac{Du}{|Du|} \right) d|Du|.\label{eq:poshomogju}
\end{eqnarray}
From \eqref{eq:generalju} it becomes clear that the meaning of $\Psi$ acting
on the {\tmem{Jacobian}} of $u$ is extended to the
jump set as acting on the {\tmem{difference}} of the left and right side
limits of $u$ at the discontinuity. This is a key point: by switching to
the measure formulation, one can handle noncontinuous functions as well.


\section{Necessary Properties of the Interaction Potential}\label{sec:axiomatic}

Before applying the methods above to the labeling problem, we start with some
basic considerations about the regularizer and the interaction potential $\pot$.
We begin by formalizing the requirements on the regularizer of the relaxed problem as mentioned in
the introduction. Let us assume we are given a general interaction potential
$\pot : \{1, \ldots, l\}^2 \rightarrow \mathbbm{R}$. Intuitively, $\pot (i, j)$
assigns a weight to switching between label $i$ and label $j$. We require
\begin{eqnarray}
  \pot (i, j) & > & 0, i \neq j\,,  \label{eq:dijinj}
\end{eqnarray}
but no other metric properties (i.e. symmetry or triangle inequality) for now.
Within this work, we postulate that the regularizer should satisfy
\begin{enumerate}
  \item[(P1)] \label{it:postconvhom}$J$ is convex and positively homogeneous
  on $\tmop{BV}(\Omega)^l$.
  
  \item[(P2)] \label{it:postconstant}$J (u) = 0$ for any constant $u$, i.e.
  there is no penalty for constant labelings.
  
  \item[(P3)] \label{it:postper}For any partition of $\Omega$ into
  two sets $S, S^\complement$ with $\tmop{Per} (S) < \infty$, and any
  $i, j \in \{1, \ldots, l\}$,
  \begin{eqnarray}
    \left. J (e^i \chi_S + e^j \chi_{S^\complement} \right) & = & \pot (i, j) \tmop{Per}(S)\,,\label{eq:goal}
  \end{eqnarray}
  i.e. a change from label $i$ to label $j$ gets penalized proportional to
  $\pot (i, j)$ as well as the perimeter of the interface. Note that this implies
  that $J$ is isotropic (i.e. rotationally invariant).
\end{enumerate}
We require convexity in (P1) in order to render global optimization
tractable. Indeed, if $J$ is convex, the overall objective function
(\ref{eq:problemrelaxed}) will be convex as well due to the linearization of
the data term. Positive homogeneity is included as it allows $J$ to be
represented as a support function (i.e. its convex conjugate is an indicator
function and $J = \sigma_{\mathcal{D}}$ for some closed convex set
$\mathcal{D}$), which will be exploited by our optimization methods.

Requirements (P3) and (P2) formalize the principle that the multilabeling
problem should reduce to the classical continuous cut (\ref{eq:continuouscut-binary}) in the two-class case.
This allows to include boundary length-based terms in the regularizer that can additionally be weighted
depending on the labels of the adjoining region (Fig.~\ref{fig:leafdifferentreg}).
\begin{figure}%
\centering
$\begin{array}{cc}%
\includegraphics[width=.30\columnwidth]{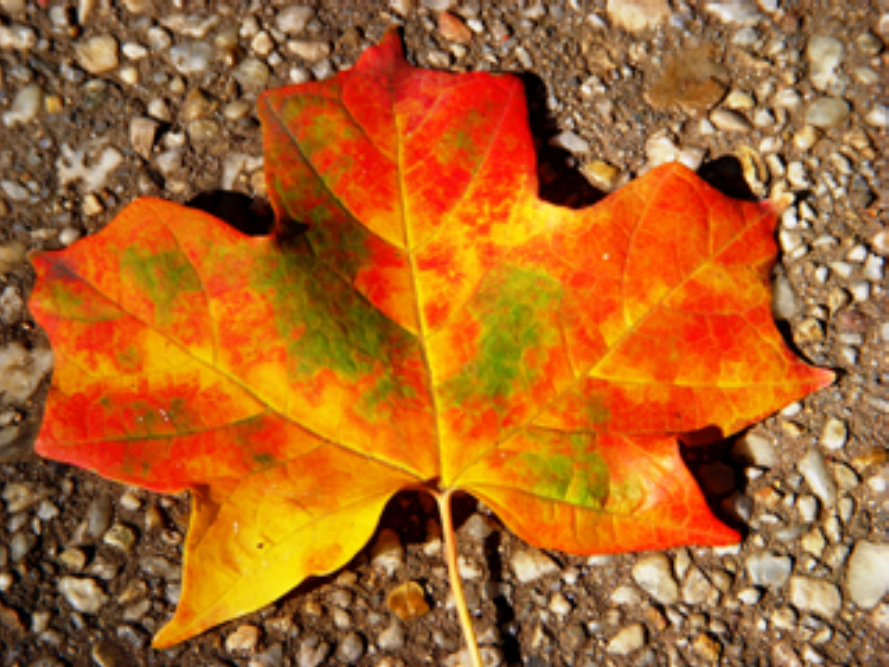} &
\includegraphics[width=.30\columnwidth]{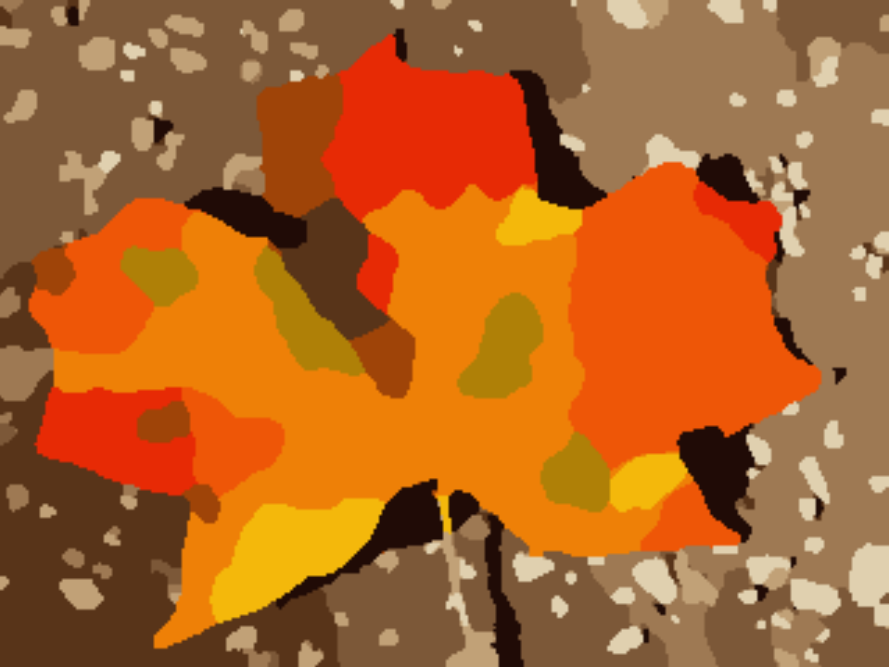} \\
\includegraphics[width=.30\columnwidth]{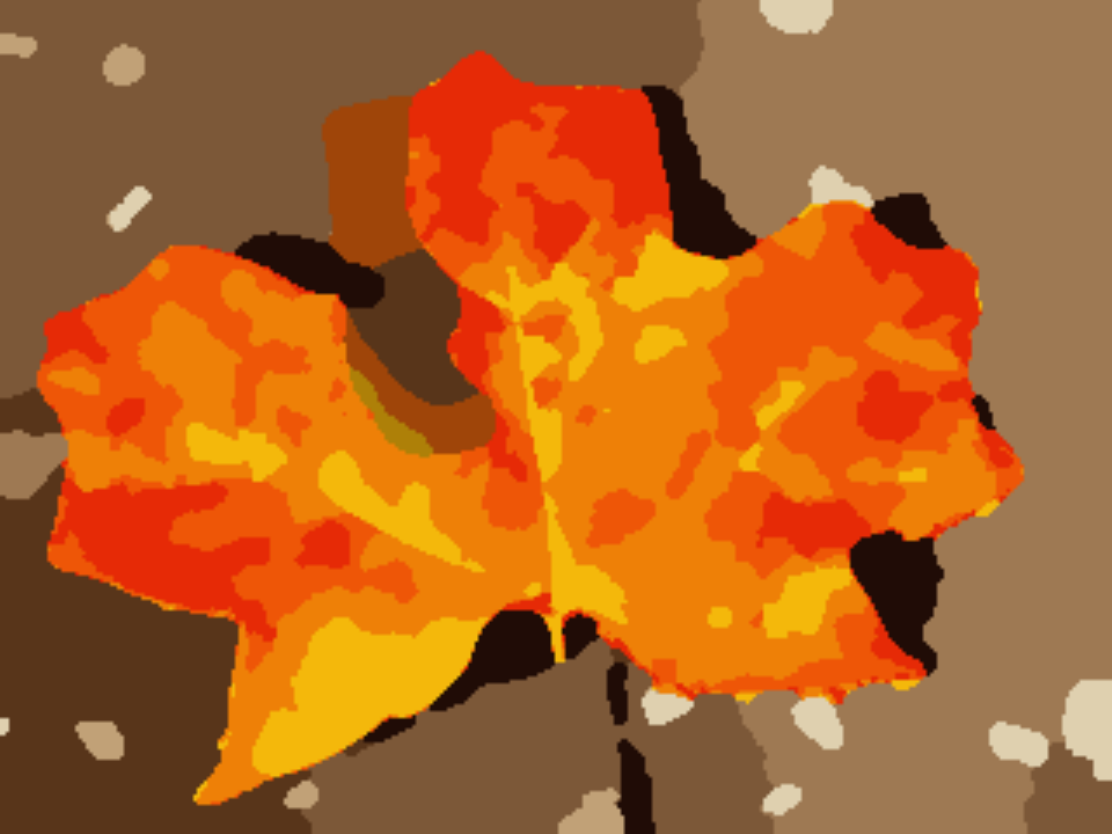} &
\includegraphics[width=.30\columnwidth]{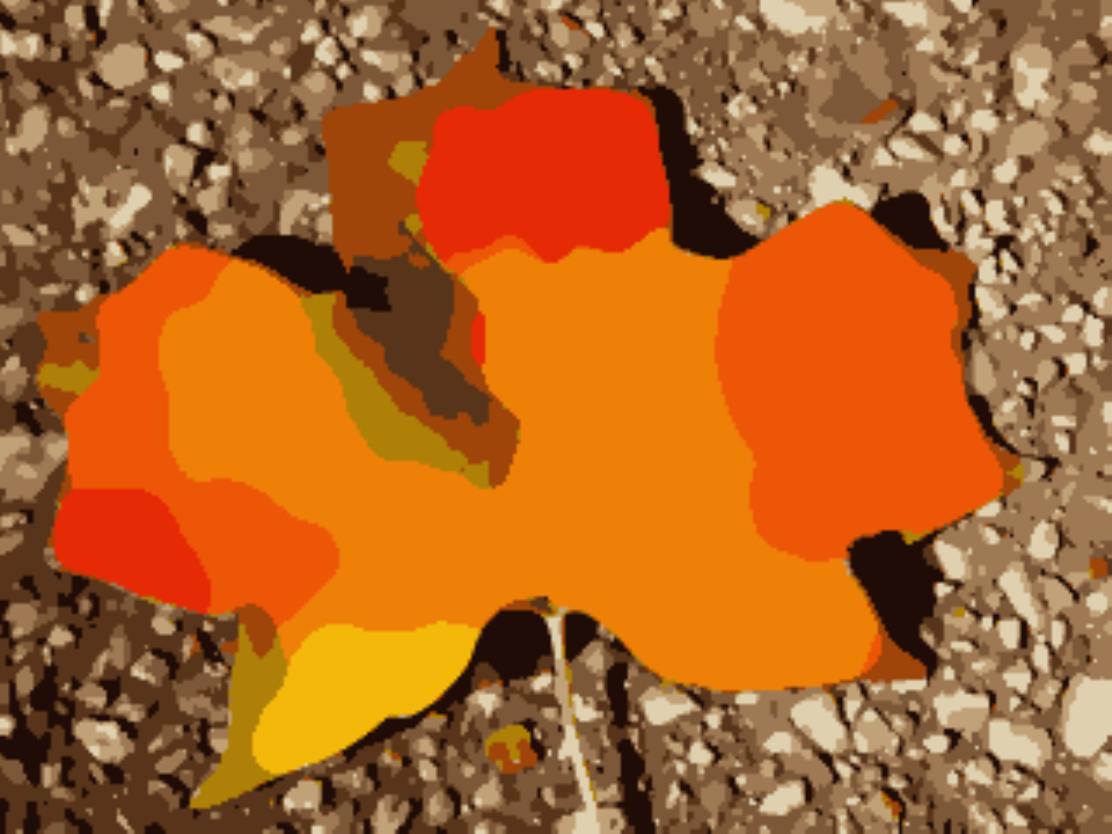}
\end{array}$%
\caption{Effect of choosing different interaction potentials. \textbf{Top row:} The
original image (left) is segmented into $12$ regions corresponding to prototypical colors vectors.
The Potts interaction potential penalizes the boundary length independently of the labels (right), which leads
to a uniformly smooth segmentation. \textbf{Bottom row:} By modifying the interaction potential,
the regularization strength is selectively adjusted to suppress
background (left) or foreground (right) structures while allowing
for fine details in the other regions.
}%
\label{fig:leafdifferentreg}
\end{figure}
Together, these requirements pose a natural restriction on $\pot$:

\begin{proposition}
  \label{prop:metricd}Let $(J, \pot)$ satisfy \tmtextit{(P1) -- (P3)}. Then $\pot$ must necessarily be a metric, i.e. for all $i, j, k \in \{1, \ldots, l\}$,
  \begin{enumerate}
    \item $\pot (i, i) = 0$,
    
    \item $\pot (i, j) = \pot (j, i), \; \forall i \neq j$,
    
    \item $\pot$ is subadditive: $\pot (i, k) \leqslant \pot (i, j) + \pot (j, k) .$
  \end{enumerate}
\end{proposition}
\begin{proof}
1. follows from (P2) and (P3) by choosing $i = j$ and $S$ with $\tmop{Per} (S) > 0$. Symmetry in 2. is
obtained from (P3) by replacing $S$ with $S^\complement$, as $\tmop{Per} (S) =
\tmop{Per} (S^\complement)$. To show 3., first note that $J(u) = 2 J(u/2+c/2)$ for any constant $c \in \mathbbm{R}^l$ and all $u\in\tmop{BV}(\Omega)^l$, since $J(u) = J(u+c-c) \leqslant 2 J((u+c)/2)+J(-c/2) = 2 J(u/2+c/2) \leqslant J(u)+J(c) = J(u)$.
Fix any set $S$ with perimeter
\begin{equation}
  c \assign \tmop{Per} (S) > 0\,.
\end{equation}
Then, using the above mentioned fact and the positive homogeneity of $J$,
\begin{eqnarray}
c \pot (i, k) & = & J \left( e^i \chi_S + e^k \chi_{S^\complement} \right)\\
& = & 2 J \left( \frac{1}{2} \left( e^i \chi_S + e^k \chi_{S^\complement}  \right) + \frac{1}{2} e^j \right)\\
& = & 2 J \left( \frac{1}{2} \left( e^i \chi_S + e^j \chi_{S^\complement} \right) +
 \frac{1}{2} \left( e^j \chi_S + e^k \chi_{S^\complement} \right) \right)\\
& \leqslant & J \left( e^i \chi_S + e^j \chi_{S^\complement} \right) + J \left( e^j \chi_S + e^k \chi_{S^\complement}\right)\\
& = & c \left( \pot \left( i, j \right) + \pot (j, k) \right)\,.
\end{eqnarray}
\end{proof}

Note that if the requirement (\ref{eq:dijinj}) is dropped, it is easy to show that if
  $\pot (i, j) = 0$ for some $i \neq j$, then $\pot (i, k) = \pot (j, k)$ for
  {\tmem{any}} $k$. In this case the classes $i$ and $j$ can be collapsed into a single class as
  far as the regularizer is concerned. The decision between label $i$ and $j$ is
  then completely local, i.e. depends only on the data term and can be postponed to
  a post-processing step by modifying the data term to
  \begin{eqnarray}
    s_i' (x) \assign s_j' (x) & \assign & \min \{s_i (x), s_j (x)\}\,.
  \end{eqnarray}
Thus (\ref{eq:dijinj}) is not a real limitation and can be always assured.
As a side note, it can be shown that, under some assumptions and in the space
of piecewise constant functions, the subadditivity of $\pot$ already follows if
$J$ is required to be lower semicontinuous {\cite[p.88]{Braides2002}}.

Proposition \ref{prop:metricd} implies that for non-metric $\pot$, we generally
cannot expect to find a regularizer satisfying (P1)--(P3). Note that here $J$
is not required to be of any particular form. In the following sections, we will
show that on the other hand, if $\pot$ is metric as in Proposition
\ref{prop:metricd}, such a regularizer can always be constructed.
This implies that the interaction potentials allowing for a regularizer that
satisfies (P1)--(P3) are exactly the metrics.


\section{Constructing Regularizers from the Interaction Potential}\label{sec:constructive}

We study next how to construct regularizers on $\tmop{BV} (\Omega)^l$
satisfying (P1)--(P3). As in (\ref{eq:generalju}) we set
\begin{equation}
J(u) \assign \int_{\Omega} \Psi(D u)\,.\label{eq:generaljurep}
\end{equation}
We additionally require $\Psi : \mathbbm{R}^{d \times l}
\rightarrow \mathbbm{R}_{\geqslant 0}$ to be a support function, i.e. for some closed, convex $\emptyset \neq \mathcal{D}_{\tmop{loc}} \subseteq \mathbbm{R}^{d\times l}$,
\begin{eqnarray}
  & & \Psi (z) = \sigma_{\mathcal{D}_{\tmop{loc}}}(z) = \sup_{v (x) \in \mathcal{D}_{\tmop{loc}}} \langle z, v (x) \rangle \,.
\end{eqnarray}
 As a support function, $\Psi$ coincides with its recession function $\Psi_{\infty}$, thus
\begin{eqnarray}
  J (u) & = & \int_{\Omega} \Psi (\mathcal{J}(u) (x)) d x + \ldots \nonumber\\
        &   & \int_{J_u} \Psi \left( \nu_u (x) \left( u^+ (x) - u^- (x) \right)^{\top} \right) d\mathcal{H}^{d - 1} + \ldots \nonumber\\
        &   & \int_{\Omega} \Psi \left( \frac{D^c u}{|D^c u|} \right) d|D^c u|\,. \label{eq:psidu}
\end{eqnarray}
Also, we have an equivalent dual formulation in analogy to (\ref{eq:tvdef-dual}),
\begin{eqnarray}
  J(u) = \sup \{ \int_{\Omega} & & \langle u\,, \tmop{Div} v \rangle d x | v \in C_c^{\infty}(\Omega)^{d \times l}\,,\;
  v(x) \in \mathcal{D}_{\tmop{loc}} \forall x \in \Omega \}\,. \label{eq:jdef-dual}
\end{eqnarray}
For simplicity we will also assume that $\mathcal{D}_{\tmop{loc}}$ is rotationally invariant
along the image dimensions, i.e. for any rotation matrix $R \in \mathbbm{R}^{d
\times d}$,
\begin{eqnarray}
  v = (v^1, \ldots, v^l) \in \mathcal{D}_{\tmop{loc}} & \Leftrightarrow & (Rv^1, \ldots, Rv^l) \in \mathcal{D}_{\tmop{loc}}\,.\label{eq:dlocrotinvar}
\end{eqnarray}
This is equivalent to $J$ being isotropic.

We will now show under which circumstances a minimizer exists in $\tmop{BV}(\Omega)^l$, and then
see how the approach can be used to construct regularizers for specific interaction potentials.


\subsection{Existence of Minimizers}

The complete problem considered here is of the form (cf.
(\ref{eq:problemrelaxed}) and (\ref{eq:psidu}))
\begin{equation}
  \inf_{u \in \mathcal{C}} f (u)\,, \quad f (u) \assign \int_{\Omega} \langle u, s \rangle d x + J (u) \label{eq:fullproblemdef}
\end{equation}
where $J (u) = \int_{\Omega} \Psi(D u)$ as in (\ref{eq:generaljurep}),
\begin{equation}
  \Psi (z) = \sup_{v (x) \in \mathcal{D}_{\tmop{loc}}} \langle z, v (x) \rangle
\end{equation}
for some closed convex $\mathcal{D}_{\tmop{loc}} \neq \emptyset$, and
\begin{eqnarray}
  &  & \mathcal{C} = \{u \in \tmop{BV} (\Omega)^l| u (x) \in \Delta_l \, \mathcal{L}^d \text{-a.e. } x \in \Omega\}\,.
\end{eqnarray}
Note that $f$ is convex, as it is the pointwise supremum of affine functions (\ref{eq:jdef-dual}).
Again for simplicity we set $\Omega = (0, 1)^d$. Then we have the following

\begin{proposition}
  \label{prop:flscness} Let $\mathcal{D}_{\tmop{loc}} \neq \emptyset$ be closed convex, $\Psi = \sigma_{\mathcal{D}_{\tmop{loc}}}$, $s \in L^1 (\Omega)^l$, and
  \begin{equation}
  	f(u) = \int_{\Omega} \langle u, s \rangle d x + \int_{\Omega} \Psi(D u)\,.
  \end{equation}  
  Additionally assume that $\mathcal{D}_{\tmop{loc}} \subseteq \mathcal{B}_{\rho_u}(0) \subseteq \mathbbm{R}^{d \times l}$ for
  some $0 < \rho_u$. Then $f$ is lower semicontinuous in $\tmop{BV} (\Omega)^l$ with respect to $L^1$ convergence.
\end{proposition}

\begin{proof}
  As the data term is continuous, it suffices to show that the regularizer is lower semicontinuous. This is an
  application of {\cite[Thm.~5.47]{Ambrosio2000}}.
  In fact, the theorem shows that $f$ is the relaxation of $\tilde{f} :
  C^1 (\Omega)^l \rightarrow \mathbbm{R}$,
  \begin{equation}
    \tilde{f} (u) \assign \int_{\Omega} \langle u, s \rangle d x + \int_{\Omega} \Psi (\mathcal{J} u (x)) d x\,,
  \end{equation}
  on $\tmop{BV} (\Omega)^l$ w.r.t. $L^1_{\tmop{loc}}$ (thus here $L^1$)
  convergence and thus lower semicontinuous in $\tmop{BV} (\Omega)^l$. To
  apply the theorem, we have to show that $\Psi$ is quasiconvex in the sense of
  \cite[5.25]{Ambrosio2000}, which holds as it is convex by construction.
  The other precondition is (at most) linear growth of $\Psi$, which holds with $0 \leqslant \Psi (x) \leqslant \rho_u \|x\|.$
\end{proof}

\begin{proposition}\label{prop:minexistence}
  Let $f, \Psi, s$ as in Prop.~\ref{prop:flscness} and additionally assume that
  \begin{equation}
    \mathcal{B}_{\rho_l}(0) \cap \{(v^1,\ldots,v^l) | \sum_i v^i = 0)\} \subseteq \mathcal{D}_{\tmop{loc}} \subseteq \mathcal{B}_{\rho_u}(0)\label{eq:minexistence}
  \end{equation}
  for some $0 < \rho_l \leqslant \rho_u$. Then the problem
  \begin{equation}
    \min_{u \in \mathcal{C}} f (u)
  \end{equation}
  has at least one minimizer.
\end{proposition}

\begin{proof}
  From the inner and outer bounds it holds that $\rho_l \|z\| \leqslant \Psi (z) \leqslant \rho_u \|z\|$ for any $z = (z^1|\ldots|z^l) \in \mathbb{R}^{d \times l}$ with $z^1 + \ldots + z^l = 0$. Moreover, the constraint
  $u \in \mathcal{C}$ implies that $D u = (D u_1 | \ldots | D u_l)$ satisfies $D u_1 + \ldots + D u_l = 0$.
  Combining this with the positive homogeneity of $\Psi$, it follows from \eqref{eq:poshomogju} that
  \begin{eqnarray}
    0 \leqslant \rho_l \tmop{TV} (u) \leqslant J (u) & \leqslant & \rho_u \tmop{TV} (u)\,.\label{eq:regboundbelow}
  \end{eqnarray}
  From
  \begin{equation}
    \int_{\Omega} \langle u, s \rangle d x \geqslant
        - \int_{\Omega} \|u(x)\|_{\infty} \|s (x)\|_1 d x\,,\label{eq:datatermboundbelow}
  \end{equation}
  the fact that $s \in L^1 (\Omega)^l$, and boundedness of $\Omega$ and $\Delta_l$, it follows that the data term is bounded from below on $\mathcal{C}$.
  
  We now show coercivity of $f$ with respect to the $\tmop{BV}$ norm: Let $(u^{(k)}) \subset \mathcal{C}$ with $\|u^{(k)} \|_1 + \tmop{TV} (u^{(k)}) \rightarrow \infty$. As the data term is bounded from below,
and using the fact that $J (u^{(k)}) \geqslant \rho_l \tmop{TV} (u^{(k)})$, it follows that $f (u^{(k)}) \rightarrow + \infty$. Thus $f$ is coercive.
  
  Equations \eqref{eq:regboundbelow} and \eqref{eq:datatermboundbelow} also show that $f$ is bounded from below. Thus we can choose a minimizing sequence $(u^{(k)})$. Due to the coercivity, the sequence $\|u^{(k)} \|_1 + \tmop{TV} (u^{(k)})$ must then be bounded from above. From this and {\cite[Thm.~3.23]{Ambrosio2000}} we conclude that there is a subsequence of $(u^{(k)})$ weakly*- (and thus
  $L^1$-) converging to some $u \in \tmop{BV} (\Omega)^l$. With the lower
  semicontinuity from Prop.~\ref{prop:flscness} and closedness of $\mathcal{C}$ with respect to
  $L^1$ convergence, existence of a minimizer follows.
\end{proof}


\subsection{Relation to the Interaction Potential}

To relate such $J$ to the labeling problem in view of (P3), we have the following

\begin{proposition}\label{prop:redbinary}
  Let $\Psi = \sigma_{\mathcal{D}_{\tmop{loc}}} \tmop{and}$ $J (u) =
  \int_{\Omega} \Psi (D u)$ as in Prop.~\ref{prop:flscness}. For some $u' \in \tmop{BV} (\Omega)$ and vectors $a, b \in \Delta_l$, let
  $u (x) = (1 - u' (x)) a + u' (x) b$. Then for any vector $y \in \mathbbm{R}^d$ with $\|y\|= 1$,
  \begin{eqnarray}
    J (u) & = & \Psi \left( y (b - a)^{\top} \right) \tmop{TV}(u')
            =   \left( \sup_{v \in \mathcal{D}_{\tmop{loc}}} \|v \left( b - a \right) \| \right) \tmop{TV} (u')\,. 
    \label{eq:jupsirep}
  \end{eqnarray}
  In particular, if $\Psi \left( y (e^i - e^j)^{\top} \right) = \pot(i,j)$ for
  some $y$ with $\|y\| = 1$, then $J$ fulfills (P3).
\end{proposition}

\begin{proof}
To show the first equality, \eqref{eq:poshomogju} implies
\begin{eqnarray}
  J (u) & = & \int_{\Omega} \Psi \left( \frac{D u}{|D u|} \right) |D u| \\
  & = & \int_{\Omega} \Psi \left( \frac{D (a + u' (b - a))}{|D (a + u' (b -
  a)) |} \right) |D (a + u' (b - a)) | \\
  & = & \int_{\Omega} \Psi \left( \frac{(D u') (b - a)^{\top}}{| (D u') (b -
  a)^{\top} |} \right) | (D u') (b - a)^{\top} | 
\end{eqnarray}
We make use of the property $| (D u') (b - a)^{\top} | = |D u' |\|b - a\|$, which is
a direct consequence of the definition of the total variation measure
and the fact that $\|w (b - a)^{\top} \|=\|w\|\|b - a\|$ for any
vector $w \in \mathbbm{R}^d$ (note that $a,b \in \mathbbm{R}^l$ are also vectors). Therefore
\begin{eqnarray}
  J (u) & = & \int_{\Omega} \Psi \left( \frac{(D u') (b - a)^{\top}}{|D u'
  |\|b - a\|} \right) |D u' |\|b - a\|, 
\end{eqnarray}
which by positive homogeneity of $\Psi$ implies
\begin{eqnarray}
  J (u) & = & \int_{\Omega} \Psi \left( \frac{D u'}{|D u' |} (b - a)^{\top}
  \right) |D u' |. 
\end{eqnarray}
Since the density function $D u' / |D u' |$ assumes values in $S^{d - 1}$,
there exists, for a.e. $x \in \Omega$, a rotation matrix mapping 
$(D u' / |D u' |) (x)$ to $y$. Together with the rotational invariance
of $\mathcal{D}_{\tmop{loc}}$ from \eqref{eq:dlocrotinvar} this implies
\begin{eqnarray}
  J (u) & = & \int_{\Omega} \Psi \left( y (b - a)^{\top} \right) |D u' | =
  \Psi \left( y (b - a)^{\top} \right) \tmop{TV} (u')\,,
\end{eqnarray}
which proves the first equality in \eqref{eq:jupsirep}. The second equality can be seen as follows:
  \begin{eqnarray}
    r & \assign & \sup_{v \in \mathcal{D}_{\tmop{loc}}} \|v \left( b - a \right) \|\\
    & = & \sup_{v \in \mathcal{D}_{\tmop{loc}}} \sup_{z \in \mathbbm{R}^d,
    \|z\| \leqslant 1} \langle z, v (b - a) \rangle\\
    & = & \sup_{z \in \mathbbm{R}^d, \|z\| \leqslant 1} \sup_{v \in
    \mathcal{D}_{\tmop{loc}}} \langle z, v (b - a) \rangle\,.
  \end{eqnarray}
  Denote by $R_z$ a rotation matrix mapping $z$ to $y$, i.e. $R_z z = y$, then
  \begin{eqnarray}
    r & = & \sup_{z \in \mathbbm{R}^d, \|z\| \leqslant 1}
            \sup_{v \in \mathcal{D}_{\tmop{loc}}} \langle R_z z, R_z v (b - a) \rangle\\
      & = & \sup_{z \in \mathbbm{R}^d, \|z\| \leqslant 1}
            \sup_{v' \in R_z \mathcal{D}_{\tmop{loc}}} \langle y, v' (b - a) \rangle\,.
  \end{eqnarray}
  The rotational invariance of $\mathcal{D}_{\tmop{loc}}$ provides $R_z \mathcal{D}_{\tmop{loc}} = \mathcal{D}_{\tmop{loc}}$, therefore
  \begin{eqnarray}
    r & = & \sup_{z \in \mathbbm{R}^d, \|z\| \leqslant 1}
            \sup_{v \in \mathcal{D}_{\tmop{loc}}} \langle y, v (b - a) \rangle\\
      & = & \sup_{v \in \mathcal{D}_{\tmop{loc}}} \langle y, v (b - a) \rangle
       =  \Psi (y (b - a)^{\top})\,.
  \end{eqnarray}
\end{proof}

As a consequence, if the relaxed multiclass formulation is restricted to two
classes by parametrizing $u = (1 - u') a + u' b$ for $u'(x) \in [0, 1]$, it
essentially reduces to the scalar continuous cut problem: Solving
\begin{equation}
  \min_{
  \begin{array}{c}
  {\scriptstyle u' \in \tmop{BV} (\Omega),}\\
  {\scriptstyle u'(x) \in [0,1], \mathcal{L}^d \text{-a.e. } x \in \Omega}
  \end{array}
  }
  \int_{\Omega} \langle (1 - u') a + u' b, s \rangle d x + J (u)
\end{equation}
is equivalent to solving
\begin{equation}
  \min_{
  \begin{array}{c}
  {\scriptstyle u' \in \tmop{BV} (\Omega),}\\
  {\scriptstyle u'(x) \in [0,1], \mathcal{L}^d \text{-a.e. } x \in \Omega}
  \end{array}
  }
  \int_{\Omega} u' (b - a) d x + \Psi (y (b - a)^{\top}) \tmop{TV} (u')\,
\end{equation}
which is just the classical binary continuous cut approach with data $(b - a)$
and regularizer weight $\Psi (y (b - a)^{\top})$, where $y \in \mathbbm{R}^d$ is some arbitrary
unit vector. For the multiclass case, assume that
\begin{eqnarray}
  u = u_P & = & e^1 \chi_{P^1} + \ldots + e^l \chi_{P^l} \label{eq:uppart}
\end{eqnarray}
for some partition $P^1 \cup \ldots \cup P^l = \Omega$ with $\tmop{Per}(P^i) < \infty, i=1,\ldots,l$.
Then the absolutely continuous and Cantor parts vanish \cite[Thm.~3.59, Thm.~3.84, Rem.~4.22]{Ambrosio2000}, and
only the jump part remains:
\begin{eqnarray}
  J (u_P) & = & \int_{S_{u_P}} \Psi \left( \nu_{u_P} \left( u_{P +} -
                u_{P -} \right)^{\top} \right) d\mathcal{H}^{d - 1}\,,
\end{eqnarray}
where $S_{u_P} = \bigcup_{i = 1, \ldots, l} \partial P^i$ is the union of the interfaces
between regions.
Define $i (x), j (x)$ such that $u_{P +} (x) = e^{i
(x)}$ and $u_{P -} (x) = e^{j (x)}$. Then
\begin{eqnarray}
  J \left( u_P \right) & = & \int_{S_{u_P}} \Psi \left( \nu_{u_P} \left(
  e^{i (x)} - e^{j (x)} \right)^{\top} \right) d\mathcal{H}^{d - 1}\,.
\end{eqnarray}
By rotational invariance,
\begin{eqnarray}
  J \left( u_P \right) & = & \int_{S_{u_P}} \Psi \left( y  \left( e^{i (x)} - e^{j (x)} \right)^{\top} \right)
  d\mathcal{H}^{d - 1}\,.  \label{eq:jup1}
\end{eqnarray}
for some $y$ with $\|y\| = 1$. Thus the regularizer locally penalizes jumps between labels $i$ and $j$ along an interface with the interface length, multiplied by the factor $\Psi ( y  ( e^{i} - e^{j} )^{\top})$ depending on the labels of the adjoining regions.

The question is how to choose the set $\mathcal{D}_{\tmop{loc}}$ such that $\Psi (y (e^i - e^j)^{\top}) = \pot(i,j)$
for a given interaction potential $\pot$. We will consider two approaches which differ
with respect to expressiveness and simplicity of use: In the \emph{local envelope} approach (Sect.~\ref{sec:envelopemethod}), $\mathcal{D}_{\tmop{loc}}$ is
chosen as large as possible. In turn, $J$ is as large as possible with the integral formulation \eqref{eq:generaljurep}. This prevents
introducing artificial minima generated by the relaxation, and potentially keeps minimizers of the relaxed
problem close to minimizers of the discrete problem. However, $\mathcal{D}_{\tmop{loc}}$ is only implicitly defined,
which complicates optimization. In contrast, in the \emph{embedding} approach (Sect.~\ref{sec:euclmetricmethod}), $\mathcal{D}_{\tmop{loc}}$ is
simpler at the cost of being able to represent only a subset of all metric potentials $\pot$. For an illustration
of the two approaches, see Fig.~\ref{fig:supportingsets}.

\begin{figure}
\centering
	{
	\psfrag{a}[br][br]{$(-1,1,0)$}
	\psfrag{b}[bl][bl]{$(1,-1,0)$}
	\psfrag{c}[br][br]{$(-1,0,1)$}
	\psfrag{d}[bl][bl]{$(1,0,-1)$}
	\psfrag{e}[t][t]{$(0,-1,1)$}
	\psfrag{f}[b][b]{$(0,1,-1)$}
	\includegraphics[width=5.0cm]{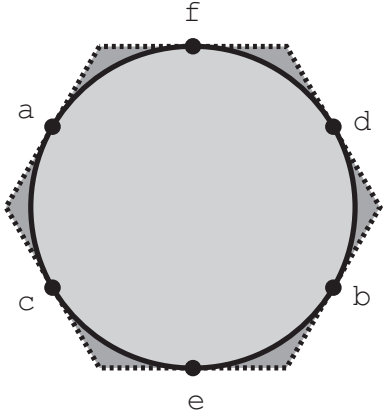}
	} 
  \caption{\label{fig:supportingsets}Illustration of the set
  $\mathcal{D}_{\tmop{loc}}$ used to build the regularizer for the uniform
  distance, $\pot (i, j) = 0$ iff $i = j$ and $\pot(i,j)=1$ otherwise, for $l = 3$ and $d = 1$,
  i.e.~in a one-dimensional space. Shown is a cut through the $z^1+z^2+z^3=0$ plane; the labels correspond
  to the points $e^i-e^j$ with $i \neq j$. The local envelope method leads to a larger set
  $\mathcal{D}_{\tmop{loc}}$ (dashed) than the Euclidean metric method (solid).
  This improves the tightness of the relaxation, but requires more expensive projection steps.}
\end{figure}


\subsection{Local Envelope Method}\label{sec:envelopemethod}

Chambolle et al. {\cite{Chambolle2008}} proposed an interesting approach for
potentials $\pot$ of the form $\pot (i, j) = \gamma (|i - j|)$ for a positive concave function
$\gamma$. The approach is derived by specifying the value of $J$ on discrete
$u$ only and then constructing an approximation of the convex envelope by
pulling the convexification into the integral.
This potentially generates tight extensions and thus one may hope that the
convexification process does not generate too many artificial non-discrete
solutions.

We propose to extend this approach to \emph{arbitrary metric} $\pot$ by setting (Fig.~\ref{fig:supportingsets})
\begin{eqnarray}
  \mathcal{D}_{\tmop{loc}}^{\pot} & \assign & \bigcap_{i \neq j} \{  v = \left( v^1, \ldots, v^l \right) \in \mathbbm{R}^{d \times l} | 
  \|v^i - v^j \| \leqslant \pot (i, j), \sum_k v^k = 0 \} \label{eq:dlocd}
\end{eqnarray}
for some given interface potential $\pot (i, j)$. By definition
$\mathcal{D}_{\tmop{loc}}^{\pot}$ is rotationally invariant, and by the considerations
in Sect.~\ref{sec:axiomatic} we assume $\pot$ to be a metric.
Then the inner and outer bound assumptions required for existence of a minimizer in Prop.~\ref{prop:minexistence} 
are satisfied. Moreover, $\pot$ can be reconstructed from $\mathcal{D}_{\tmop{loc}}^{\pot}$:

\begin{proposition}
  \label{prop:supij}Let $d : \{1, \ldots, l\}^2 \rightarrow \mathbbm{R}_{\geqslant 0}$ be a metric.
  Then for any $i, j$,
  \begin{eqnarray}
    \sup_{v \in \mathcal{D}_{\tmop{loc}}^{\pot}} \left( \left( v^i \right)_1 -
    \left( v^j \right)_1 \right) & = & \pot (i, j)\,. 
  \end{eqnarray}
\end{proposition}

\begin{proof}
  ``$\leqslant$'' follows from the definition \eqref{eq:dlocd}. ``$\geqslant$'' can be shown
  using a network flow argument: We have
  \begin{eqnarray}
    & & \sup_{v \in \mathcal{D}_{\tmop{loc}}^{\pot}} \left( \left( v^i \right)_1 - \left(v^j \right)_1 \right) \\
    & \geqslant & \sup \{ p_i - p_j \; | \; p \in \mathbbm{R}^l, \sum_k p_k = 0,
    \forall i', j' : p_{i'} - p_{j'} \leqslant \pot(i', j') \}\label{eq:flowfirstset}\\
    & \overset{(\ast)}{=} & \sup \{ p_i - p_j \; | \; p \in \mathbbm{R}^l,
    \forall i', j' : p_{i'} - p_{j'} \leqslant \pot(i', j') \}\label{eq:flowsecondset}\\
    & \overset{(\ast\ast)}{=} & \pot (i, j)\,.
  \end{eqnarray}
  Equality $(\ast)$ holds since each $p$ in the set in \eqref{eq:flowsecondset} can be
  associated with $\tilde{p} \assign p - \frac{1}{l} \sum_k p_k$, which is
  contained in the set in \eqref{eq:flowfirstset} and satisfies
  and $p_i - p_j = \tilde{p}_i - \tilde{p}_j$.
  The last equality $(\ast\ast)$ follows from {\cite[5.1]{Murota2003}}
  with the notation $\gamma = \pot$ (and $\bar{\gamma} = \pot$, since $\pot$ is a metric and
  therefore the triangle inequality implies that the length of the shortest path from $i$ to
  $j$ is always $\pot (i, j)$).
\end{proof}

The final result of this section is the following:
\begin{proposition}\label{prop:jdsatisfies}
  Let $d : \mathbbm{R}^{l \times l} \rightarrow \mathbbm{R}_{\geqslant 0}$ be a metric.
  Define $\mathcal{D}_{\tmop{loc}} \assign \mathcal{D}_{\tmop{loc}}^{\pot}$ as in~(\ref{eq:dlocd}), $\Psi_{\pot} \assign \sigma_{\mathcal{D}_{\tmop{loc}}^{\pot}}$ and
  \begin{eqnarray}
    J_{\pot} \assign \int_{\Omega} \Psi_{\pot}(D u)\label{eq:jddef}
  \end{eqnarray}  
  as in (\ref{eq:psidu}). Then $J_{\pot}$ satisfies (P1)--(P3).
\end{proposition}

\begin{proof}
  (P1) and (P2) are clear from the definition of $J_{\pot}$. (P3) follows directly from Prop.~\ref{prop:supij} and Prop.~\ref{prop:redbinary} with $y = e^1$.
\end{proof}

Defining $\mathcal{D}_{\tmop{loc}}^{\pot}$ as in (\ref{eq:dlocd}) provides us with a way to extend the desired
regularizer for {\tmem{any}} metric $\pot$ to non-discrete $u \in \mathcal{C}$ via
(\ref{eq:psidu}). The price to pay is that there is no simple closed
expression for $\Psi$ and thus for $J_{\pot}$, which potentially complicates optimization.
Note that in order to define $\mathcal{D}_{\tmop{loc}}^{\pot}$, $\pot$ does not have to be a metric. However Prop.~\ref{prop:supij}
then does not hold in general, so $J$ is not a true extension of the desired regularizer, although it still
provides a lower bound.

\subsection{Euclidean Metric Method} \label{sec:euclmetricmethod}

In this section, we consider a regularizer which is less powerful but more
efficient to evaluate. The classical total variation for vector-valued $u$
as defined in (\ref{eq:tvdef-dual}) is
\begin{equation}
  \tmop{TV}(u) = \int_{\Omega} \| D u \|\,.
\end{equation}
This classical definition has also been used in color denoising and is
also referred to as $\tmop{MTV}$ {\cite{Sapiro1996,Duval2008}}. We propose to
extend this definition by choosing an {\tmem{embedding matrix}} $A \in
\mathbbm{R}^{k \times l}$ for some $k \leqslant l$, and defining
\begin{eqnarray}
  J_A (u) & \assign & \tmop{TV} (Au)\,. \label{eq:tva}
\end{eqnarray}
This corresponds to substituting the Frobenius matrix norm on the distributional gradient
with a linearly weighted variant. In the framework of (\ref{eq:psidu}), it amounts to setting
$\mathcal{D}_{\tmop{loc}} = \mathcal{D}_{\tmop{loc}}^A$ (cf. Fig.~\ref{fig:supportingsets}) with
\begin{eqnarray}
  \mathcal{D}_{\tmop{loc}}^A & \assign & \{v' A | v' \in
  \mathbbm{R}^{d \times k}, \|v' \| \leqslant 1\} = \mathcal{B}_1 (0) A \label{eq:dlocadef}
\end{eqnarray}
Clearly, $0 \in \mathcal{D}_{\tmop{loc}}^A$ and
\begin{eqnarray}
  \Psi (z) & = & \sigma_{\mathcal{D}_{\tmop{loc}}^A}(z) = \sup_{v' \in \mathcal{B}_1 (0) A} \langle
  z, v' \rangle = \sup_{v \in \mathcal{B}_1 (0)} \langle z, v A \rangle\\
  & = & \sup_{v \in \mathcal{B}_1 (0)} \langle z A^{\top}, v \rangle
  =\|z A^{\top}\|\,. \label{eq:psiza}
\end{eqnarray}
In particular, we formally have
\begin{eqnarray}
  \Psi (D u) & = & \|(D u) A^{\top}\| = \|D(A u)\|\,,
\end{eqnarray}
as $u \mapsto D u$ is linear in
$u$. To clarify the definition, we may rewrite this to
\begin{eqnarray}
  \tmop{TV}_A (u) & = & \int_{\Omega} \sqrt{\|D_1 u\|_A^2 + \ldots +\|D_d u\|_A^2} \,,
\end{eqnarray}
where $\|v\|_A \assign (v^{\top} A^{\top} Av)^{1 / 2}$. In this sense, the approach can
be understood as replacing the Euclidean norm by a linearly weighted variant.

It remains to show for which interaction potentials $\pot$ assumption (P3) can be satisfied. The next
proposition shows that this is possible for the class of \emph{Euclidean} metrics.

\begin{proposition}
	Let $\pot$ be an Euclidean metric, i.e. there exist $k \in \mathbbm{N}$ as well as $a^1, \ldots, a^l \in \mathbbm{R}^k$ such that $\pot (i, j) =\|a^i - a^j \|$. Then for $A = (a^1|\ldots|a^l)$, $J_A \assign \tmop{TV}_A$ satisfies (P1)--(P3).
\end{proposition}

\begin{proof}
 (P1) and (P2) are clearly satisfied. In order to show (P3) we
 apply Prop.~\ref{prop:redbinary} and assume $\|y\|=1$ to obtain with (\ref{eq:psiza})
 \begin{eqnarray}
 	 \Psi(y (e^i - e^j)^{\top}) & = & \| y (e^i - e^j)^{\top} A^{\top} \| \\
 	 & = & \| y (a^i - a^j)^{\top} \| = \|a^i - a^j\|\,. \nonumber
 \end{eqnarray}
\end{proof}

The class of Euclidean metrics comprises some important special cases:
\begin{itemize}
  \item The {\tmem{uniform}}, {\tmem{discrete}} or {\tmem{Potts}} metric as
  also considered in {\cite{Zach2008,Lellmann2009}} and as a special case in
  {\cite{Kleinberg1999,Komodakis2007}}. Here $\pot (i, j) = 0$ iff $i = j$ and $ \pot (i, j) = 1$ in any other case, which corresponds to $A = (1 / \sqrt{2}) I$.
  
  \item The {\tmem{linear}} (label) metric, $\pot (i, j) = c|i - j|$, with $A =
  (c, 2 c, \ldots, lc)$. This regularizer is suitable to problems where the
  labels can be naturally ordered, e.g. depth from stereo or grayscale image
  denoising.
  
  \item More generally, if label $i$ corresponds to a prototypical vector
  $z^i$ in $k$-dimensional feature space, and the Euclidean norm is an
  appropriate metric on the features, it is natural to set $\pot (i, j) =\|z^i - z^j \|$,
  which is Euclidean by construction. This corresponds to a
  regularization in feature space, rather than in ``label space''.
\end{itemize}
Note that the boundedness assumption involving $\rho_{l}$ required for the existence proof (Prop.~\ref{prop:minexistence}) is only fulfilled if the kernel of $A$ is sufficiently small,
i.e. $\ker A \subseteq \{t e | t \in \mathbbm{R}\}$, with $e=(1,\ldots,1)^{\top} \in  \mathbbm{R}^{l}$: Otherwise,
$\mathcal{D}_{\tmop{loc}}^A = \left(\mathcal{B}_{1}(0) A\right) \cap \{(v^1,\ldots,v^l)|\sum_i v^i=0\}$
is contained in a subspace of at most dimension $(l - 2) d$, and \eqref{eq:minexistence} cannot be satisfied
for any $\rho_l > 0$.
Thus if $\pot$ is a degenerate Euclidean metric which can be represented by an embedding into a lower-dimensional space, as is the case with the linear metric, it has to be regularized for the existence result in Prop.~\ref{prop:minexistence} to hold. This can for example be achieved by choosing an orthogonal basis $(b^1,\ldots,b^j)$ of $\ker A$, where $j = \tmop{dim} \ker A$, and substituting $A$ with $A' \assign (A^{\top},\varepsilon b^1,\ldots, \varepsilon b^j)^{\top}$, enlarging $k$ as required. However these observations are mostly of theoretical
interest, as the existence of minimizers for the discretized problem follows already
from compactness of the (finite-dimensional) discretized constraint set.

Non-Euclidean $\pot$, such as the \tmtextit{truncated linear
metric}, $\pot (i, j) = \min \{2, |i - j|\}$, cannot be represented exactly by
$\tmop{TV}_A$. In the following we will demonstrate how to construct approximations for these cases.

Assume that $\pot$ is an arbitrary metric with squared matrix representation
$D\in \mathbbm{R}^{l \times l}$, i.e.~$D_{ij} = \pot (i, j)^2$. Then it is known
{\cite{Borg2005}} that $\pot$ is Euclidean iff for $C \assign I -
\frac{1}{l} ee^{\top}$, the matrix $T \assign - \frac{1}{2} CDC$ is positive
semidefinite. In this case, $\pot$ is called \emph{Euclidean distance matrix}, and 
$A$ can be found by factorizing $T = A^{\top} A$.
If $T$ is not positive semidefinite, setting the nonnegative eigenvalues in
$T$ to zero yields an Euclidean approximation. This method is known as
{\tmem{classical scaling}} {\cite{Borg2005}} and does not necessarily give
good absolute error bounds.
\begin{figure}[tb]
	\centering
	\includegraphics[width=.80\columnwidth]{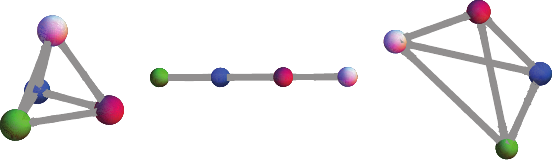}
  \caption{\label{fig:emb3d}Euclidean embeddings into $\mathbbm{R}^3$ for several interaction potentials with four classes. \textbf{Left to right:} Potts; linear metric; non-Euclidean truncated linear metric.
  The vertices correspond to the columns $a^1,\ldots,a^l$ of the embedding matrix $A$.
  For the truncated linear metric an optimal approximate embedding was
  computed as outlined in Sect.~\ref{sec:euclmetricmethod} with the
  matrix norm $\|X\|_M \assign \max_{i, j} |X_{i j} |$.}
\end{figure}%

More generally, for some non-metric, nonnegative $\pot$, we can formulate the problem of finding the
``closest'' Euclidean distance matrix $E$ as the minimization problem of a
matrix norm $\|E - D\|_M$ over all $E \in \mathcal{E}_l$, the set of $l \times
l$ Euclidean distance matrices. Fortunately, there is a linear bijection $B :
\mathcal{P}_{l - 1} \rightarrow \mathcal{E}_l$ between $\mathcal{E}_l$ and the
space of positive semidefinite $(l - 1) \times (l - 1)$ matrices
$\mathcal{P}_{l - 1}$ {\cite{Gower1985,Johnson1995}}. This allows us to
rewrite our problem as a {\tmem{semidefinite program}} {\cite[(p.534--541)]{Wolkowicz2000}}
\begin{equation}
  \inf_{S \in \mathcal{P}_{l - 1}} \|B (S) - D\|_M\,.\label{eq:sdpapprox}
\end{equation}
Problem~\eqref{eq:sdpapprox} can be solved using available SDP solvers. Then $E
= B (S) \in \mathcal{E}_l$, and $A$ can be extracted by factorizing $-
\frac{1}{2} CEC$. Since both $E$ and $D$ are explicitly known,
$\varepsilon_E \assign \max_{i, j} |(E_{ij})^{1/2} - (D_{ij})^{1/2} |$
can be explicitly computed and provides an a posteriori bound on the maximum distance error.
Fig.~\ref{fig:emb3d} shows a visualization of
some embeddings for a four-class problem. In many cases, in particular when
the number of classes is large, the Euclidean embedding provides a good approximation
for non-Euclidean metrics (Fig.~\ref{fig:penguin-approxcutlinear}).
\begin{figure}[tb]
\centering
\includegraphics[width=.40\columnwidth]{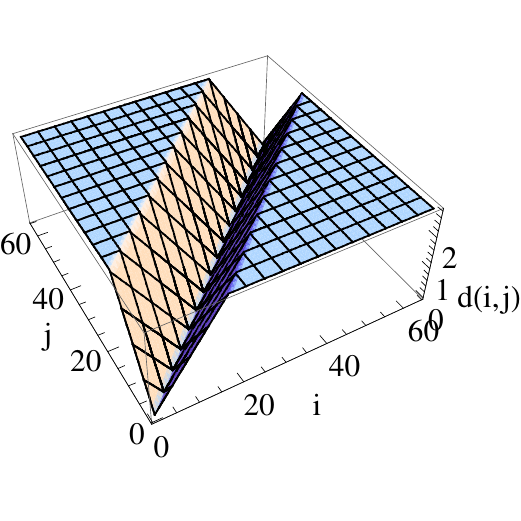}
\includegraphics[width=.40\columnwidth]{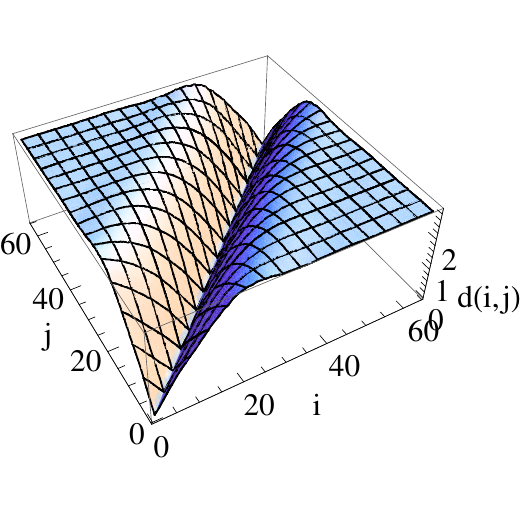}\\
\includegraphics[width=.70\columnwidth]{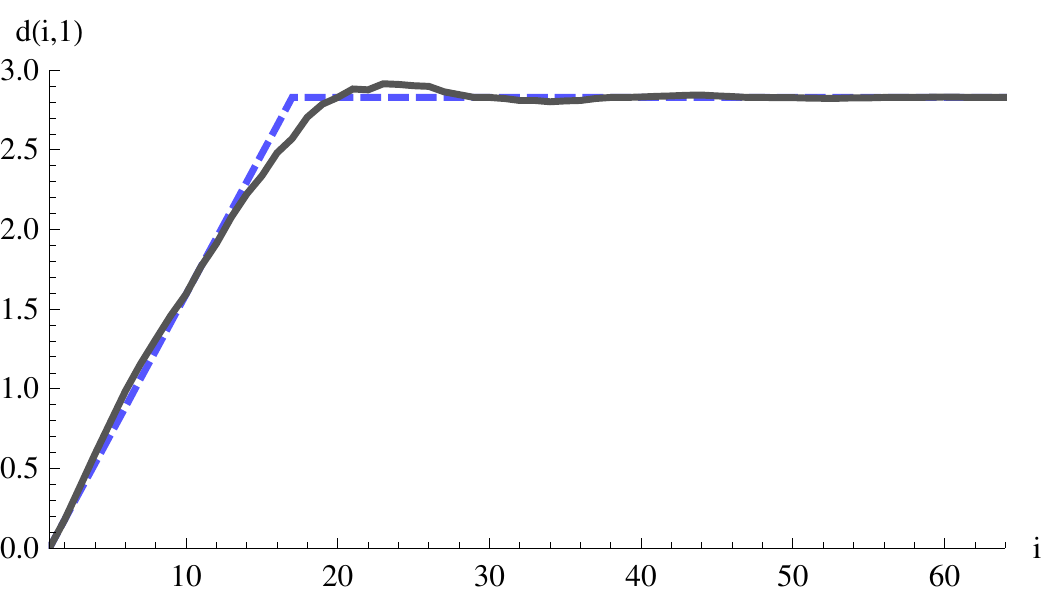}
\caption{Euclidean approximation of the Non-Euclidean truncated linear metric with interaction potential $\pot(i,j) = \sqrt{2}/8 \min\{|i-j|,16\}$.
\textbf{Left to Right:} Original potential for $64$ classes; potential after Euclidean approximation. \textbf{Bottom:} cross section of the original (dashed) and approximated (solid) metric at $i=1$. The approximation was computed using semidefinite programming as outlined in Sect.~\ref{sec:euclmetricmethod}. It represents the closest Euclidean metric with respect to the matrix norm $\|X-Y\|_M \assign \sum_{i,j} |X_{i j} - Y_{i j}|$. The maximal error with respect to the original potential is $\varepsilon_E = 0.2720$.
}%
\label{fig:penguin-approxcutlinear}%
\end{figure}

Based on the embedding matrices computed in this way, the Euclidean distance
approach can be used to solve approximations of the labeling problem with
arbitrary metric interaction potentials, with the
advantage of having a closed expression for the regularizer.


\section{Discretized Problem}\label{sec:discrete}


\subsection{Saddle Point Formulation}

We now turn to solving the discretization of the relaxed problem (\ref{eq:fullproblemdef}). In order
to formulate generic algorithms, we study the bilinear saddle point problem,
\begin{eqnarray}
  & & \min_{u \in \mathcal{C}} \max_{v \in \mathcal{D}} g \left( u, v \right)\,, 
  \quad g (u, v) \assign \langle u, s \rangle + \langle Lu, v \rangle - \langle b, v \rangle\,. \label{eq:discretemodelx}
\end{eqnarray}
As will be shown in Sect.~\ref{sec:discrchap}--\ref{sec:discr-eucl}, this covers both
$J_{\pot}$ \eqref{eq:jddef} and $J_A$ \eqref{eq:tva}
as well as many other -- even non-uniform and non-isotropic -- regularizers.

In a slight abuse of notation, we will denote by $u, s \in \mathbbm{R}^n$ also
the discretizations of $u$ and $s$ on a uniform grid. Furthermore we require a linear operator $L \in \mathbbm{R}^{m \times n}$, a vector $b \in \mathbbm{R}^m$ for some $m, n \in \mathbbm{N}$,
and bounded closed convex sets $\mathcal{C} \subseteq \mathbbm{R}^n,
\mathcal{D} \subseteq \mathbbm{R}^m$. Intuitively, $L$ discretizes the gradient operator
and $\mathcal{D}$ corresponds to $\mathcal{D}_{\tmop{loc}}$, i.e. specifies $\Psi$ in a
dual formulation.
The primal and dual objectives are
\begin{equation}
  f (u) \assign \max_{v \in \mathcal{D}} g \left( u, v \right) \quad \text{and} \quad
  f_D (v) \assign \min_{u \in \mathcal{C}} g (u, v)\,, \label{eq:primaldualobj}
\end{equation}
respectively. The dual problem then consists of maximizing $f_D(v)$ over $\mathcal{D}$.  As $\mathcal{C}$ and $\mathcal{D}$ are bounded, it follows from {\cite[Cor.~37.6.2]{Rockafellar1970}} that a saddle point $(u^{\ast},
v^{\ast})$ of $g$ exists. With {\cite[Lemma 36.2]{Rockafellar1970}}, this implies strong duality, i.e.
\begin{equation}
\min_{u \in \mathcal{C}} f (u)
= f (u^{\ast})
= g (u^{\ast}, v^{\ast})
= f_D (v^{\ast})
= \max_{v \in \mathcal{D}} f_D (v)\,.
\end{equation}
In our case, $\mathcal{C}, \mathcal{D}$ exhibit a specific product structure, which
allows to compute $f_D$ as well as the orthogonal projections
$\Pi_{\mathcal{C}}$ and $\Pi_{\mathcal{D}}$ efficiently, a fact that will
prove important in the algorithmic part. The evaluation of the primal objective $f$ on
the other hand can be more difficult, depending on the definition of $\mathcal{D}_{\tmop{loc}}$ resp.
$\mathcal{D}$, but is not required by the optimizer.
Note that in the discrete framework, we may easily substitute non-homogeneous,
spatially varying or even nonlocal regularizers by choosing $L$ and $b$
appropriately.


\subsection{Discretization}\label{sec:discrchap}

We discretize $\Omega$ by a regular grid, $\Omega =\{1, \ldots, n_1 \} \times
\cdots \times \{1, \ldots, n_d \} \subseteq \mathbbm{R}^d$ for $d \in \mathbbm{N}$, consisting of $n \assign | \Omega |$ pixels
$(x^1,\ldots,x^n)$ in lexicographical order.
We represent $u$ by its values $u = (u^1|\ldots|u^l) \in \mathbbm{R}^{n \times l}$ at these points,
and denote by $u^{x^j} \in \mathbbm{R}^l, j=1,\ldots,n,$ the $j$-th row of $u$, corresponding
to $u(x^j)$. The multidimensional image space $\mathbbm{R}^{n \times l}$ is equipped with the Euclidean inner product $\langle
\cdot, \cdot \rangle$ over the vectorized elements. We identify $u \in \mathbbm{R}^{n \times l}$
with the vector in $\mathbbm{R}^{nl}$ obtained by concatenating the columns.

Let $\mathrm{\tmop{grad}} : = ( \mathrm{\tmop{grad}}_1^{\top}|\ldots|
\mathrm{\tmop{grad}}_d^{\top})^{\top}$ be the $d$-dimensional forward
differences gradient operator for Neumann boundary conditions. Then
$\mathrm{\tmop{div}} : = - \mathrm{\tmop{grad}}^{\top}$ is the backward
differences divergence operator for Dirichlet boundary conditions. These
operators extend to $\mathbbm{R}^{n \times l}$ via
$\mathrm{\mathrm{\tmop{Grad}}} : = (I_l \otimes \mathrm{\tmop{grad}})$,
$\mathrm{\tmop{Div}} : = (I_l \otimes \mathrm{\tmop{div}})$.
We define, for some $k \geqslant 1$ as will be specified below, the convex sets
\begin{eqnarray}
  &  & \mathcal{C} \assign \{u \in \mathbbm{R}^{n \times l} |u^{x^j} \in
  \Delta_l, j = 1, \ldots, n\}\,, \\
  &  & \mathcal{D} \assign \prod_{x \in \Omega} \mathcal{D}_{loc} \subseteq
  \mathbbm{R}^{n \times d \times k}\,. 
\end{eqnarray}
The set $\mathcal{D}_{\tmop{loc}}$ and the operator $L$ depend on the chosen
regularizer. Note that for $L \assign
\tmop{Grad}$, $k = l$ and
\begin{equation}
  \mathcal{D}_{loc} \assign \mathcal{D}_{\tmop{loc}}^I \assign \left\{ p =
  (p^1|\ldots|p^l) \in \mathbbm{R}^{d \times l} |\|p\| \leqslant 1
  \right\}\,,
\end{equation}
the primal objective in the saddle point formulation discretizes
the classical vector-valued total variation,
\begin{equation}
  \tmop{TV} (u) \assign \sigma_{\tmop{Div} \mathcal{D}_{}} (u) =
  \sigma_{\mathcal{D}} (\tmop{Grad} u) = \sum_{j = 1}^n \|G_{x^j} u\|\,,  \label{eq:tvdef-basic}
\end{equation}
where each of the $G_{x^j}$ is an $(ld) \times (nl)$ matrix composed of rows of $(
\mathrm{\tmop{Grad}})$ such that $G_{x^j} u$ constitutes the vectorized
discrete Jacobian of $u$ at the point $x^j$.

Projections on the set $\mathcal{C}$ are highly separable and can be computed
exactly in a finite number of steps {\cite[Alg.~4]{Michelot1986}}, see Alg.~\ref{alg:projsimplex}
for a summary of the pointwise operation.
\begin{algorithm}[tb]
\caption{Projection of $y\in\mathbbm{R}^l$ onto the standard simplex $\Delta_l$}\label{alg:projsimplex}
\begin{algorithmic}[1]
\STATE $Z^{(0)} \leftarrow \emptyset$, $y^{(0)} \leftarrow y \in \mathbbm{R}^l$, $k \leftarrow 0$.
\REPEAT
		\STATE $\tilde{y}^{(k+1)}_i \leftarrow \left\{	
\begin{array}{ll}
	0, & i \in Z^{(k)}\,,\\
	y^{(k)}_i - \frac{e^{\top} y^{(k)}}{l-|Z^{(k)}|}, & \text{otherwise}\,,
\end{array}\right.\, (i = 1,\ldots,l)\,.$
    \STATE $Z^{(k+1)} \leftarrow \{i \in \{1,\ldots,l\} | i \in Z^{(k)} \text{ or } \tilde{y}^{(k+1)}_i < 0\}$.
    \STATE $y^{(k+1)}_i \leftarrow \max \{\tilde{y}^{(k+1)}_i,0\}\quad(i = 1,\ldots,l)$.
    \STATE $k \leftarrow k+1$.
\UNTIL{$\tilde{y}^{(k+1)} \geqslant 0$}.
\end{algorithmic}
\end{algorithm}
The dual objective is
\begin{equation}
  f_D(v) = - \langle b,v \rangle + \min_{u \in \mathcal{C}} \langle u, L^{\top} v + s \rangle\,.  
\end{equation}
Since the minimum decouples spatially, and $\min_{y \in \Delta_l} \langle y, z \rangle = \tmop{vecmin}(z) \assign \min_{i} z_i$ for all $z\in\mathbb{R}^l$, the dual objective can always be evaluated by summing, over all points $x \in \Omega$, the minima over the components of $L^{\top} v + s$ corresponding to each point, denoted by $(L^{\top} v + s)_{x^j}$,
\begin{equation}
  f_D(v) = - \langle b,v \rangle + \sum_{j=1}^n \tmop{vecmin}\left((L^{\top} v + s)_{x^j}\right)\,.
\end{equation}
This fact is helpful if evaluating the primal objective is costly because the dual
set $\mathcal{D}_{\tmop{loc}}$ has a complicated structure.


\subsection{Specialization for the Local Envelope Method}\label{sec:discr-envelope}
For a metric interaction potential $\pot: \{1,\ldots,l\}^2\rightarrow \mathbbm{R}_{\geqslant 0}$, we set $k \assign l$ and
\begin{eqnarray}
  L & \assign & \tmop{Grad}\,, \nonumber\\
  \mathcal{D}_{\tmop{loc}} & \assign & \mathcal{D}_{\tmop{loc}}^{\pot} = \bigcap_{i \neq j} \{ v = \left( v^1|\ldots|v^l \right) \in \mathbbm{R}^{d \times l} |
  \|v^i - v^j \| \leqslant \pot (i, j) \}\label{eq:defdlocd}
\end{eqnarray}
as in (\ref{eq:dlocd}). We arrive at the saddle point form (\ref{eq:discretemodelx}) with
$\mathcal{C}$, $\mathcal{D}$, and $L$ defined as above, $m = nl$ and $b = 0$.
However due to the generality of the regularizer (cf. Prop.~\ref{prop:jdsatisfies}), the primal objective $f$ cannot be easily computed anymore (for the special case of three labels there is a derivation in {\cite{Chambolle2008}}).
Projections on $\mathcal{D}$ can be computed for all $x \in \Omega$ in
parallel, while projections on the individual sets $\mathcal{D}_{\tmop{loc}}^{\pot}$ can
be computed by the Dykstra algorithm, cf.~{\cite{Chambolle2008}} and Sect.~\ref{sec:projdualcons}.


\subsection{Specialization for the Euclidean Metric Method}\label{sec:discr-eucl}
For $A \in \mathbbm{R}^{k \times l}$ as in (\ref{eq:tva}), the Euclidean regularizer
can be obtained by setting
\begin{eqnarray}
  L & \assign & (\tmop{Grad}) (A \otimes I_n)\,, \nonumber\\
  \mathcal{D}_{\tmop{loc}} & \assign & \mathcal{D}_{\tmop{loc}}^I = \left\{ v
  = (v^1|\ldots|v^k) \in \mathbbm{R}^{d \times k} |\|v\| \leqslant 1 \right\}\,. \label{eq:defdloci}
\end{eqnarray}
Departing from the definition (\ref{eq:dlocadef}) of $\mathcal{D}_{\tmop{loc}}^A$,
$A$ is merged into $L$, as the optimization method will rely
on projections on $\mathcal{D}_{\tmop{loc}}$. Including $A$ into
$\mathcal{D}_{\tmop{loc}}$, i.e. by setting
\begin{equation}
 \mathcal{D}_{\tmop{loc}} \assign (A\otimes I_n)^{\top} \mathcal{D}_{\tmop{loc}}^I
\end{equation}
and $L \assign \tmop{Grad}$, would prevent computing the projection in closed form.
Projecting on the unit ball $\mathcal{D}_{\tmop{loc}}^I$ on the other hand is
trivial. The discretized regularizer can be explicitly evaluated, since
\begin{eqnarray}
  \Psi (z) & = & \|z A^{\top}\|\,.
\end{eqnarray}
Comparison to (\ref{eq:tvdef-basic}) yields
\begin{eqnarray}
  J_A (u) & = & \tmop{TV} ((A \otimes I_n) u)\,. \label{eq:tvadiscr}
\end{eqnarray}
We finally arrive at the form (\ref{eq:discretemodelx}) with $\mathcal{C}$,
$\mathcal{D}$, and $L$ defined as above, $m = nk$ and $b = 0$. Projections on
$\mathcal{D}$ are highly separable and thus can be computed
easily. The primal objective can be evaluated in closed form using (\ref{eq:tvadiscr}).


\subsection{Optimality}\label{sec:optimality}

If $f_D$ and $f$ can be explicitly computed, any $v \in \mathcal{D}$ provides an optimality bound on the primal objective,
with respect to the optimal value $f(u^{\ast})$, via the numerical \emph{primal-dual gap} $f(u) - f_D(v)$,
\begin{equation}
  0 \leqslant f (u) - f (u^{\ast}) \leqslant f (u) - f_D (v)\,.\label{eq:gap}
\end{equation}
Assuming $f$ and $f_D$ can be evaluated, the gap is a convenient stopping criterion.
To improve the scale invariance, it is often practical to stop on the \emph{relative gap}
\begin{equation}
	\frac{f(u) - f_D(v)}{f_D(v)} \label{eq:relativegap}
\end{equation}
instead, which gives a similar bound.
However convergence in the objective alone does not necessarily imply convergence in $u$, as the
minimizer of the original problem (\ref{eq:discretemodelx}) is generally not unique. This
stands in contrast to the well-known ROF-type problems \cite{Rudin1992}, where the
uniqueness is induced by a quadratic data term.

For some applications, after solving the relaxed problem a discrete solution
-- or ``hard'' labeling -- still needs to be recovered, i.e. the relaxed solution
needs to be {\tmem{binarized}}. For the continuous two-class case with the classical
$\tmop{TV}$ regularizer, {\cite{Nikolova2006}} showed that an exact solution
can be obtained by thresholding at almost any threshold. However, their
results do not seem to transfer to the multi-class case.

Another difficulty lies in the discretization: In order to apply the
thresholding theorem, a crucial ``coarea''-like property must hold for the
{\tmem{discretized}} problem {\cite{Chambolle2009}}, which holds for the graph-based pairwise
$\ell_1$-, but not the higher order $\ell_2$-discretization of the TV. Thus, even in the
two-class case, simple thresholding may lead to a suboptimal discrete solution.

Currently we are not aware of an a priori bound on the error introduced by the
binarization step in the general case. In practice, any dual feasible point
together with (\ref{eq:gap}) yields an a posteriori optimality bound: Let
$u^{(N)}, v^{(N)}$ be a pair of primal resp. dual feasible iterates,
$\bar{u}^{(N)}$ the result of the binarization step applied to $u^{(N)}$, and
$\bar{u}^{\ast}$ the optimal discrete solution. Then $\bar{u}^{(N)}$ is primal
feasible, and its suboptimality with respect to the optimal discrete solution is bounded from above by
\begin{eqnarray}
  & f ( \bar{u}^{(N)}) - f ( \bar{u}^{\ast}) \leqslant f ( \bar{u}^{(N)}) - f (u^{\ast})
  \leqslant f ( \bar{u}^{(N)}) - f_D (v^{(N)})\,. &
\end{eqnarray}
Computation of $f_D$, and in many cases also $f$, is efficient as outlined in Sect.~\ref{sec:discrchap}.


\subsection{Improving the Binarization}\label{sec:binarization}

There seems to be no obvious best choice for the binarization step. The
simplest choice is the \emph{first-max} approach: the label $\ell (x)$ is set to the index of the first
maximal component of the relaxed solution $\bar{u}^{(N)} (x)$. However, this
might lead to undesired side effects: Consider the segmentation of a grayscale
image with the three labels $1, 2, 3$ corresponding to the gray level intensity,
together with the linear distance $\pot (i, j) = |i - j|$, which is exactly
representable using the Euclidean distance approach with $A = (1\,2\,3)$.
Assume there is a region where $\bar{u}^{(N)} = (1 / 3 + \delta(x), 1 / 3, 1 / 3 - \delta(x))$ for some small noise
$\delta(x) \in \mathbbm{R}$. The most ``natural'' choice given the interpretation
as grayscale values is the constant labeling $\ell (x) = 2$. The first-max
approach gives $\ell (x) \in \{1, 3\}$, depending on the sign of $\delta(x)$,
which leads to a noisy final segmentation.

On closer inspection, the first-max approach -- which works well for the Potts
distance -- corresponds to choosing
\begin{eqnarray}
  \ell (x) & = & \arg \min_{\ell \in \{1, \ldots, l\}} \|u (x) - e^{\ell} \|\,, 
\end{eqnarray}
with the smallest $\ell$ chosen in case of ambiguity. We propose to extend this
to non-uniform distances by setting
\begin{eqnarray}
  \ell (x) & = & \arg \min_{\ell \in \{1, \ldots, l\}} \bar{\Psi} \left( u (x)
  - e^{\ell} \right)\,,  \label{eq:nonuniformbinarization}\\
  &  & \bar{\Psi} : \mathbbm{R}^l \rightarrow \mathbbm{R}, \bar{\Psi} (z)
  \assign \Psi \left( e^1 z^{\top} \right)\,. \nonumber
\end{eqnarray}
That is, we select the label corresponding to the nearest unit vector
{\tmem{with respect to $\bar{\Psi}$}} (note that instead of $e^1$ we could choose
any normalized vector as $\Psi$ is assumed to be rotationally invariant). In fact, for the linear distance
example above we have $\bar{\Psi} (z) = | - z_1 + z_3 |$. Thus
\begin{eqnarray}
  \bar{\Psi} (u (x) - e^1) & = & |1 - 2 \delta(x) |\,, \\
  \bar{\Psi} (u (x) - e^2) & = & |2 \delta(x) |\,, \nonumber\\
  \bar{\Psi} (u (x) - e^3) & = & |1 + 2 \delta(x) |\,, \nonumber
\end{eqnarray}
and for small $\delta$ we will get the stable and semantically correct choice $\ell(x) = 2$.
This method proved to work well in practice, and considerably reduced the
suboptimality introduced by the binarization step (Sect.~\ref{sec:tightness}). In case there is no closed
form expression of $\Psi$, it can be numerically approximated as outlined in Sect.~\ref{sec:objeval}.


\section{Optimization}\label{sec:optimization}

When optimizing the saddle point problem (\ref{eq:discretemodelx}), one must take care
of its large-scale nature and the inherent nonsmoothness of the objective.
While interior-point solvers are known to be very fast for small to medium sized
problems, they are not particularly suited well for massively parallel computation, such
as on the upcoming GPU platforms, due to the expensive inner Newton iterations.

We will instead focus on \emph{first order} methods involving only basic operations that
can be easily parallelized due to their local nature, such as evaluations of $L$ and $L^{\top}$
and projections on $\mathcal{C}$ and $\mathcal{D}$.
The first two methods are stated here for comparison,
the third one is new. It additionally requires
evaluating $(I+L^{\top} L)^{-1}$, which is potentially non-local,
but in many cases can still be computed fast and explicitely using the Discrete Cosine Transform as shown below.

The following approaches rely on computing projections on the full sets $\mathcal{D}$ resp.~$\mathcal{D}_{\tmop{loc}}$, which requires an iterative procedure for non-trivial
constraint sets such as obtained when using the local envelope method (Sect.~\ref{sec:projdualcons}).
We would like to mention that by introducing further auxiliary variables and a suitable splitting approach,
these inner iterations can also be avoided \cite{Lellmann2010}. However the number of additional variables grows
quadratically in the number of labels, therefore the approach is only feasible for
relatively small numbers of labels and memory requirements are usually one to several orders of magnitude
higher than for the approaches discussed next.


\subsection{Fast Primal-Dual Method}\label{sec:fpd}

One of the most straightforward approaches for optimizing (\ref{eq:discretemodelx}) is to fix
small primal and dual step sizes $\tau_P$ resp. $\tau_D$, and alternatingly apply projected gradient descent resp. ascent on the
primal resp. dual variables. This \emph{Arrow-Hurwicz} approach \cite{Arrow1964} was proposed in a 
PDE framework for solving the two-class labeling problem in \cite{Appleton2006} and recently used in \cite{Chambolle2008}.
An application to denoising problems can be found in \cite{Zhu2008}. However it seems nontrivial to derive sufficient conditions for convergence. Therefore in \cite{Pock2009} the authors propose the \emph{Fast Primal-Dual} (FPD) method, a variant of Popov's saddle point method \cite{Popov1980} with provable convergence. The algorithm is summarized in Alg.~\ref{alg:fpd}.

\begin{algorithm}[tb]
\caption{FPD Multi-Class Labeling} \label{alg:fpd}
\begin{algorithmic}[1]
\STATE Choose $\bar{u}^{(0)} \in \mathbbm{R}^{n \times l}, v^{(0)} \in \mathbbm{R}^{n \times d \times l}$.
\STATE Choose $\tau_P > 0, \tau_D > 0, N \in \mathbbm{N}$.
\FOR{$k = 0, \ldots, N-1$}
    \STATE $v^{(k+1)} \leftarrow \Pi_{\mathcal{D}} \left( v^{(k)} + \tau_D \left( L \bar{u}^{(k)} - b \right) \right)$.
    \STATE $u^{(k+1)} \leftarrow \Pi_{\mathcal{C}} \left( u^{(k)} - \tau_P \left( L^{\top} v^{(k+1)} + s \right) \right)$.
    \STATE $\bar{u}^{(k+1)} \leftarrow 2 u^{(k+1)} - u^{(k)}$.
\ENDFOR
\end{algorithmic}
\end{algorithm}

Due to the explicit steps involved, there is an upper bound condition on the step size to assure convergence, which can be shown to be $\tau_P \tau_D < 1 / \|L\|^2$ \cite{Pock2009}. The operator norm can be bounded according to
\begin{equation}
  \|L\| \leqslant \|\tmop{Grad}\| \leqslant 2 \sqrt{d}
\end{equation}
for the envelope method, and
\begin{equation}
  \|L\| \leqslant \| \tmop{Grad} \|\|A\| \leqslant 2 \sqrt{d} \|A\|
\end{equation}
for the Euclidean metric method. As both the primal and dual iterates are always feasible due to the projections, a stopping criterion based on the primal-dual gap as outlined in Sect.~\ref{sec:optimality} can be employed.


\subsection{Nesterov Method}\label{sec:nesterov}

We will provide a short summary of the application of Nesterov's multi-step method {\cite{Nesterov2004}}
to the saddle point problem (\ref{eq:discretemodelx}) as proposed in \cite{Lellmann2009a}. In contrast to
the FPD method, it treats the nonsmoothness by first applying a smoothing step and
then using a smooth constrained optimization method. The amount of smoothing is balanced in such
a way that the overall number of iterations to produce a solution with a specific accuracy is minimized. 

The algorithm has a theoretical worst-case complexity of $O(1 / \varepsilon)$ for finding
an $\varepsilon$-optimal solution, i.e. $f(u^{(N)})-f(u^{\ast})\leqslant \varepsilon$.
It has been shown to give accurate results for denoising {\cite{Aujol2008}} and general $\ell_1$-norm based problems {\cite{Weiss2007,Becker2009}}.
Besides the desired accuracy, no other parameters have to be provided.
The complete algorithm for our saddle point formulation is shown in Alg.~\ref{alg:nesterovmain}.

\begin{algorithm}[tb]
\caption{Nesterov Multi-Class Labeling} \label{alg:nesterovmain}
\begin{algorithmic}[1]
\STATE Let $c_1 \in \mathcal{C}$, $c_2 \in \mathcal{D}$ and $r_1, r_2 \in
  \mathbbm{R}$ such that $\mathcal{C} \subseteq \mathcal{B}_{r_1} (c_1)$ and
  $\mathcal{D} \subseteq \mathcal{B}_{r_2} (c_2)$; $C \geqslant \|L\|$.
\STATE Choose $x^{(0)} \in \mathcal{C}$ and $N \in \mathbbm{N}$.
\STATE Let $\mu \leftarrow \frac{2\|L\|}{N + 1}  \frac{r_1}{r_2 }$.
\STATE Set $G^{(- 1)} \leftarrow 0, w^{(- 1)} \leftarrow 0$.

\FOR{$k = 0, \ldots, N$}
    \STATE $V \leftarrow \Pi_{\mathcal{D}} \left( c_2 + \frac{1}{\mu}  \left(Lx^{(k)} - b \right) \right)$.  
    \STATE $w^{(k)} \leftarrow w^{(k - 1)} + (k + 1) V$. \label{step:nesaverage}
    \STATE $v^{(k)} \leftarrow \frac{2}{(k+1)(k+2)} w^{(k)}$.
    \STATE $G \leftarrow s + L^{\top} V$.
    \STATE $G^{(k)} \leftarrow G^{(k - 1)} + \frac{k + 1}{2} G$.
    \STATE $u^{(k)} \leftarrow \Pi_{\mathcal{C}} \left( x^{(k)} - \frac{\mu}{\|L\|^2} G \right)$.
    \STATE $z^{(k)} \leftarrow \Pi_{\mathcal{C}} \left( c_1 - \frac{\mu}{\|L\|^2} G^{(k)} \right)$.
    \STATE $x^{(k + 1)} \leftarrow \frac{2}{k + 3} z^{(k)} + \left( 1 - \frac{2}{k + 3} \right) u^{(k)}$.
\ENDFOR
\end{algorithmic}
\end{algorithm} 

The only expensive operations are the projections $\Pi_{\mathcal{C}}$ and
$\Pi_{\mathcal{D}}$, which are efficiently computable as shown above. The
algorithm converges 
in any case and provides explicit suboptimality bounds:

\begin{proposition}\label{prop:nesterovconv}
  In Alg.~\ref{alg:nesterovmain}, the iterates $u^{(k)}$,$v^{(k)}$ are primal resp. dual feasible, i.e. $u^{(k)} \in \mathcal{C}$, $v^{(k)} \in \mathcal{D}$. Moreover, for any solution $u^{\ast}$ of the relaxed problem (\ref{eq:discretemodelx}), the relation
  \begin{equation}
    f (u^{(N)}) - f (u^{\ast}) \leqslant f (u^{(N)}) - f_D (v^{(N)}) \leqslant
    \frac{2 r_1 r_2 C}{(N + 1)} \label{eq:gapnesterov}
  \end{equation}
  holds for the final iterates $u^{(N)}$,$v^{(N)}$.
\end{proposition}

\begin{proof}
  Apply {\cite[Thm.~3]{Nesterov2004}} with the notation $\hat{f} (u) = \langle u, s
  \rangle$, $A = L$, $\hat{\phi} (v) = \langle b, v \rangle$, $d_1 (u) \assign
  \frac{1}{2} \|u - c_1 \|^2$, $d_2 (v) \assign \frac{1}{2} \|v - c_2 \|^2$,
  $D_1 = \frac{1}{2} r_1^2$, $D_2 = \frac{1}{2} r_2^2$, $\sigma_1 = \sigma_2 =
  1$, $M = 0$.
\end{proof}

\begin{corollary}
  For given $\varepsilon > 0$, applying Alg.~\ref{alg:nesterovmain} with
  \begin{eqnarray}
    N & = & \left\lceil 2 r_1 r_2 C \varepsilon^{- 1} - 1 \right\rceil 
  \end{eqnarray}
  yields an $\varepsilon$-optimal solution of (\ref{eq:discretemodelx}), i.e.
  $f (u^{(N)}) - f (u^{\ast}) \leqslant \varepsilon$.
\end{corollary}

For the discretization outlined in Sect.~\ref{sec:discrchap}, we choose $c_1 =
\frac{1}{l} e$ and $r_1 = \sqrt{n (l - 1) / l}$, which leads to the following
complexity bounds to $u^{(N)}$ with respect to the suboptimality~$\varepsilon$.
\begin{itemize}
  \item {\tmem{Envelope method (\ref{eq:defdlocd}).}} From the previous remarks,
  $C \assign 2 \sqrt{d} \geqslant \| L \|$.
  Moreover $c_2 = 0$ and by Prop.~\ref{prop:bounddlocd} (see Appendix), we have $\mathcal{D}_{\tmop{loc}}^{\pot}
  \subseteq \mathcal{B}_{\alpha_{\pot}} (0)$ with
  \begin{eqnarray}
    \alpha_{\pot} & \assign & \min_i \left(\sum_j \pot (i, j)^2\right)^{1/2}\,,
  \end{eqnarray}
  and thus $r_2 = \alpha_{\pot} \sqrt{n}$.
  The total complexity in terms of the number of iterations is
  \begin{eqnarray}
    & & O(\varepsilon^{- 1} n \sqrt{d} \alpha_{\pot})\,.
  \end{eqnarray}
  \item {\tmem{Euclidean metric method (\ref{eq:defdloci}).}} Here we may set $C = 2 \sqrt{d} \|A\|$,
  $c_2 = 0$ and $r_2 = \sqrt{n}$ for a total complexity of
  \begin{eqnarray}
    & & O(\varepsilon^{- 1} n \sqrt{d} \|A\|)\,.
  \end{eqnarray}
\end{itemize}
We arrive at a parameter-free algorithm, with the exception of the desired
suboptimality bound. From the sequence $(u^{(k)},v^{(k)})$ we may again compute the current
primal-dual gap at each iteration.
As a unique feature, the number of required iterations can be determined
a priori and independently of the variables in the data term, which could
prove useful in real-time applications.


\subsection{Douglas-Rachford Method}\label{sec:douglas-rachford}

We demonstrate how to apply the Douglas-Rachford splitting
approach \cite{Douglas1956} to our problem. By introducing auxiliary variables,
we can again reduce the inner steps to projections on the sets $\mathcal{C}$ and $\mathcal{D}$.
This is in contrast to a more straightforward splitting approach such as \cite{Lellmann2009},
where the inner steps require to solve ROF-type problems that include a quadratic data term.

Minimizing a proper, convex, lower-semicontinuous (\tmtextit{lsc})
function $f : X \rightarrow \mathbbm{R} \cup \{\pm\infty\}$
over a finite dimensional vector space $X \assign \mathbbm{R}^n$ can be regarded as finding a zero of its
-- necessarily maximal monotone {\cite{Rockafellar2004}} -- subdifferential
operator $T \assign \partial f : X \rightarrow 2^{X}$. In the operator
splitting framework, $\partial f$ is assumed to be decomposable into the sum
of two ``simple'' operators, $T = A + B$, of which forward and backward steps
can practically be computed. Here, we consider the (tight)
{\tmem{Douglas-Rachford-Splitting}} iteration {\cite{Douglas1956,Lions1979}}
with the fixpoint iteration
\begin{eqnarray}
  \bar{u}^{(k + 1)} & = & (J_{\tau A} (2 J_{\tau B} - I) + (I - J_{\tau B})) (\bar{u}^{(k)})\,, \label{eq:dr-fixpoint}
\end{eqnarray}
where $J_{\tau T} \assign (I + \tau T)^{- 1}$ is the \tmtextit{resolvent} of
$T$. Under the very general constraint that $A$ and $B$ are maximal monotone
and $A + B$ has at least one zero, the sequence $(\bar{u}^{(k)})$ is uniquely defined and will converge to a
point $\bar{u}$, with the additional property that $u \assign J_{\tau B} (\bar{u})$ is a
zero of $T$ {\cite[Thm.~3.15, Prop.~3.20, Prop.~3.19]{Eckstein1989}}.
In particular if $f = f_1 + f_2$ for proper, convex, lsc $f_i$ such that
the relative interiors of their domains have a nonempty intersection,
\begin{equation}
\tmop{ri}(\tmop{dom} f_1) \cap \tmop{ri} (\tmop{dom} f_2) \neq \emptyset\,,
\end{equation}
it follows {\cite[Cor.~10.9]{Rockafellar2004}} that $\partial f = \partial f_1 + \partial f_2$, and
$A \assign \partial f_1$, $B \assign \partial f_2$ are maximal monotone. As
$x \in J_{\tau \partial f_i} (y) \Leftrightarrow x = \arg\min \{ (2 \tau)^{- 1} \|x -
y\|^2_2 + f_i (x) \}$, the computation of the resolvents reduces to proximal
point optimization problems involving only the $f_i$.
However, for straightforward splittings of \eqref{eq:discretemodelx}, such as
\begin{eqnarray}
  & & \min_{u} \underbrace{\max_{v \in \mathcal{D}} g \left( u, v \right)}_{f_1(u)} + \underbrace{\delta_{\mathcal{C}}(u)}_{f_2(u)}\,,
\end{eqnarray}
evaluating the resolvents requires to solve ROF-type problems \cite{Lellmann2009},
which is a strongly non-local operation and requires an iterative procedure,
introducing additional parameters and convergence issues.
Instead, we follow the procedure in \cite{Eckstein1992,Setzer2009a} of adding
auxiliary variables before splitting the objective in order to simplify the individual steps of the algorithm.
We introduce $w = L u$ and split according to
\begin{eqnarray}
  & & \min_{u \in \mathcal{C}} \max_{v \in \mathcal{D}} \langle u, s \rangle + \langle Lu, v \rangle - \langle b, v \rangle\\
  & = & \min_u \langle u, s \rangle + \sigma_{\mathcal{D}} (L u - b) + \delta_{\mathcal{C}}(u)\\
  & = & \min_{u, w} h(u,w) \assign \underbrace{\delta_{L u = w} (u, w)}_{h_1(u,w)} +
        \underbrace{\langle u, s \rangle + \delta_{\mathcal{C}} (u) +
        \sigma_{\mathcal{D}} (w - b)}_{h_2(u,w)}.\label{eq:dr-howtosplit}
\end{eqnarray}
We apply the tight Douglas-Rachford iteration to this formulation: Denote
\begin{eqnarray}
  (u^{(k)},w^{(k)}) & \assign & J_{\tau \partial B} (\bar{u}^{(k)},\bar{w}^{(k)})\,,\\
  (u'^{(k)},w'^{(k)}) & \assign & J_{\tau A} (2 J_{\tau B} - I) (\bar{u}^{(k)},\bar{w}^{(k)})\nonumber\\
  & = & J_{\tau A} (2 u^{(k)} - \bar{u}^{(k)}, 2 w^{(k)}-\bar{w}^{(k)})\,.
\end{eqnarray}
Then $(\bar{u}^{(k+1)},\bar{w}^{(k+1)}) = (\bar{u}^{(k)}+u'^{(k)}-u^{(k)},\bar{w}^{(k)}+w'^{(k)}-w^{(k)})$,
according to \eqref{eq:dr-fixpoint}.
Evaluating the resolvent $J_{\tau B}$ is equivalent to a proximal step on $h_2$; moreover 
due to the introduction of the auxiliary variables it decouples,
  \begin{eqnarray}
    u'^{(k)} & = & \arg\min_{u'} \{ \frac{1}{2} \|u' - (2 u^{(k)} - \bar{u}^{(k)}) +
    \tau s\|_2^2 + \delta_{\mathcal{C}} (u')\} \nonumber\\    
    & = & \Pi_{\mathcal{C}} \left( (2 u^{(k)} - \bar{u}^{(k)}) - \tau s \right), \\
    w'^{(k)} & = & \arg\min_{w'} \{ \frac{1}{2 \tau} \|w' - (2 w^{(k)} -
    \bar{w}^{(k)})\|_2^2 + \sigma_{\mathcal{D}} (w' - b)\} \nonumber\\
    & = & (2 w^{(k)} - \bar{w}^{(k)}) - \tau \Pi_{\mathcal{D}} \left(
    \frac{1}{\tau} (2 w^{(k)} - \bar{w}^{(k)} - b) \right) .  \label{eq:wkprimegeneral}
  \end{eqnarray} 
In a similar manner, $J_{\tau A}$ resp.~the proximal step on $h_1$ amounts to the least-squares problem
  \begin{equation}
    (u^{(k)}, w^{(k)}) = \arg \min_{u, w} \{ \frac{1}{2 \tau}  \left( \|u
    - \bar{u}^{(k)} \|^2_2 +\|w - \bar{w}^{(k)} \|^2 \right) + \delta_{L u = w} (u, w)\}.
  \end{equation}
Substituting the constraint $w^{(k)} = L u^{(k)}$ yields
  \begin{eqnarray}
    u^{(k)} & = & \arg \min_u \{\|u - \bar{u}^{(k)} \|^2_2 +\|L u - \bar{w}^{(k)} \|^2 \}\nonumber\\
    & = & (I + L^{\top} L)^{- 1}  \left( \bar{u}^{(k)} + L^{\top}  \bar{w}^{(k)} \right) .  \label{eq:ukiata}
  \end{eqnarray}
Finally, Alg.~\ref{alg:dr-primal} is obtained by substituting
$w''^{(k)} \assign\Pi_{\mathcal{D}} \left( \frac{1}{\tau} (2 w^{(k)} - \bar{w}^{(k)} - b) \right)$.
\begin{algorithm}[tb]
\caption{Douglas-Rachford Multi-Class Labeling} \label{alg:dr-primal}
\begin{algorithmic}[1]
\STATE Choose $\bar{u}^{(0)} \in \mathbbm{R}^{n \times l}, \bar{w}^{(0)} \in \mathbbm{R}^{n\times d \times l}$ (or set $\bar{w}^{(0)} = L \bar{u}^{(0)}$).
\STATE Choose $\tau > 0$.
\STATE $k \leftarrow 0$.
\WHILE{(not converged)}
  \STATE $u^{(k)} \leftarrow \Pi_{\mathcal{C}} \left( \bar{u}^{(k)} - \tau s \right)$.
  \STATE $w''^{(k)} \leftarrow \Pi_{\mathcal{D}} \left(\frac{1}{\tau} (\bar{w}^{(k)} - b) \right)$.

  \STATE $u'^{(k)} \leftarrow (I + L^{\top} L)^{- 1}  \left( (2 u^{(k)} - \bar{u}^{(k)}) + L^{\top} (\bar{w}^{(k)} - 2 \tau w''^{(k)}) \right)$.\label{step:ukiltl}
  \STATE $w'^{(k)} \leftarrow L u'^{(k)}$. 
  
  \STATE $\bar{u}^{(k + 1)} \leftarrow \bar{u}^{(k)} + u'^{(k)} - u^{(k)}$.
  \STATE $\bar{w}^{(k + 1)} \leftarrow w'^{(k)} + \tau w''^{(k)}$.
	\STATE $k \leftarrow k + 1$.
\ENDWHILE
\end{algorithmic}
\end{algorithm} 
  
Solving the linear equation system \eqref{eq:ukiata} can often be greatly accelerated by exploiting the fact that under the forward difference discretization with Neumann boundary conditions, $\tmop{grad}^{\top} \tmop{grad}$ diagonalizes
under the discrete cosine transform (DCT-2) \cite{Strang1999,Lellmann2009}:
\begin{equation}
  \tmop{grad}^{\top} \tmop{grad} = B^{-1} \tmop{diag} (c) B
\end{equation}
where $B$ is the orthogonal transformation matrix of the DCT and $c$ is the
vector of eigenvalues of the discrete Laplacian. In both approaches presented above,
$L$ is of the form $L = A \otimes \tmop{grad}$ for some (possibly identity) matrix
$A \in \mathbbm{R}^{k \times l}, k \leqslant l$. First, compute the decomposition
$A^{\top} A = V^{- 1} \tmop{diag} (a) V$ with
$a \in \mathbbm{R}^l$ and an orthogonal matrix $V \in \mathbbm{R}^{l \times
l}$, $V^{- 1} = V^{\top}$. Then
\begin{eqnarray}
  (I + L^{\top} L)^{- 1} & = & \left( V^{\top} \otimes I_n \right) \left( I_l \otimes B^{- 1} \right)\cdot\nonumber\\
  & & \left( I + \tmop{diag} (a) \otimes \tmop{diag} (c) \right)^{- 1}
  \left( I_l \otimes B \right)  \left( V \otimes I_n \right)\label{eq:dctainversion}
\end{eqnarray}
(see Appendix for the proof).
Thus step~\ref{eq:ukiata} can be achieved fast and accurately through matrix multiplications with $V$, discrete cosine transforms, and one $O (n l)$ product for inverting the inner diagonal matrix.

In addition to convergence of the primal iterates $(u^{(k)})$, it can be shown that the sequence $(w''^{(k)})$ from Alg.~\ref{alg:dr-primal} actually converges to a solution of the dual problem:
\begin{proposition}\label{prop:dr-convergence}
  Let $\mathcal{C}$, $\mathcal{D}$ be bounded, closed and convex sets with $\tmop{ri}
  (\mathcal{C}) \neq \emptyset$ and $\tmop{ri} (\mathcal{D}) \neq \emptyset$.
  Then Alg.~\ref{alg:dr-primal} generates a
  sequence of primal/dual feasible pairs $(u^{(k)}, w''^{(k)}) \in \mathcal{C} \times
  \mathcal{D}$ such that, for any saddle point $(u^{\ast},v^{\ast})$ of the relaxed problem (\ref{eq:discretemodelx}), 
  \begin{eqnarray}
    f (u^{(k)}) & \overset{k \rightarrow \infty}{\rightarrow} & f (u^{\ast}) = f_D (v^{\ast})\,,\\
    f_D (w''^{(k)}) & \overset{k \rightarrow \infty}{\rightarrow} & f_D (v^{\ast}) = f (u^{\ast})\,.
  \end{eqnarray}
  Moreover,
  \begin{eqnarray}
    f (u^{(k)}) - f (u^{\ast}) & \leqslant & f (u^{(k)}) - f_D (w''^{(k)})\label{eq:suboptboundinprop}
  \end{eqnarray}
  provides an upper bound for the suboptimality of the current iterate.
\end{proposition}
\begin{proof}
See Appendix.
\end{proof}

Thus the Douglas-Rachford approach also allows to use the primal-dual gap
\begin{equation}
f(u^{(k)})-f_D(w''^{(k)})
\end{equation}
as a stopping criterion.

Very recently, a generalization of the FPD method \cite{Pock2009} has been proposed \cite{Chambolle2010}. The
authors show that under certain circumstances, their method is equivalent to Douglas-Rachford splitting.
As a result, it is possible to show that Alg.~\ref{alg:dr-primal}
can alternatively be interpreted as an application of the primal-dual method from \cite{Chambolle2010}
to the problem formulation \eqref{eq:dr-howtosplit}.


\subsection{Projection on the Dual Constraint Set}\label{sec:projdualcons}

For the Euclidean metric approach, projection on the unit ball $\mathcal{D}_{\tmop{loc}}^{I}$ and
thus on $\mathcal{D}$ is trivial:
\begin{eqnarray}
  \Pi_{\mathcal{D}_{\tmop{loc}}^I} (v) & = & \left\{
    \begin{array}{ll}
      v, & \|v\| \leqslant 1\,,\\
      \frac{v}{\|v\|}\,, & \text{otherwise}\,.
  \end{array} \right.
\end{eqnarray}
Projection on $\mathcal{D}_{\tmop{loc}}^{\pot}$ for the general metric case is more
involved. We represent $\mathcal{D}_{\tmop{loc}}^{\pot}$ as the intersection of convex sets,
\begin{eqnarray}
  & & \mathcal{D}_{\tmop{loc}}^{\pot} = \mathcal{R} \cap \mathcal{S}\,,\;\mathcal{R} \assign \{ v \in \mathbbm{R}^{d \times l} | \sum_i v^i = 0 \}\,, \\
  & & \mathcal{S} \assign \bigcap_{i < j} \mathcal{S}^{i, j}\,, \; \mathcal{S}^{i, j} \assign \{v \in \mathbbm{R}^{d \times l} |\|v^i - v^j \| \leqslant \pot (i, j)\}\,.\nonumber
\end{eqnarray}
Since $\Pi_{\mathcal{R}}$ is amounts to a translation along $e=(1,\ldots,1)$, and $\mathcal{S}$ is
translation invariant in the direction of $e$, the projection can be decomposed:
\begin{eqnarray}
  \Pi_{\mathcal{D}_{\tmop{loc}}^{\pot}}(v) = \Pi_{\mathcal{R}}( \Pi_{\mathcal{S}} (v) )\,.
\end{eqnarray}
We then follow the idea of {\cite{Chambolle2008}} to use Dykstra's method {\cite{Boyle1986}} for iteratively computing an approximation of $\Pi_{\mathcal{S}}$ using only projections on the individual sets $\mathcal{S}^{i,j}$.
However, any recent multiple-splitting method could be used \cite{Combettes2008,Goldfarb2009}.
In our case,
$\Pi_{\mathcal{S}^{i, j}}$ can be stated in closed form:
\begin{eqnarray}
  \Pi_{\mathcal{S}^{i, j}} (v) & = & \left\{ \begin{array}{ll}
    v, & \|v^i - v^j \| \leqslant \pot (i, j)\,,\\
    (w^1, \ldots, w^l), & \text{otherwise}\,,
  \end{array} \right.
\end{eqnarray}
where
\begin{eqnarray}
  w^k & = & \left\{ \begin{array}{ll}
    v^k, & k \nin \{i, j\}\,,\\
    v^i - \frac{\|v^i - v^j \|- \pot (i, j)}{2}  \frac{v^i - v^j}{\|v^i - v^j\|}, & k = i\,,\\
    v^j + \frac{\|v^i - v^j \|- \pot (i, j)}{2}  \frac{v^i - v^j}{\|v^i - v^j\|}, & k = j\,.
  \end{array} \right.
\end{eqnarray}
The complete Dykstra method for projecting a vector $v$ onto $\mathcal{S}$ is outlined in Alg.~\ref{alg:dykstra}.
\begin{algorithm}[tb]
\caption{Dykstra's Method for Projecting onto the Intersection of Convex Sets} \label{alg:dykstra}
\begin{algorithmic}[1]
\STATE Associate with each $(i, j), i < j$ a unique index $t \in \{1, \ldots, k\}$, $k \assign l (l - 1)$.
\STATE $x \leftarrow v \in \mathbbm{R}^{d \times l}$.
\STATE $y^1, \ldots, y^k \leftarrow 0 \in \mathbbm{R}^{d \times l}$.
\WHILE{($x$ not converged)}
  \FOR{$t = 1, \ldots, k$}
    \STATE $x' \leftarrow \Pi_{S^{(i_t, j_t)}} (x + y^t)$.
    \STATE $y^t \leftarrow x + y^t - x'$.
    \STATE $x \leftarrow x'$.
  \ENDFOR
\ENDWHILE
\end{algorithmic}
\end{algorithm} 
While the sequence of $y$ may be unbounded, $x$ converges to
$\Pi_{\mathcal{S}} (v)$ (cf.~{\cite{Gaffke1989,Xu2000}}).


\subsection{Computing the Objective}\label{sec:objeval}

For computing the primal-dual gap and the binarization step, it is necessary to
evaluate the objective for a given $u$. Unfortunately, in the general case
this is nontrivial as the objective's integrand $\Psi$ is defined implicitly as
\begin{eqnarray}
  \Psi (z) & = & \max_{v \in \mathcal{D}_{\tmop{loc}}^{\pot}} \langle z, v \rangle\,. \label{eq:psizobj}
\end{eqnarray}
We exploit that $\Pi_{\mathcal{D}}$ can be computed and use an iterative gradient-projection method,
\begin{eqnarray}
  v^{(k + 1)} & \leftarrow & \Pi_{\mathcal{D}} (v^{(k)} + \tau z)
\end{eqnarray}
for some $\tau > 0$, starting with $v^{(0)} = z$. As the objective in (\ref{eq:psizobj}) is linear,
$\mathcal{D}_{\tmop{loc}}^{\pot}$ is bounded and thus $\Psi$ is bounded from above, convergence
follows {\cite{Goldstein1964,Levitin1966}} for any step size $\tau$ {\cite{Wang2004}}:
\begin{equation}
\lim_{k \rightarrow \infty} \langle z, v^{(k)} \rangle = \Psi (z)\,.
\end{equation}
There is a trade-off in choosing the step size, as large $\tau$ lead to a
smaller number of outer iterations, but an increased number of nontrivial operations
in the projection. We chose $\tau = 2$, which worked well for all examples.
It is also possible to use any of the other nonsmooth optimization methods presented above.


\section{Experiments}\label{sec:experiments}

Regarding the practical performance of the presented approaches, we focused on
two main issues: convergence speed and tightness of the relaxation. We will first
quantitatively compare the presented algorithms in terms of runtime and the number of inner iterations,
and then provide some results on the effect of the Euclidean metric vs. the envelope regularization.

The algorithms
were implemented in Matlab with some core functions, such as the computation of the
gradient and the projections on the dual sets, implemented in C++.
We used Matlab's built-in FFTW interface for computing the DCT for the Douglas-Rachford
approach. All experiments were conducted on an Intel Core2 Duo $2.66$ GHz with $4$ GB of RAM
and 64-bit Matlab 2009a.


\subsection{Relative Performance}\label{sec:relperformance}

To compare the convergence speed of the three different approaches, we computed the relative primal-dual gap at
each iteration as outlined in Sect.~\ref{sec:optimization}. As it bounds the suboptimality
of the current iterate (see Sect.~\ref{sec:optimality}), it constitutes a reliable and convenient criterion for
performance comparison.

Unfortunately the gap is not available for the envelope method, as it requires the primal objective
to be evaluatable. Using a numerical approximation such as the one in Sect.~\ref{sec:objeval} is not
an option, as these methods can only provide a \emph{lower} bound for the objective. This would lead
to an underestimation of the gap, which is critical as one is interested in
the behavior when the gap is very close to zero. Therefore we restricted the
gap-based performance tests to the Euclidean metric regularizer.

In order to make a fair comparison we generally analyzed the progression of the gap with respect to computation time, rather than the number of iterations.

For the first tests we used the synthetical $256 \times 256$ ``four colors'' input image (Fig.~\ref{fig:fourcinput}).
It represents a typical segmentation problem with several objects featuring sharp corners and round structures above a uniform background. The label set consists of three classes for the foreground and one class for the background.
The image was overlaid with iid Gaussian noise with $\sigma = 1$ and truncated to
$[0,1]$ on all RGB channels. We used a simple $\ell^1$ data term, $s_i(x) = \| g(x) - c^i \|_1$,
where $g(x) \in [0,1]^3$ are the RGB color values of the input image in $x$, and $c^i$ is a
prototypical color vector for the $i$-th class.
\begin{figure}
\centering
\includegraphics[width=.24\columnwidth]{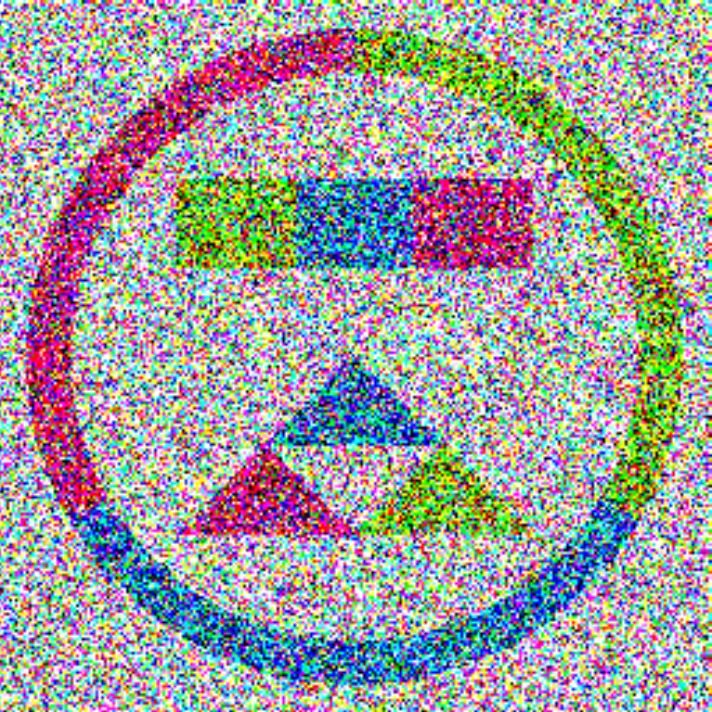}
\includegraphics[width=.24\columnwidth]{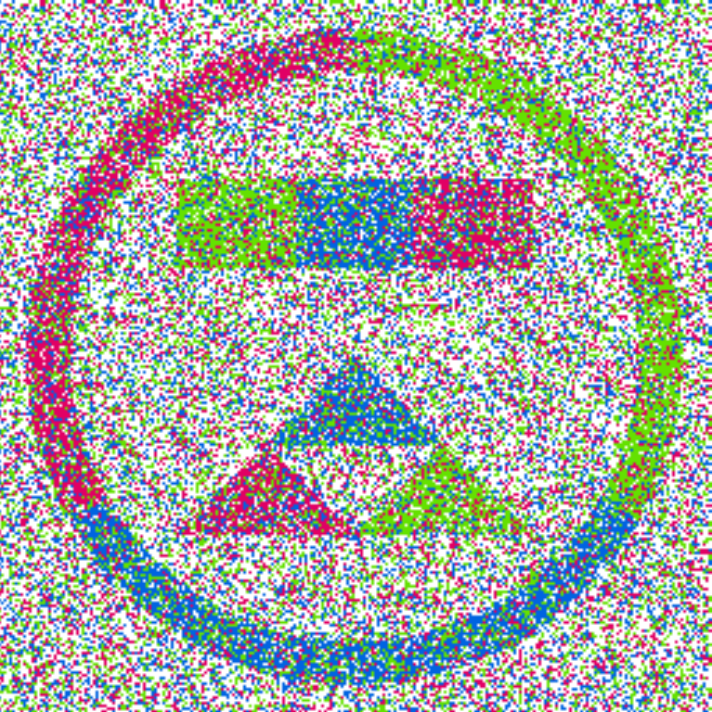} 
\includegraphics[width=.24\columnwidth]{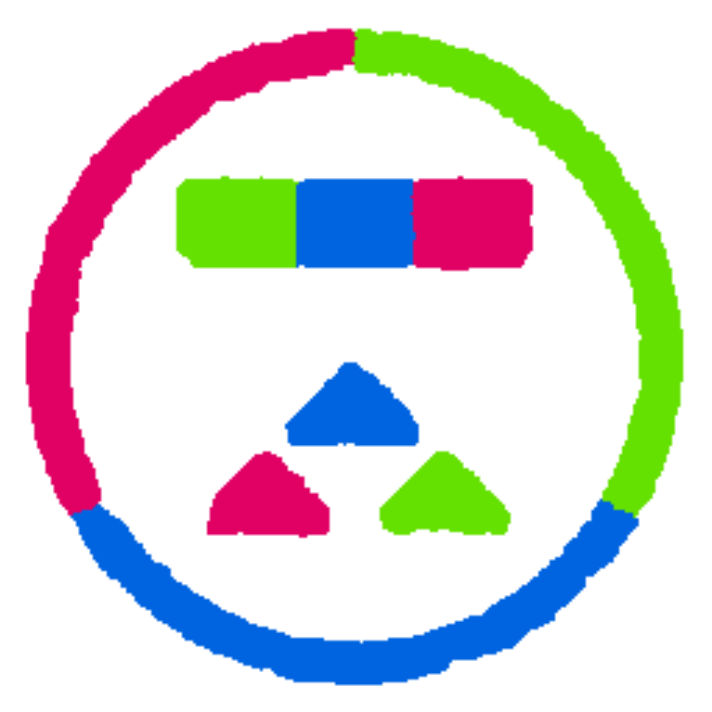}
\includegraphics[width=.24\columnwidth]{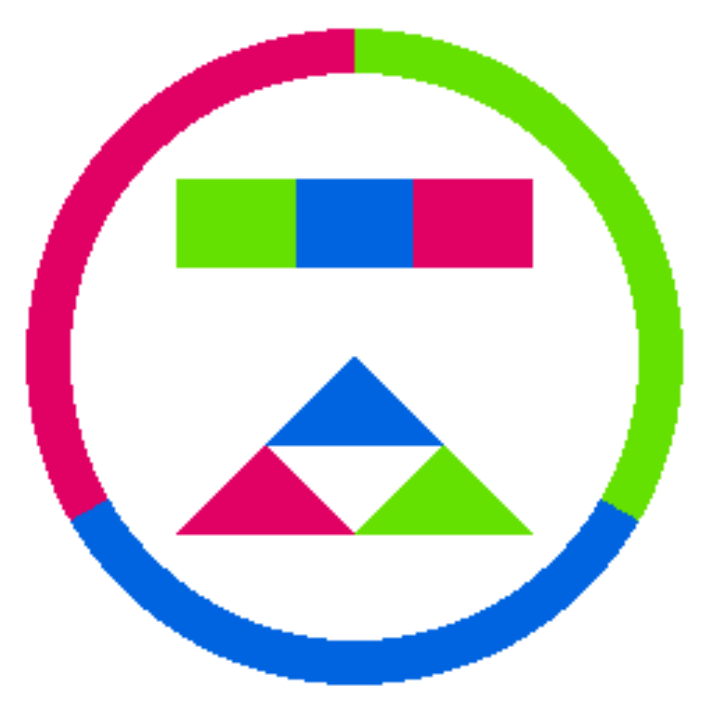}
\caption{Synthetical ``four colors'' input image for the performance tests. \textbf{Top row:}
Input image with Gaussian noise, $\sigma = 1$; local labeling without regularizer.
\textbf{Bottom row:} Result with the Potts regularizer and Douglas-Rachford optimization; ground truth.
}%
\label{fig:fourcinput}
\end{figure}

The runtime analysis shows that FPD and Douglas-Rachford perform similar,
while the Nesterov method falls behind considerably in both the primal and the dual objective (Fig.~\ref{fig:fourcspeedpottsobjective}).

\begin{figure*}
\centering
\setlength{\unitlength}{.48\columnwidth}
\begin{picture}(1,.63571)
\put(0,0){\includegraphics[width=\unitlength]{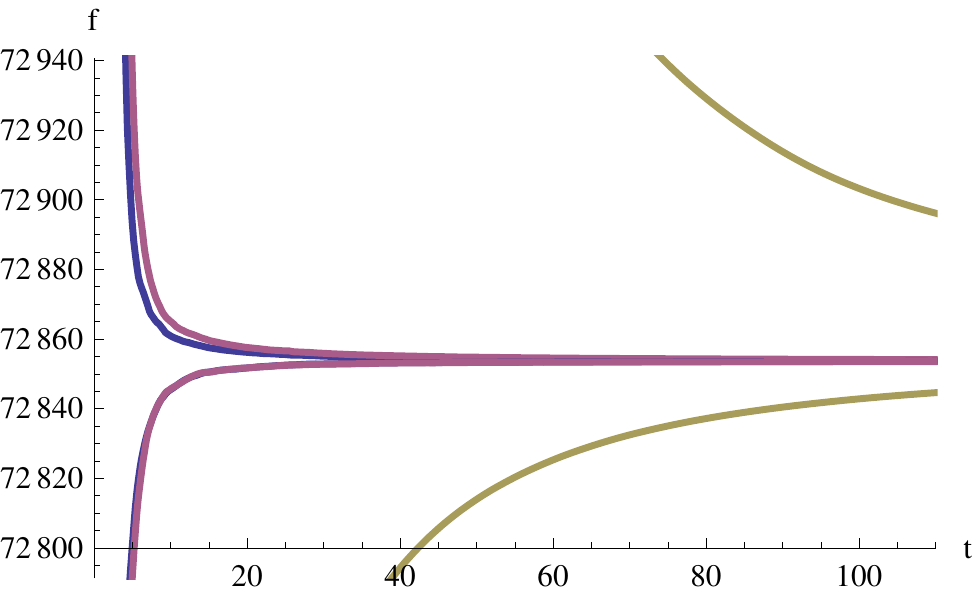}}
\put(.725,.35){\includegraphics[width=0.2\unitlength]{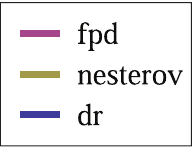}}
\end{picture}
\setlength{\unitlength}{.48\columnwidth}
\begin{picture}(1,.63571)
\put(0,0){\includegraphics[width=\unitlength]{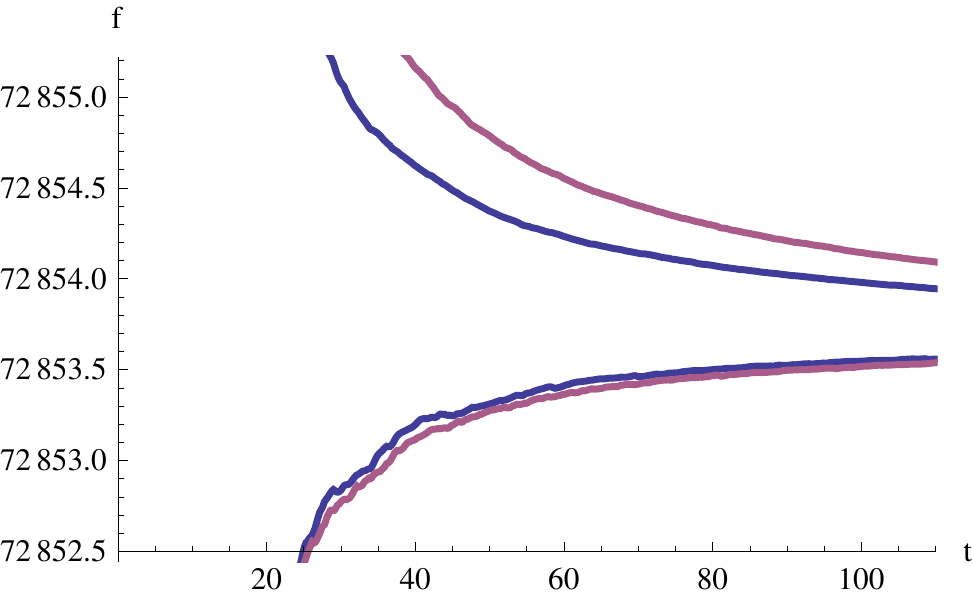}}
\put(.725,.41){\includegraphics[width=0.2\unitlength]{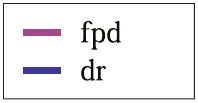}}
\end{picture}
\caption{Convergence speed for the ``four colors'' image in Fig.~\ref{fig:fourcinput}. \textbf{Left:} Primal (upper) and dual (lower) objectives vs. computation time for the (from top to bottom) Nesterov, Fast Primal-Dual (FPD) and Douglas-Rachford methods. \textbf{Right:} Detailed view of the FPD and DR methods. The primal and dual objectives provide
upper and lower bounds for the objective of the true optimum. Douglas-Rachford and FPD perform equally,
while the Nesterov method falls behind.
}%
\label{fig:fourcspeedpottsobjective}
\end{figure*}

The picture changes when considering the gap with respect to the number of iterations rather than time, eliminating
influences of the implementation and system architecture. To achieve the same accuracy, Douglas-Rachford requires only one third of the iterations compared to FPD (Fig.~\ref{fig:fourcspeedpottsgap}). This advantage does not fully translate
to the time-based analysis as the DCT steps increase the per-step computational cost significantly. However in this
example the projections on the sets $\mathcal{C}$ and $\mathcal{D}$ were relatively cheap compared to the DCT. In situations
where the projections dominate the time per step, the reduced iteration count can be expected to lead to an almost equal reduction in computation time.

\begin{figure*}
\centering
\setlength{\unitlength}{.48\columnwidth}
\begin{picture}(1,.63571)
\put(0,0){\includegraphics[width=\unitlength]{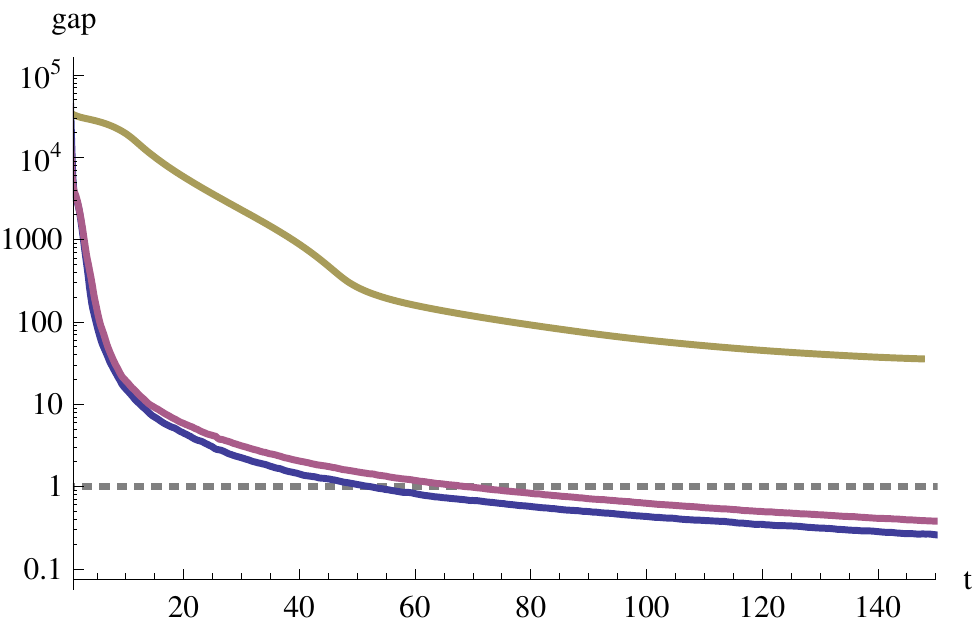}}
\put(.725,.41){\includegraphics[width=0.2\unitlength]{manual_legend_three.pdf}}
\end{picture}
\setlength{\unitlength}{.48\columnwidth}
\begin{picture}(1,.63571)
\put(0,0){\includegraphics[width=\unitlength]{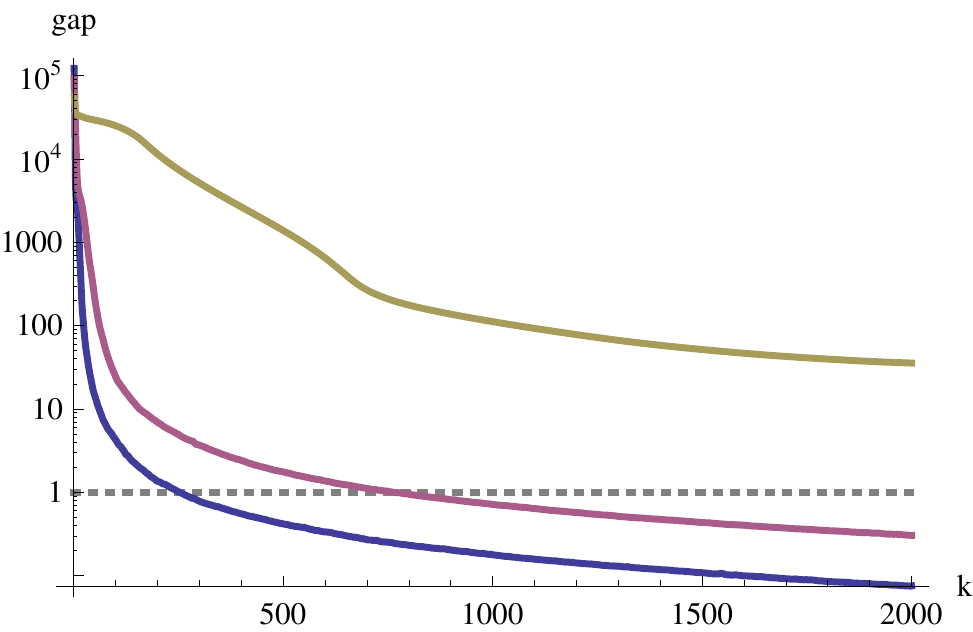}}
\put(.725,.41){\includegraphics[width=0.2\unitlength]{manual_legend_three.pdf}}
\end{picture}
\caption{Primal-Dual gap for Fig.~\ref{fig:fourcspeedpottsobjective} with respect to time and number of iterations. \textbf{Top:} Primal-Dual gap vs. time and number of iterations. The Nesterov method (top) again falls behind, while FPD (center) and Douglas-Rachford (bottom) are equally fast. \textbf{Bottom:} Primal-Dual gap vs. number of iterations. The Douglas-Rachford method requires only one third of the FPD iterations, which makes it suitable for problems with expensive inner steps.
}%
\label{fig:fourcspeedpottsgap}
\end{figure*}

One could ask if these relations are typical to the synthetical data used. However we found them confirmed on
a large majority of the problems tested. As one example of a real-world example, consider the ``leaf'' image (Fig.~\ref{fig:leafdifferentreg}). We computed a segmentation into $12$ classes with Potts regularizer, again based on the $\ell^1$ distances for the data term, with very similar relative performance as for the ``four colors'' image (Fig.~\ref{fig:leafobjtime}).

\begin{figure}
\centering
\setlength{\unitlength}{.48\columnwidth}
\begin{picture}(1,.63571)
\put(0,0){\includegraphics[width=\unitlength]{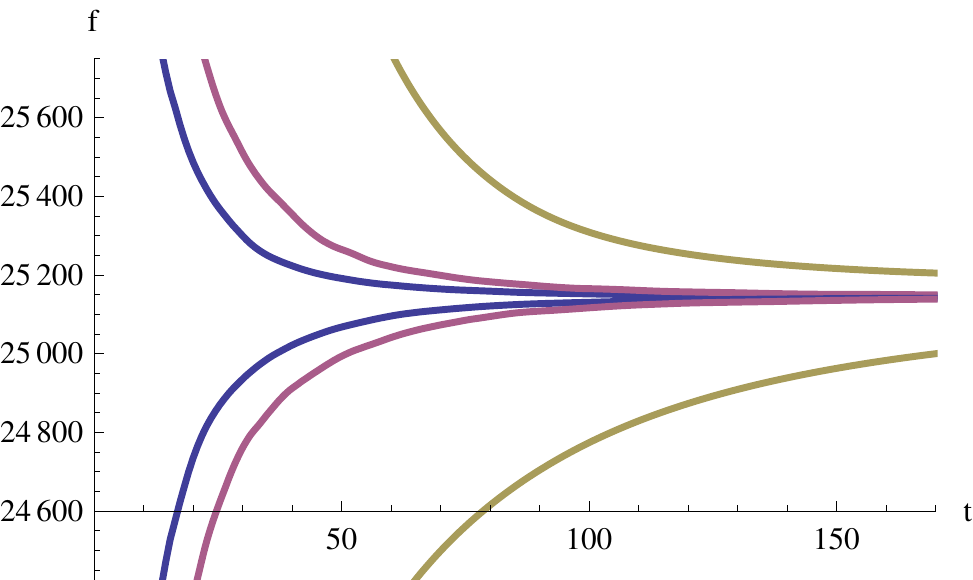}}
\put(.725,.41){\includegraphics[width=0.2\unitlength]{manual_legend_three.pdf}}
\end{picture}
\setlength{\unitlength}{.48\columnwidth}
\begin{picture}(1,.63571)
\put(0,0){\includegraphics[width=\unitlength]{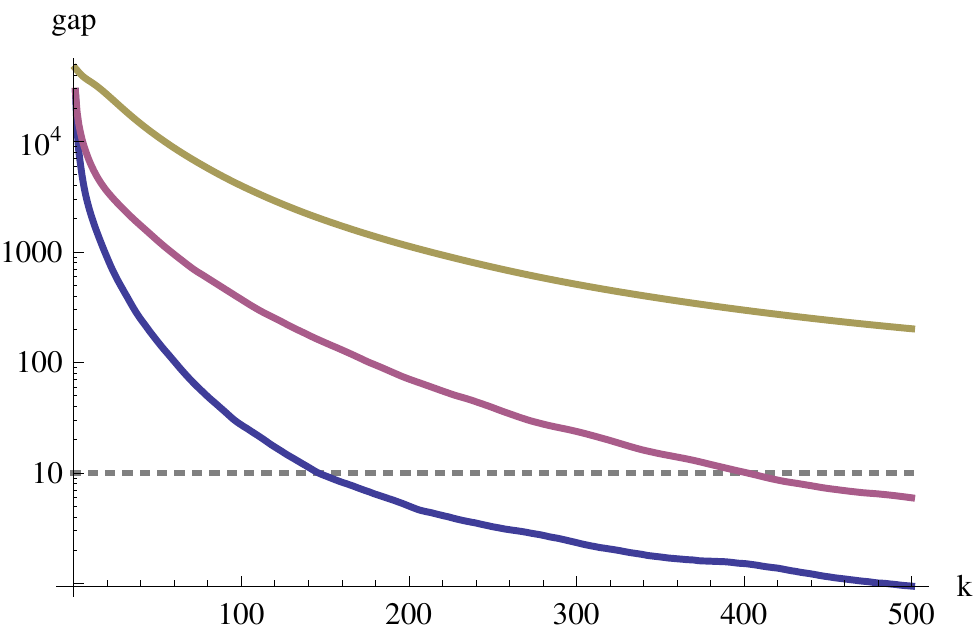}}
\put(.725,.41){\includegraphics[width=0.2\unitlength]{manual_legend_three.pdf}}
\end{picture}
\caption{Objectives and primal-dual gap for the real-world leaf image in Fig.~\ref{fig:leafdifferentreg} with
$12$ classes and Potts potential. \textbf{Left:} Primal (upper) and dual (lower) objectives vs. time.
The Nesterov method (top) falls behind the FPD (center)
and Douglas-Rachford (bottom) methods. \textbf{Right:} Gap vs. number of iterations. As with the
synthetical four-colors image (Fig.~\ref{fig:fourcspeedpottsgap}), the Douglas-Rachford approach reduces the number of
required iterations by approximately a factor of $3$.
}%
\label{fig:leafobjtime}
\end{figure}


\subsection{Number of Variables and Regularization Strength}

To examine how the presented methods scale with increasing image size, we evaluated the ``four colors'' image at $20$ different scales ranging from $16\times16$ to $512\times512$. Note that if the grid spacing is held constant, the
regularizer weights must be scaled proportionally to the image width resp. height in order to
obtain structurally comparable results, and not mix up effects of the problem size and of the regularization strength.

In order to compensate for the increasing number of variables, the stopping criterion was based on the
relative gap \eqref{eq:relativegap}. The algorithms terminated as soon as the relative gap fell below $10^{-4}$.
The Nesterov method consistently produced gaps in the $10^{-3}$ range and never achieved the threshold within the limit of $2000$ iterations. Douglas-Rachford and FPD scale only slightly superlinearly with the problem size, which is quite a good result for such comparatively simple first-order methods (Fig.~\ref{fig:fourcspeedpottstimevssize}).
\begin{figure}
\centering
\setlength{\unitlength}{.48\columnwidth}
\begin{picture}(1,.63571)
\put(0,0){\includegraphics[width=\unitlength]{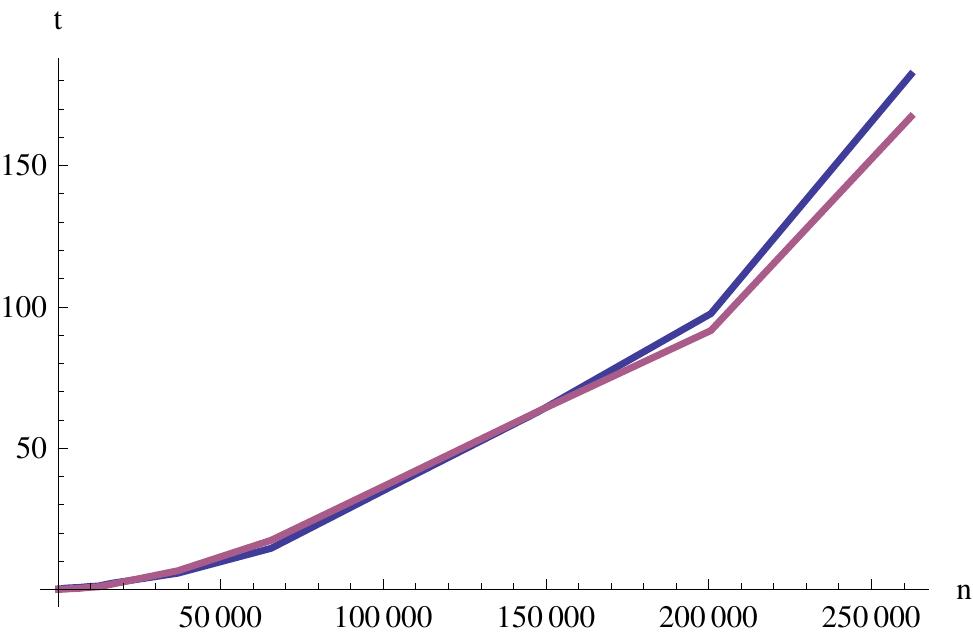}}
\put(.725,.11){\includegraphics[width=0.2\unitlength]{manual_legend_two.pdf}}
\end{picture}
\caption{Computation time for increasing problem size for the Douglas-Rachford (top, dark) and FPD (bottom, light) methods. Shown is the time in seconds required to achieve a relative gap of $10^{-4}$. The computational
effort scales slightly superlinearly with the number of pixels. The Nesterov method never converged to the
required accuracy within the limit of $2000$ iterations.
}%
\label{fig:fourcspeedpottstimevssize}
\end{figure}

While we deliberately excluded influences of the regularizer in the previous experiment, it is also interesting to examine
how algorithms cope with varying regularization strength. We fixed a resolution of $256 \times 256$ and performed
the same analysis as above, scaling the regularization term by an increasing sequence of $\lambda$ in the $[0.1,5]$ range (Fig.~\ref{fig:fourcvaryinglambda}).

Interestingly, for low regularization, where much of the noise remains in the solution, FPD clearly takes the lead. For scenarios with large structural changes, Douglas-Rachford performs better. We observed two peaks in the runtime plot
which we cannot completely explain. However we found that at the first peak, structures in the image did not disappear at in parallel but rather one after the other, which might contribute to the slower convergence. Again, the Nesterov method never achieved the required accuracy.
\begin{figure*}
\centering
\setlength{\unitlength}{.98\columnwidth}
\begin{picture}(1,.27182)
\put(0,0){\includegraphics[width=\unitlength]{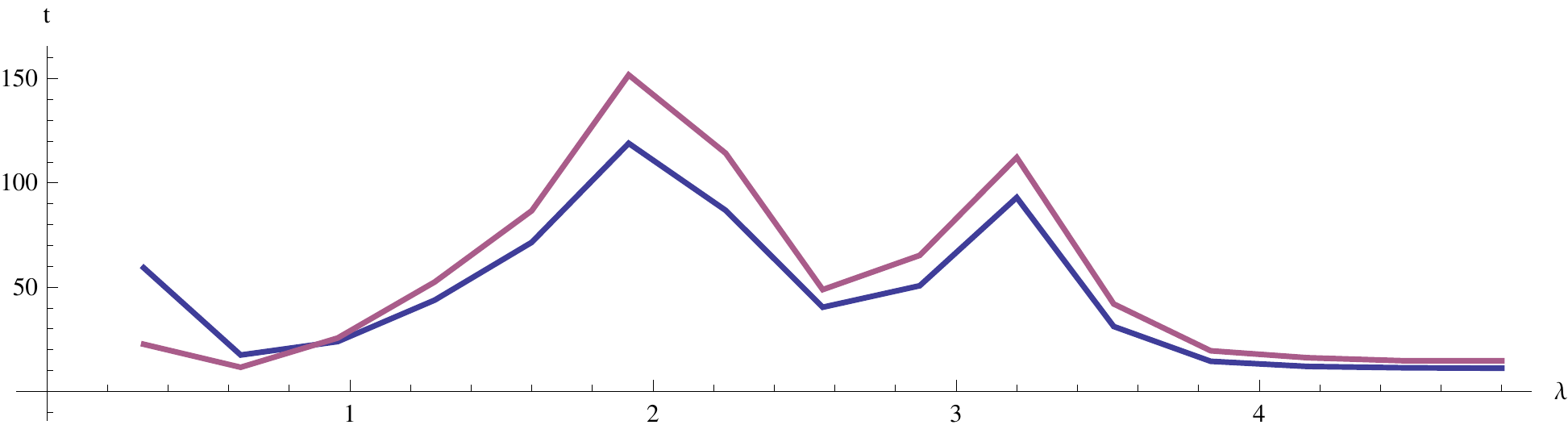}}
\put(.85,.17){\includegraphics[width=0.0855\unitlength]{manual_legend_two.pdf}}
\end{picture}\\
\includegraphics[width=.9404\columnwidth]{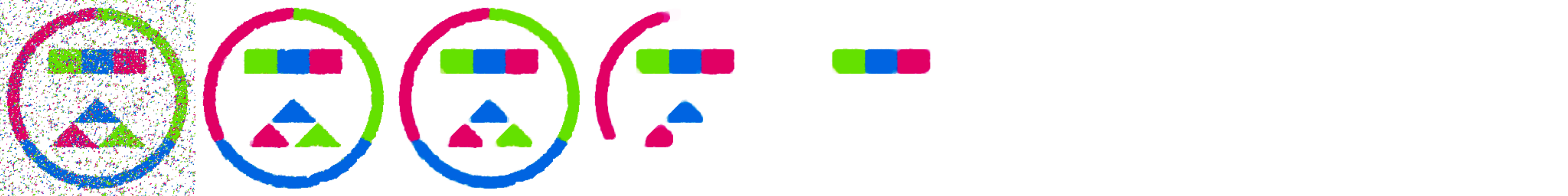}
\caption{Computation time for varying regularization strength $\lambda$ for the Douglas-Rachford (top, dark) and FPG (bottom, light) methods. The images at the bottom show the final result for the $\lambda$ above. FPD is strong on
low regularization problems, while Douglas-Rachford is better suited for problems with large structural changes. The
Nesterov method never achieved the relative gap of $10^{-5}$ within $2000$ iterations.
}%
\label{fig:fourcvaryinglambda}
\end{figure*}


\subsection{Breaking Points}

We have no clear explanation why the Nesterov method appears to almost always fall behind. However it is possible
to compare its behavior with the a priori bound from Prop.~\ref{prop:nesterovconv}. By inspecting the derivations in the original work \cite{Nesterov2004}, it can be seen that exactly one half of the final
bound comes from the smoothing step, while the other half is caused by the finite number of iterations:
\begin{equation}
\delta_{\tmop{total}} = \delta_{\tmop{smooth}} + \delta_{\tmop{iter}},\quad \textnormal{where} \quad \delta_{\tmop{smooth}} = \delta_{\tmop{iter}}\,.
\end{equation}
Moreover, $\delta_{\tmop{iter}}$ decreases with $1/(k+1)^2$, which gives an a priori per-iteration suboptimality bound of
\begin{equation}
\delta_{\tmop{total}}^{(k)} = \delta_{\tmop{smooth}} + \left(\frac{N+1}{k+1}\right)^2 \delta_{\tmop{iter}}\,.
\end{equation}
On the ``four colors'' image, the actual gap stays just below $\delta_{\tmop{total}}^{(k)}$ in the beginning (Fig.~\ref{fig:fourcspeedpottsnesterov}). This implies that the theoretical suboptimality bound can hardly be improved, e.g. by choosing constants more precisely. Unfortunately, the bound is generally rather large, in this case at $\delta_{\tmop{total}} = 256.8476$ for $2000$ iterations. While the Nesterov method outperforms the theoretical bound $\delta_{\tmop{total}}$ by a factor of $2$ to $10$ and even drops well below the worst-case smoothing error $\delta_{\tmop{smooth}}$, it still cannot compete with the other methods, which achieve a gap of $0.3052$ (FPD) resp. $0.0754$ (Douglas-Rachford).

\begin{figure}
\centering
\includegraphics[width=.6\columnwidth]{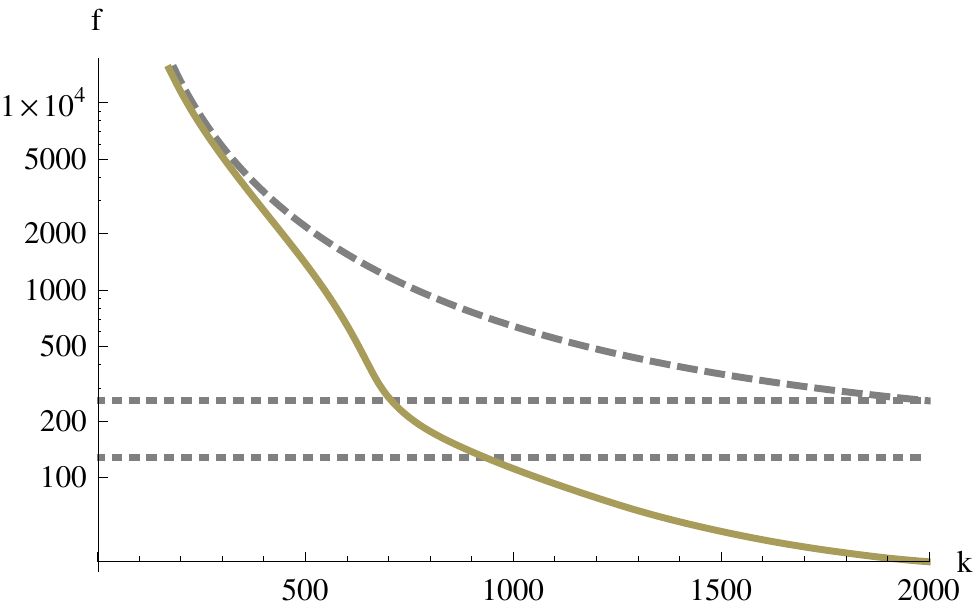}
\caption{Theoretical vs. practical performance of the Nesterov method for Fig.~\ref{fig:fourcspeedpottsobjective}. As expected, the method stays below the theoretical per-iteration bound $\delta_{\tmop{total}}^{(k)}$ (dashed). At the final iteration, the worst-case total bound $\delta_{\tmop{total}}$ (dotted, top) is outperformed by a factor of $7$, which implicates that the error introduced by smoothing is also well below its worst-case bound $\delta_{\tmop{smooth}}$ (dotted, bottom).
}%
\label{fig:fourcspeedpottsnesterov}
\end{figure}

There is an interesting extreme case where the Nesterov method seems to come to full strength. Consider the noise-free
``triple point'' inpainting problem (Fig.~\ref{fig:inverse}). The triple junction in the center can only be reconstructed
by the Potts regularizer, as the $\ell^1$ data term has been blanked out around the center. By reversing the sign of the data term, one obtains the ``inverse triple point'' problem, an extreme case that has also been studied in \cite{Chambolle2008} and shown to be an example where the relaxation leads to a strictly nonbinary solution.

On the inverse problem, the Nesterov method catches up and even surpasses FPD. This stands in contrast with the regular triple point problem, where all methods perform as usual. We conjecture that this sudden
strength comes from the inherent averaging over all previous gradients (step \ref{step:nesaverage} in Alg.~\ref{alg:nesterovmain}): in fact, on the inverse problem Douglas-Rachford and FPD display a pronounced oscillation in the
primal and dual objectives, which is accompanied by slow convergence. In contrast, the Nesterov method consistently shows a monotone and smooth convergence.
\begin{figure*}
\centering
{
\setlength{\tabcolsep}{2pt}
\begin{tabular}{m{.15\textwidth}m{.15\textwidth}m{.15\textwidth}m{.48\textwidth}}%
\includegraphics[width=.15\textwidth]{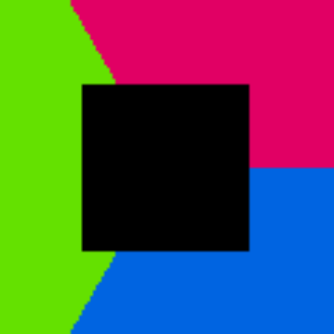}&%
\includegraphics[width=.15\textwidth]{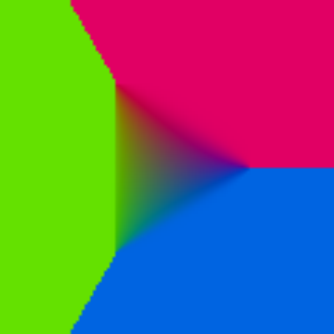}&%
\includegraphics[width=.15\textwidth]{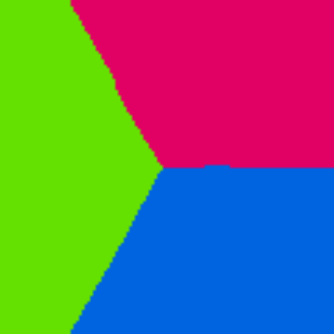}&%
\setlength{\unitlength}{.48\textwidth}
\begin{picture}(1,.63571)
\put(0,0){\includegraphics[width=\unitlength]{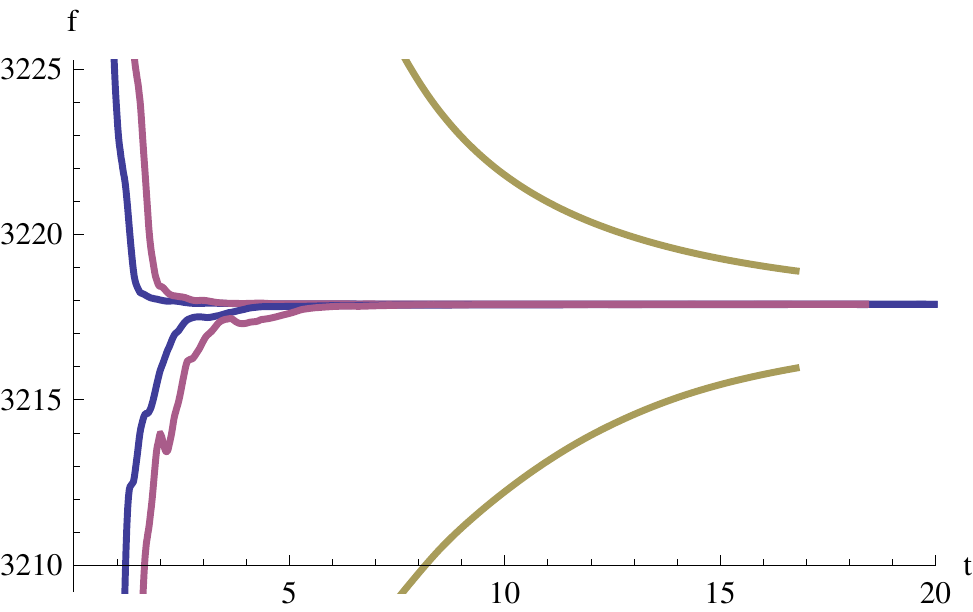}}
\put(.725,.41){\includegraphics[width=0.2\unitlength]{manual_legend_three.pdf}}
\end{picture}\\
\includegraphics[width=.15\textwidth]{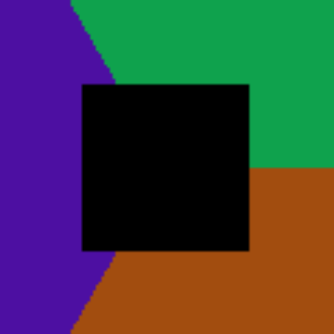}&%
\includegraphics[width=.15\textwidth]{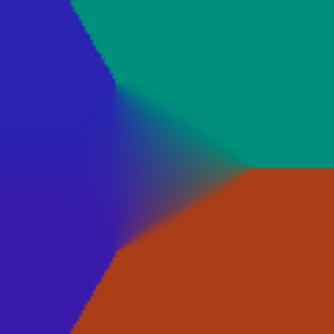}&%
\includegraphics[width=.15\textwidth]{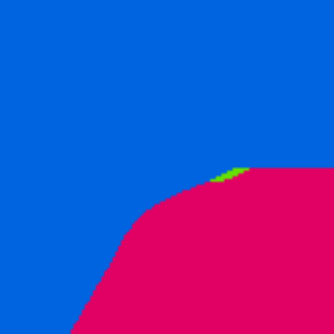}&%
\setlength{\unitlength}{.48\textwidth}
\begin{picture}(1,.63571)
\put(0,0){\includegraphics[width=\unitlength]{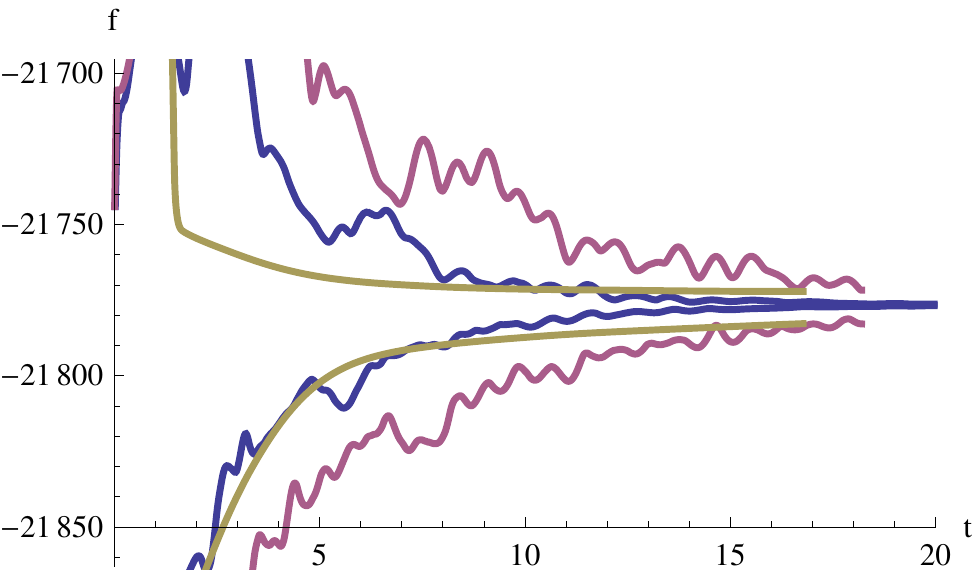}}
\put(.725,.41){\includegraphics[width=0.2\unitlength]{manual_legend_three.pdf}}
\end{picture}
\end{tabular}%
}
\caption{Primal and dual objectives for the triple point (top) and inverse triple point (bottom) inpainting problems.
\textbf{Left to right}: Input image with zeroed-out region around the center; relaxed solution; binarized solution;
primal (upper) and dual (lower) energies vs.~time. The triple junction in the center has to be reconstructed solely by the Potts regularizer. The inverse triple point problem exhibits a strictly nonbinary relaxed solution. For the inverse triple point, Douglas-Rachford (bottom) and FPD (center) show an oscillatory behavior which slows down convergence. The Nesterov approach (top) does not suffer from oscillation due to the inherent
averaging, and surpasses FPD on the inverse problem.
}%
\label{fig:inverse}
\end{figure*}


\subsection{Performance for the Envelope Relaxation}

\begin{figure*}
\centering
\setlength{\unitlength}{.48\columnwidth}
\begin{picture}(1,.63571)
\put(0,0){\includegraphics[width=\unitlength]{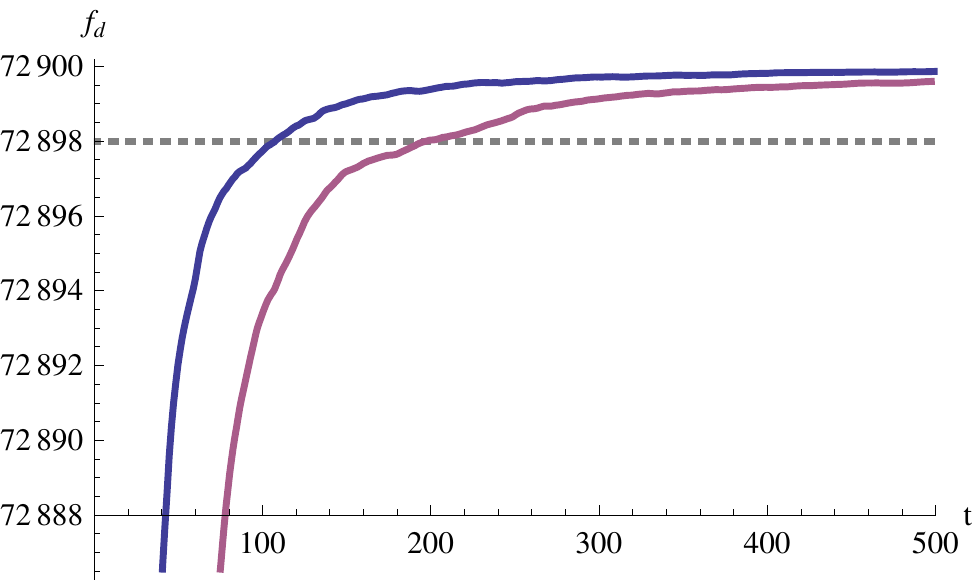}}
\put(.725,.15){\includegraphics[width=0.2\unitlength]{manual_legend_two.pdf}}
\end{picture}
\setlength{\unitlength}{.48\columnwidth}
\begin{picture}(1,.63571)
\put(0,0){\includegraphics[width=\unitlength]{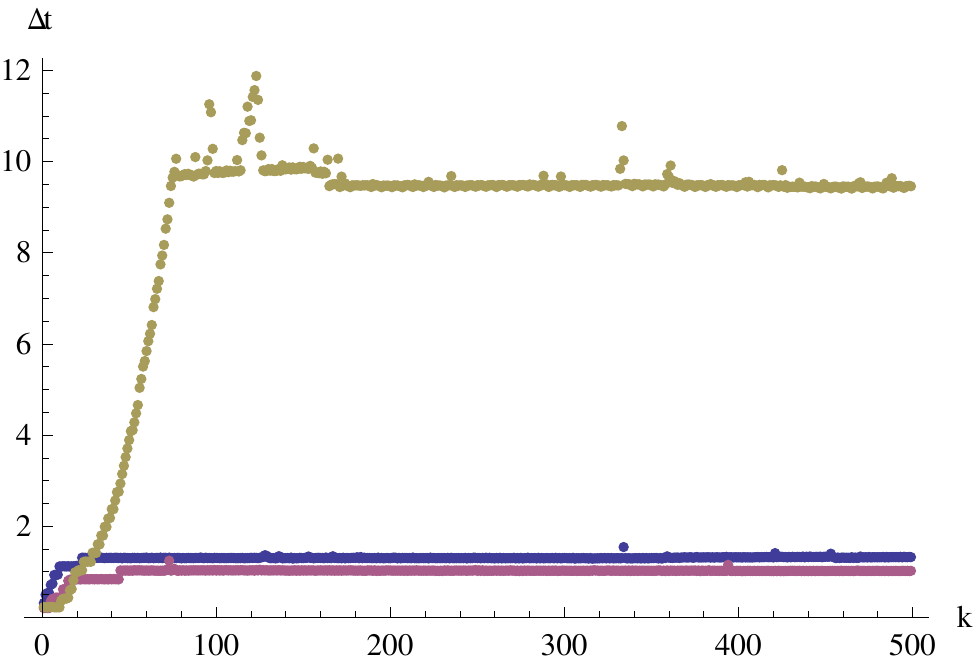}}
\put(.725,.15){\includegraphics[width=0.2\unitlength]{manual_legend_three.pdf}}
\end{picture}
\caption{Performance on the ``four colors'' image with Potts interaction potential and the envelope regularizer. \textbf{Left:} Dual objectives of for Douglas-Rachford (top) and FPD (bottom) vs. time. The reduced iteration count of the Douglas-Rachford method becomes more apparent in the time plot as the time per iteration is now dominated by the projection rather than the DCT. \textbf{Right:} Time per iteration for Nesterov (top), Douglas-Rachford (center) and FPD (bottom). The Nesterov method fails
to converge as it accumulates errors from the approximate projections, which in turn leads to slower and more inexact projections.
}%
\label{fig:fourcenvelope}
\end{figure*}

Undoubtedly, the difficulty when using the envelope based regularizer
comes from the slow and inexact projection steps which have to be approximated iteratively. Therefore we re-evaluated the ``four colors'' benchmark image with the envelope regularizer. The iterative Dykstra projection (Alg.~\ref{alg:dykstra}) was stopped when the iterates differed by at most $\delta = 10^{-2}$, with an additional limit of $50$ iterations. While the gap cannot be computed in this case, the dual objective can still be evaluated and provides an indicator for the convergence speed.

We found that in comparison to the Euclidean metric regularizer from the previous examples, the margin between FPD and Douglas-Rachford increases significantly. This is consistent with the remarks in Sect.~\ref{sec:relperformance}: the lower iteration count of the Douglas-Rachford method becomes more important, as the projections dominate the per-iteration runtime (Fig.~\ref{fig:fourcenvelope}).

Surprisingly the Nesterov method did not converge at all. On inspecting the per-iteration runtime,
we found that after the first few outer iterations,
the iterative projections became very slow and eventually exceeded the limit of $50$ iterations with $\delta$ remaining between $2$ and $5$. In contrast, $20$ Dykstra iterations were usually sufficient to obtain $\delta = 10^{-9}$ (Douglas-Rachford) resp.~$\delta = 10^{-11}$ (FPD).

We again attribute this to the averaging property of the Nesterov method: as it accumulates the results of the previous
projections, errors from the inexact projections build up. This is accelerated by the dual variables quickly becoming
infeasible with increasing distance to the dual feasible set, which in turn puts higher demands on the iterative projections.
Douglas-Rachford and FPD did not display this behavior and consistently required $5$ to $6$ Dykstra iterations from the first
to the last iteration.


\subsection{Tightness of the Relaxations}\label{sec:tightness}

Besides the properties of the optimization methods, it is interesting to study the effect 
of the relaxation -- i.e. Euclidean metric or envelope type -- on the relaxed and binarized solutions.

To get an insight into the tightness of the relaxations, we used the Douglas-Rachford method to repeat
the ``triple point'' inpainting experiment in \cite{Chambolle2008} with both relaxations (Fig.~\ref{fig:discreteness}).
Despite the inaccuracies in the projections, the envelope regularizer generates
a nearly binary solution: $97.6\%$ of all pixels were assigned
``almost binary'' labels with an $\ell^{\infty}$ distance of less than $0.05$ to one of the
unit vectors $\{e^1,\ldots,e^l\}$. For the Euclidean metric method, this constraint was only satisfied
at $88.6\%$ of the pixels. The result for the envelope relaxation is very close to the sharp triple
junction one would expect from the continuous formulation, and shows that the envelope relaxation is
tighter than the Euclidean metric method.

However, after binarization both approaches generate almost identical discrete results. The Euclidean
metric method was more than four times faster, with $41.1$ seconds per $1000$ iterations vs. $172.16$
seconds for the envelope relaxation, which required $8$--$11$ Dykstra steps per outer iteration.
\begin{figure}
\centering
$\begin{array}{ccc}
\includegraphics[width=.20\columnwidth]{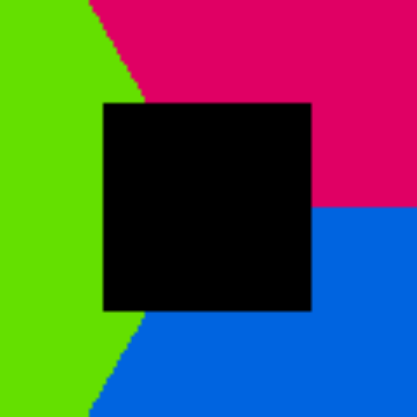} &
\includegraphics[width=.20\columnwidth]{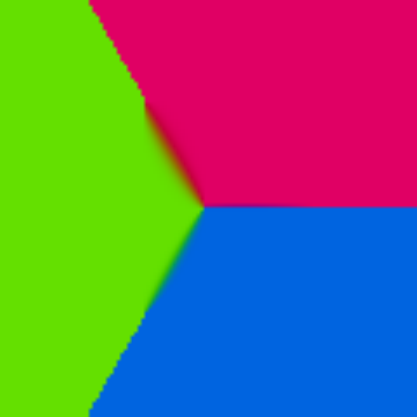} &
\includegraphics[width=.20\columnwidth]{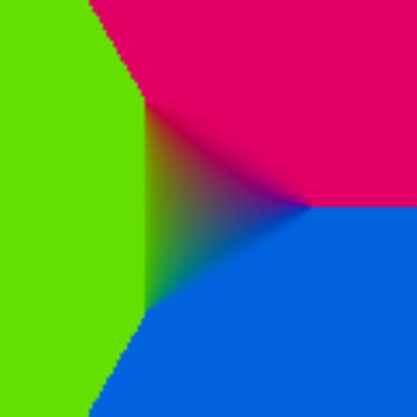}\\
\includegraphics[width=.20\columnwidth]{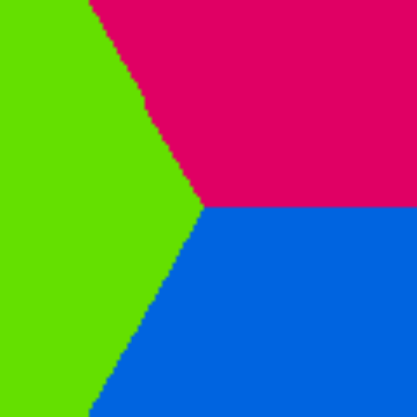} &
\includegraphics[width=.20\columnwidth]{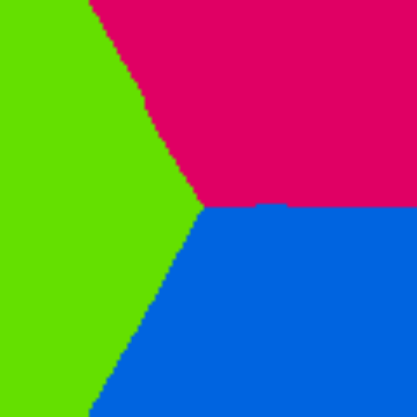} &
\includegraphics[width=.20\columnwidth]{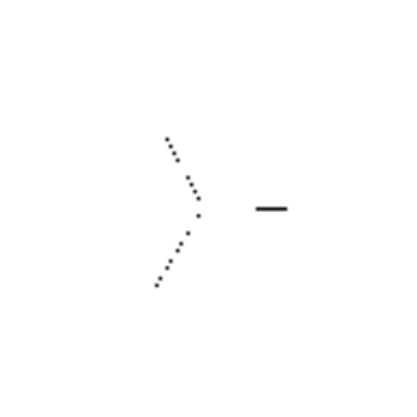}
\end{array}$
\caption{Tightness of the relaxation. \textbf{Top row:} In the input (left), the data term was blanked out in
a quadratic region. All structure within the region is generated purely by the regularizer with a standard Potts interface potential. The envelope relaxation is tighter and generates a much more ``binary'' solution (center) than the Euclidean metric method (right). \textbf{Bottom row:} After binarization of the relaxed solutions, the envelope (left) and Euclidean metric (center) methods generate essentially the same solution, as can be seen in the difference image (right). The Euclidean
metric method performed more than four times faster due to the inexpensive projections.
}%
\label{fig:discreteness}%
\end{figure}

While the triple point is a problem specifically designed to challenge the regularizer, real-world images
usually contain more structure as well as noise, while the data term is available for most or all of the image.
To see if the above results also hold under these conditions, we repeated the previous experiment with the
``sailing'' image and four classes (Fig.~\ref{fig:discretenesssailing}). The improved tightness of the envelope
relaxation was also noticeable, with $96.2\%$ vs. $90.6\%$ of ``almost binary'' pixels. However,
due to the larger number of labels and the larger image size of $360 \times 240$, runtimes increased to
$4253$ (envelope) vs. $420$ (Euclidean metric) seconds.

The relaxed as well as the binarized solutions show some differences but are hard to distinguish visually. It is difficult to
pinpoint if these differences are caused by the tighter relaxation or by numerical issues:
while the Douglas-Rachford method applied to the Euclidean metric relaxation converged to a final relative gap of $1.5 \cdot 10^{-6}$, no such bound is available to estimate the accuracy of the solution for the envelope relaxation, due to the inexact projections and the intractable primal objective.
\begin{figure}
\centering
$\begin{array}{cc}
\includegraphics[width=.3\columnwidth]{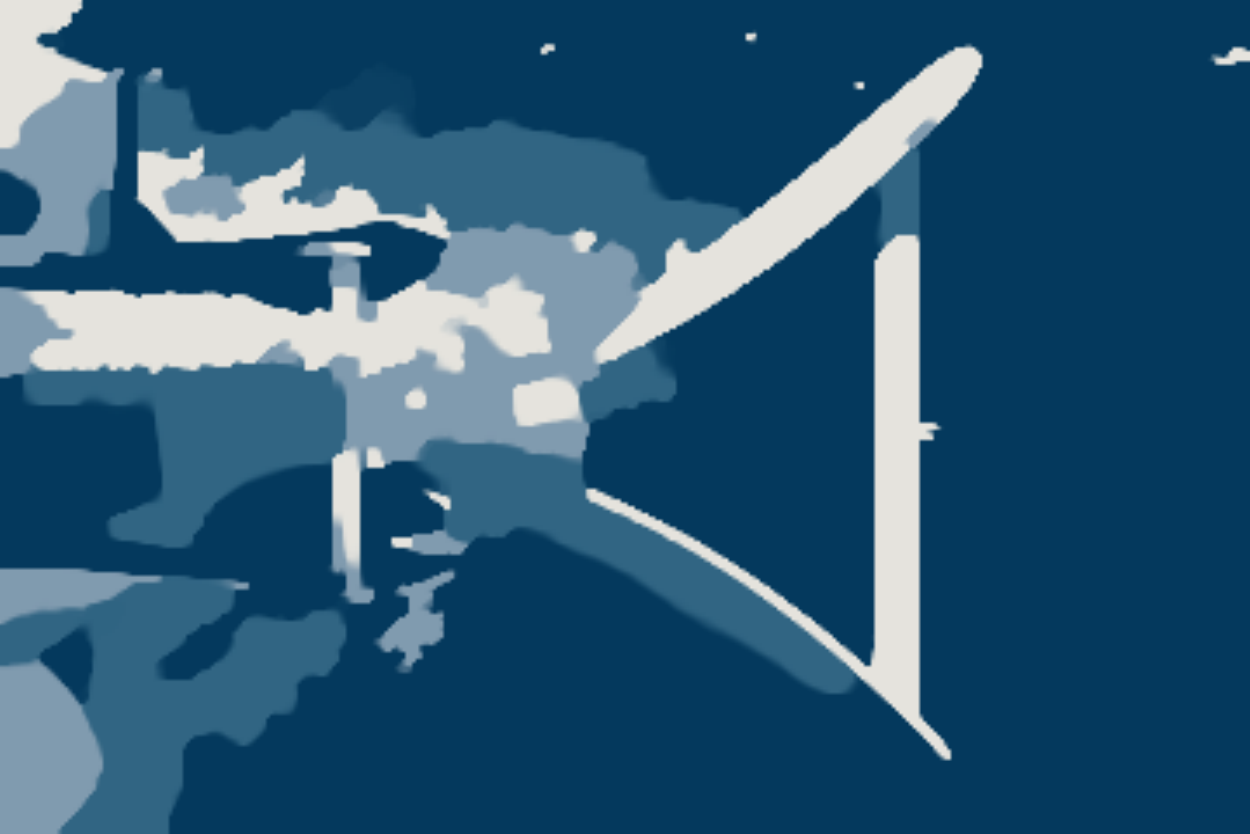} &
\includegraphics[width=.3\columnwidth]{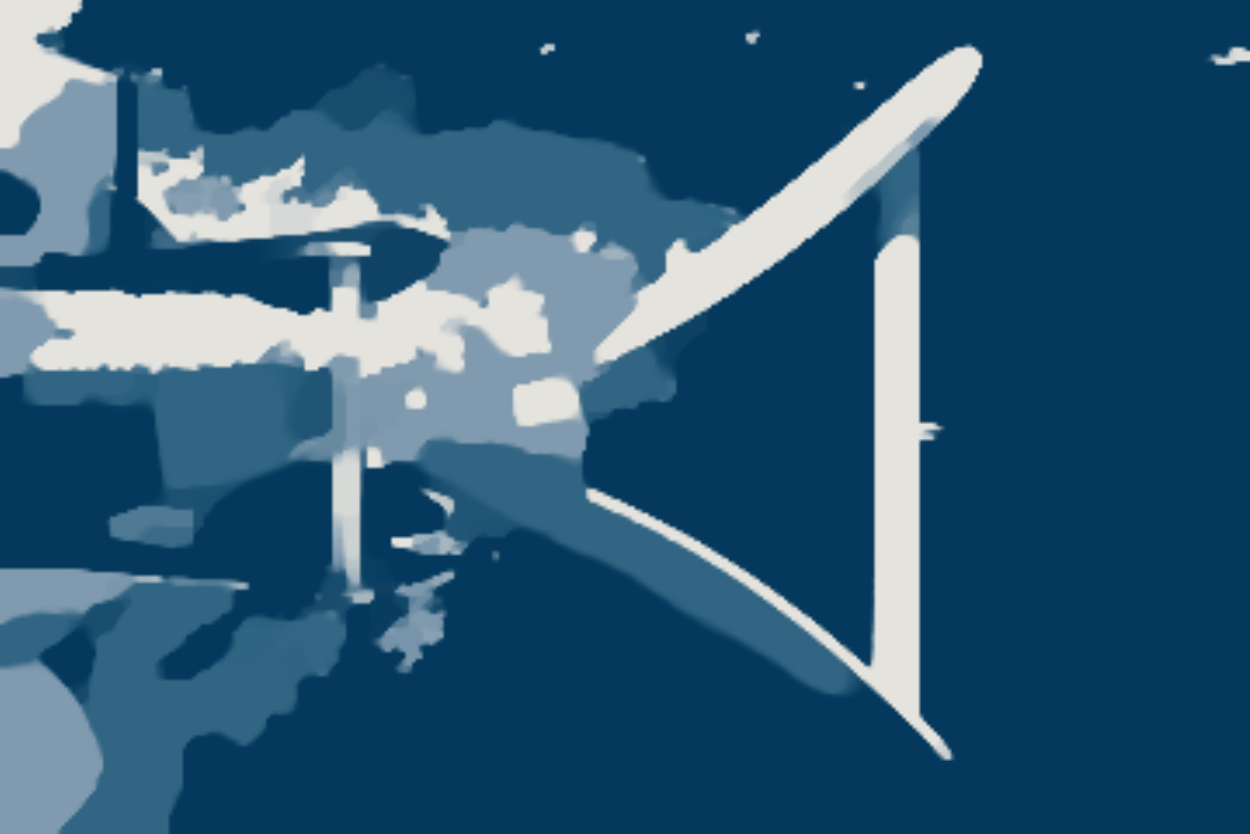}\\
\includegraphics[width=.3\columnwidth]{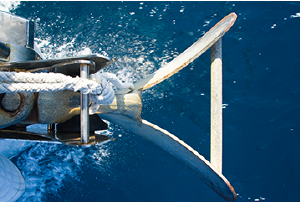} &
\includegraphics[width=.3\columnwidth]{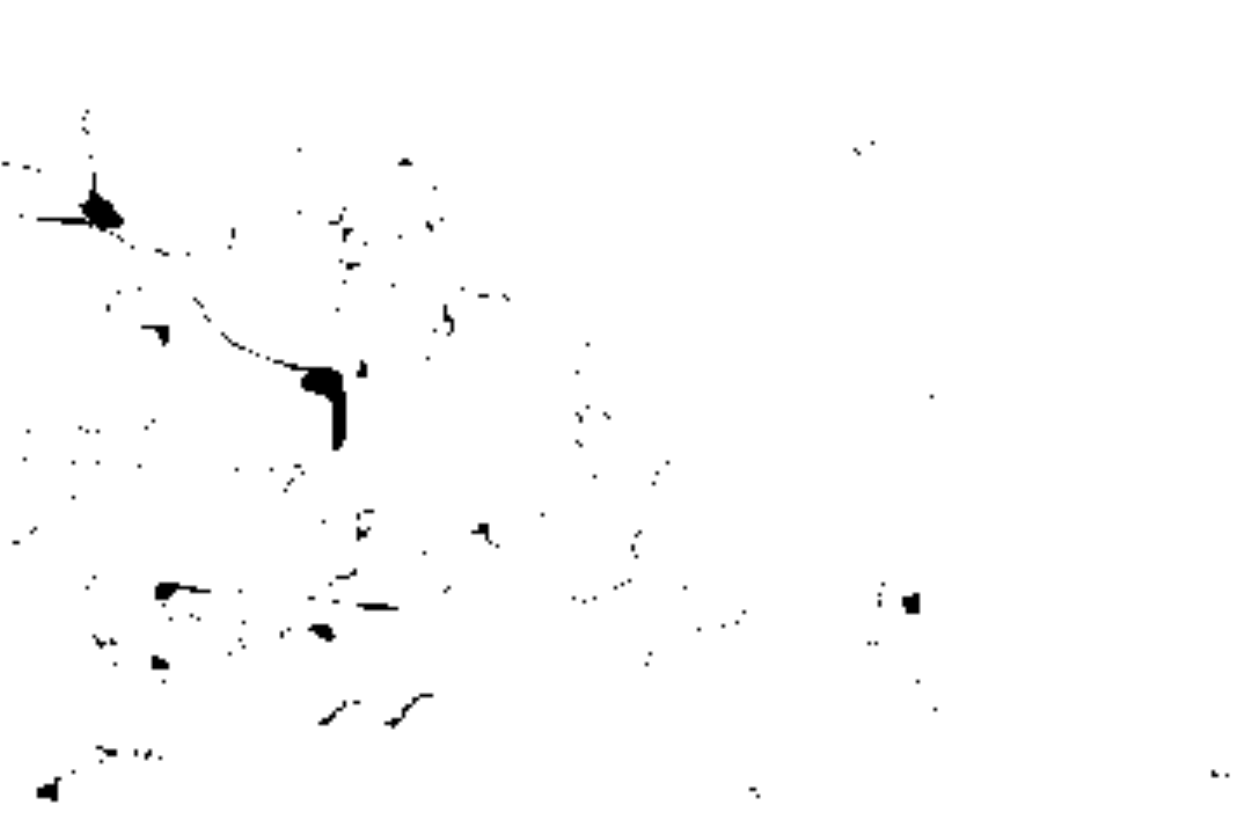}
\end{array}$
\caption{Effect of the choice of relaxation method on the real-world ``sailing'' image (image courtesy of F.~Becker). \textbf{Top row:} Four-class segmentation using envelope (left) and Euclidean metric (right) methods. Shown are the solutions of the relaxed problem. \textbf{Bottom row:} Original image (left); difference image of the discretized solutions (right). While the envelope relaxation
leads to substantially more ``almost discrete'' values in the relaxed solution, it also runs more than $10$ times slower
and does not provide a suboptimality bound. The generated solutions are visually almost identical.
}%
\label{fig:discretenesssailing}%
\end{figure}


\subsection{Binarization and Global Optimality}

As an example for a problem with a large number of classes, we analyzed the ``penguin'' inpainting problem
from~\cite{Szeliski2006}. We chose $64$ labels corresponding to $64$ equally spaced gray values. The input image
contains a region where the image must be inpainted in addition to removing the considerable noise.
Again the data term was generated by the $\ell^1$ distance, which reduces here to the absolute difference of the gray values.
In order to remove noise but not overly penalize hard contrasts, such as between the black wing and the white front,
we chose a regularizer based on the truncated linear potential as introduced in Sect.~\ref{sec:euclmetricmethod}.

Due to the large number of labels, this problem constitutes an example where the Euclidean metric approach
is very useful. As the complexity of the projections for the envelope relaxation grows quadratically with the
number of labels, computation time becomes prohibitively long for a moderate amount of classes.
In contrast, the Euclidean metric method requires considerably less computational effort and still
approximate the potential function to a reasonable accuracy (Fig.~\ref{fig:penguin-approxcutlinear}).

In the practical evaluation, the Douglas-Rachford method converged in $1000$ iterations to a relative gap of $8.3 \cdot 10^{-4}$, and recovered both smooth details near the beak, and hard edges in the inpainting region (Fig.~\ref{fig:penguin-discretization}).
This example also clearly demonstrates the superiority of the improved binarization
scheme proposed in Sect.~\ref{sec:binarization}. As opposed to the first-max scheme, the improved scheme generated considerably
less noise. The energy increased only by $2.78\%$ compared to $15.78\%$ for the first-max approach.

The low energy increase is directly related to global optimality for the discrete problem: as the relaxed solution is
provably nearly optimal, we conclude that the energy of the \emph{binarized} solution must lie within
$2.78\%$ of the optimal energy for the original \emph{combinatorial} problem (\ref{eq:problemcomb}). Similar results were obtained for the other images: $5.64\%$ for the ``four colors'' demo, $1.02\%$ for the ``leaf'' image and $0.98\%$ for the ``triple point'' problem.

These numbers indicate that the relaxation seems to be quite tight in many cases, and allows to recover good approximations for the solution of the discrete combinatorial labeling problem by solving the convex relaxed problem.
\begin{figure}[tb]
\centering
$\begin{array}{ccc}
\includegraphics[width=.20\columnwidth]{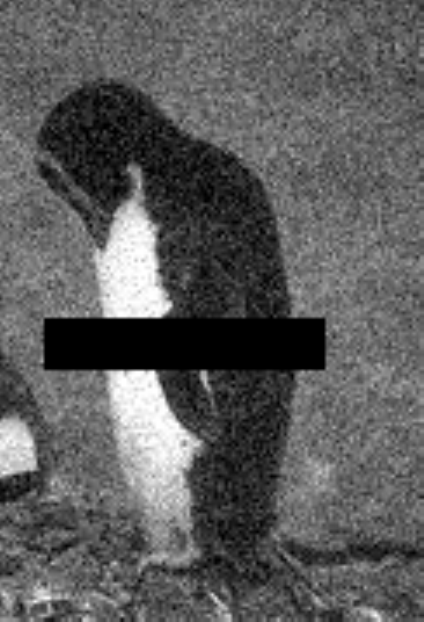} &
\includegraphics[width=.20\columnwidth]{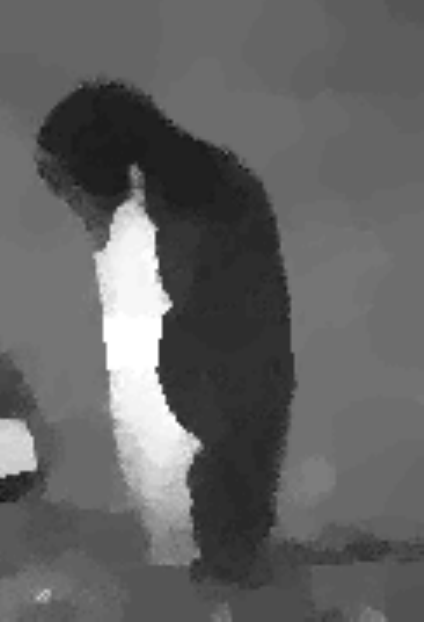} &
\includegraphics[width=.20\columnwidth]{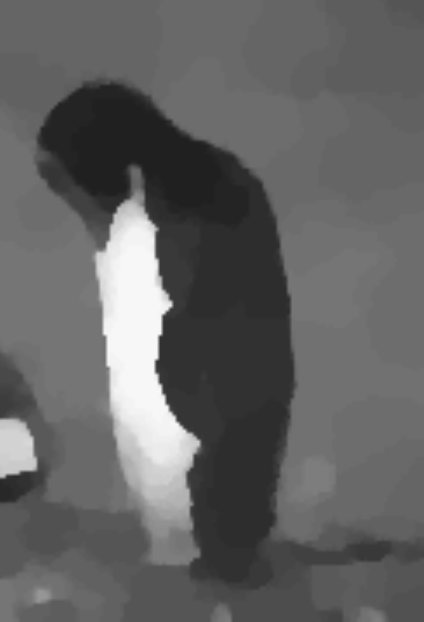} \\
\includegraphics[width=.20\columnwidth]{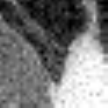} &
\includegraphics[width=.20\columnwidth]{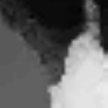} &
\includegraphics[width=.20\columnwidth]{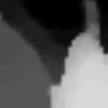} 
\end{array}$
\caption{Denoising/inpainting problem with $64$ classes and the nontrivial truncated linear potential approximated by an Euclidean metric. \textbf{Left to right:} Noisy input image with inpainting region marked black \cite{Szeliski2006}; result with first-max binarization; result with the proposed binarization scheme (\ref{eq:nonuniformbinarization}). The first-max method introduces noticeable noise in the binarization step. The proposed method takes into account the non-uniformity of the used ``cut-linear'' potential (Fig.~\ref{fig:penguin-approxcutlinear}), resulting in a clean labeling and an energy increase of only $2.78\%$ vs. $15.78\%$ for the first-max method. This shows that the obtained solution solves the
originally combinatorial multiclass labeling problem to a suboptimality of $2.78\%$.
}%
\label{fig:penguin-discretization}%
\end{figure}


\section{Conclusion and Further Work}

The present work provides a reference and framework for
continuous multilabeling approaches. The presented algorithms are robust and fast and
are suited for massive parallelization. From the experiments it became clear that
solving the convex relaxed problem allows to recover very good solutions for the
original combinatorial problem in most cases.

The performance evaluations showed that the Douglas-Rachford method
consistently requires about one third of the iterations compared to the
Fast Primal-Dual method. For low regularization and
fast projections, FPD outperforms the Douglas-Rachford method. In all other cases, Douglas-Rachford performs equally or better,
with a speedup of $2$-$3$ if the projections are expensive. Overall, the proposed
Douglas-Rachford method approach appears to be a solid all-round method that also handles extreme cases well.

From the viewpoint of numerics, in our evaluations we did not specifically consider the effect of choosing
different step sizes for the Douglas-Rachford method. Also, it seems as if the smooth
optimization step in the Nesterov method usually performs much better than its theoretical
bound. Adjusting the strategy for choosing the smoothing parameter could yield a faster
overall convergence and possibly render the method competitive. In this regard, it would also
be interesting to include the inexactness of the projections into the convergence analysis.
There are also several theoretical questions left, such as how to include non-metric distances
into the continuous formulation.

In any case we think that the present framework unites continuous and discrete
worlds in an appealing way, and hopefully contributes to reducing the popularity gap
compared to purely grid-based, graph cut methods.

\textbf{Acknowledgments.} The authors would like to thank Simon Setzer
for the stimulating discussion. We are also grateful to an anonymous
reviewer, whose extensive comments greatly helped to improve the presentation of this work.


\section{Appendix}\label{sec:appendix}

\begin{proposition}\label{prop:bounddlocd}
  Let $v = (v^1, \ldots, v^l) \in \mathcal{D}_{\tmop{loc}}^{\pot}$, then
  \begin{equation}
    \|v\| \leqslant \min_i \Bigl( \sum_j \pot (i, j)^2 \Bigr)^{\frac{1}{2}}\,.
  \end{equation}
\end{proposition}

\begin{proof}[Proof of Prop.~\ref{prop:bounddlocd}.]
From the constraint $\sum_{j=1}^{l} v^j = 0$ in \eqref{eq:dlocd} we deduce, for arbitrary but fixed $i \in \{1, \ldots, l\}$,
\begin{eqnarray}
  \sum_{j=1}^{l} \|v^j \|^2 & \leqslant & \left( \sum_{j=1}^{l} \|v^j \|^2 \right) - 0 + l\|v^i \|^2\nonumber\\
  & = & \left( \sum_{j=1}^{l} \|v^j \|^2 \right) - 2 \langle v^i, \sum_{j=1}^{l} v^j \rangle + l\|v^i \|^2\nonumber\\  
  & = & \sum_{j=1}^{l} \left( \|v^j \|^2 - 2 \langle v^i, v^j \rangle +\|v^i \|^2 \right)\nonumber\\
  & = & \sum_{j=1}^{l} \|v^j - v^i \|^2 \leqslant \sum_{j=1}^{l} \pot (i, j)^2 .
\end{eqnarray}
As $i$ was arbitrary this proves the assertion.
\end{proof}

\begin{proof}[Proof of Eq.~\eqref{eq:dctainversion}]
We mainly use the fact that $(P \otimes Q)(R \otimes S) = (P R)\otimes(Q S)$ for matrices $P,Q,R,S$ with compatible dimensions:
\begin{eqnarray}
(I + L^{\top} L)^{- 1} & = & \left( I + \left( A \otimes \tmop{grad}
  \right)^{\top}  \left( A \otimes \tmop{grad} \right) \right)^{- 1}\\
  & = & \left( I + \left( A^{\top} A \right) \otimes \left(
  \tmop{grad}^{\top} \tmop{grad} \right) \right)^{- 1}\\
  & = & \left( I + \left( V^{- 1} \tmop{diag} (a) V \right) \otimes \left(
  B^{- 1} \tmop{diag} (c) B \right) \right)^{- 1}\\
  & = & \left( I + \left( V^{- 1} \otimes B^{- 1} \right) \left( \tmop{diag}
  (a) \otimes \tmop{diag} (c) \right)  \left( V \otimes B \right) \right)^{-
  1}\\
  & = & \left( \left( V^{- 1} \otimes B^{- 1} \right) \left( I + \tmop{diag}
  (a) \otimes \tmop{diag} (c) \right)  \left( V \otimes B \right) \right)^{-
  1}\\
  & = & \left( V^{- 1} \otimes B^{- 1} \right) \left( I + \tmop{diag} (a)
  \otimes \tmop{diag} (c) \right)^{- 1}  \left( V \otimes B \right).
\end{eqnarray}  
Using $V^{-1}=V^{\top}$, \eqref{eq:dctainversion} follows.
\end{proof}

\begin{proof}[Proof of Prop.~\ref{prop:dr-convergence}]
The idea of the proof is to show that the sequence $(w''^{(k)})$ is exactly the
minimizing sequence produced by the algorithm applied to the dual problem (\ref{eq:primaldualobj}),
with step size $1/\tau$. Thus, if the dual algorithm converges, $(w''^{(k)})$ converges to
the solution of the dual problem, which proves the assertion.
To show the equivalency, first note that the formulation of the primal problem from \eqref{eq:discretemodelx},
\begin{eqnarray}
  & & \min_{u \in \mathcal{C}} \max_{v \in \mathcal{D}}
  \langle u, s \rangle + \langle Lu, v \rangle - \langle b, v \rangle
\end{eqnarray}
already covers the dual problem
\begin{eqnarray}
  & & \max_{v \in \mathcal{D}} \min_{u \in \mathcal{C}}
  \langle u, s \rangle + \langle Lu, v \rangle - \langle b, v \rangle\nonumber\\
  & = & \min_{v \in \mathcal{D}} \max_{u \in \mathcal{C}}
  \langle b, v \rangle + \langle u, - L^{\top} v \rangle - \langle u, s \rangle
\end{eqnarray}
by the substitutions
\begin{eqnarray}
  & & v \leftrightarrow u, \; \mathcal{C} \leftrightarrow \mathcal{D}\,, \; b
  \leftrightarrow s\,, \; L \leftrightarrow - L^{\top}\,.
\end{eqnarray}
The dual problem can thus be solved by applying the above substitutions to Alg.~\ref{alg:dr-primal}.
Additionally substituting $w \leftrightarrow z$ in Alg.~\ref{alg:dr-primal} in order to avoid confusion
with iterates of the primal method, leads to the dual algorithm, Alg.~\ref{alg:dr-dual}.
\begin{algorithm}[tb]
\caption{Dual Douglas-Rachford for Multi-Class Labeling} \label{alg:dr-dual}
\begin{algorithmic}[1]
\STATE Choose $\bar{v}^{(0)} \in \mathbbm{R}^{n \times d \times l}, \bar{z}^{(0)} \in \mathbbm{R}^{n \times l}$.
\STATE Choose $\tau_D > 0$.
\STATE $k \leftarrow 0$.
\WHILE{(not converged)}
  \STATE $v^{(k)} \leftarrow \Pi_{\mathcal{D}} \left( \bar{v}^{(k)} - \tau_D b \right)$.
  \STATE $z''^{(k)} \leftarrow \Pi_{\mathcal{C}} \left(\frac{1}{\tau_D} (\bar{z}^{(k)} - s) \right)$.

  \STATE $v'^{(k)} \leftarrow (I + L L^{\top})^{- 1}  \left( (2 v^{(k)} - \bar{v}^{(k)}) + (-L) (\bar{z}^{(k)} - 2 \tau_D z''^{(k)}) \right)$.
  \STATE $z'^{(k)} \leftarrow (-L^{\top}) v'^{(k)}$. 
  
  \STATE $\bar{v}^{(k + 1)} \leftarrow \bar{v}^{(k)} + v'^{(k)} - v^{(k)}$.
  \STATE $\bar{z}^{(k + 1)} \leftarrow z'^{(k)} + \tau_D z''^{(k)}$.
	\STATE $k \leftarrow k + 1$.
\ENDWHILE
\end{algorithmic}
\end{algorithm} 
  We first show convergence of Alg.~\ref{alg:dr-primal} and Alg.~\ref{alg:dr-dual}. By construction, these amount to applying Douglas-Rachford splitting to the primal resp. dual formulations
\begin{eqnarray}
& & \min_{u, w} \underbrace{\delta_{L u = w} (u, w)}_{=: h_1(u,w)} +
                \underbrace{\langle u, s \rangle +
                \delta_{\mathcal{C}} (u) +
                \sigma_{\mathcal{D}} (w - b)}_{=: h_2(u,w)}\,,\\
& & \min_{v, z} \underbrace{\delta_{- L^{\top} v = z} (v, z)}_{=: h_{D,1}(v,z)} +
                \underbrace{\langle v, b \rangle +
                \delta_{\mathcal{D}} (v) +
                \sigma_{\mathcal{C}} (z - s)}_{=: h_{D,2}(v,z)}\,.
\end{eqnarray} 
As both parts of the objectives
  are proper, convex and lsc, it suffices to show additivity of
  the subdifferentials {\cite[Cor.~10.9]{Rockafellar2004}} {\cite[Thm.~3.15]{Eckstein1989}} {\cite[Prop.~3.20, Prop.~3.19]{Eckstein1989}} {\cite{Eckstein1992}}. Due to the boundedness of $\mathcal{C}$ and $\mathcal{D}$,
  \begin{eqnarray}
    \tmop{ri} \left( \tmop{dom} h_2 \right) \cap \tmop{ri} \left( \tmop{dom} h_1 \right)
    & = & \left( \tmop{ri} \left( \mathcal{C} \right) \times \mathbbm{R}^{n d l} \right) \cap \{Lu = w\}\nonumber\\
    & = & \{(u, Lu) | u \in \tmop{ri} (\mathcal{C})\}\,,\\
    \tmop{ri} \left( \tmop{dom} h_{D,2} \right) \cap \tmop{ri} \left( \tmop{dom} h_{D,1} \right)
    & = &\left( \tmop{ri} (\mathcal{D}) \times \mathbbm{R}^{n l} \right) \cap \{- L^{\top} v = z\}\nonumber\\
    & = &\{(v, - L^{\top} v) | v \in \tmop{ri} (\mathcal{D})\}\,.
  \end{eqnarray}
  Both of these sets are nonempty as $\tmop{ri} (\mathcal{C}) \neq \emptyset \neq \tmop{ri}(\mathcal{D})$.
  This implies additivity of the subdifferentials for the proposed objective \cite[Cor.~10.9]{Rockafellar2004}
  and thus convergence (in the iterates as well as the objective) of the tight Douglas-Rachford iteration (cf. \cite[Thm.~3.15]{Eckstein1989}).

  We will now show that the primal and dual algorithms essentially generate the
  same iterates, i.e. from $\tau \assign \tau_P = 1/\tau_D$, $\bar{u}^{(k)} = \tau \bar{z}^{(k)}$
  and $\bar{w}^{(k)} = \tau \bar{v}^{(k)}$, it follows that $\bar{u}^{(k + 1)} = \tau \bar{z}^{(k + 1)}$,
  $\bar{w}^{(k + 1)} = \tau \bar{v}^{(k + 1)}$ and $u^{(k)} = z''^{(k)}$, $v^{(k)} = w''^{(k)}$.
  The last two equalities follow immediately from the previous ones by definition of the algorithms. Furthermore,
  \begin{eqnarray}
    \bar{v}^{(k + 1)} & = & \bar{v}^{(k)} + v'^{(k)} - v^{(k)} \nonumber\\
    & = & \bar{v}^{(k)} + ( I + LL^{\top} )^{- 1}\nonumber\\
    & & \cdot ( ( 2 v^{(k)} - \bar{v}^{(k)} ) + ( - L )  ( \bar{z}^{(k)} - 2 \tau^{-1} z''^{(k)} ) ) - v^{(k)}\\
    & = & \tau^{-1} \bar{w}^{(k)} + ( I + LL^{\top} )^{- 1}  ( ( 2 w''^{(k)} - \tau^{-1} \bar{w}^{(k)} )\nonumber\\
    & & + ( - L )  ( \tau^{-1} \bar{u}^{(k)} - 2 \tau^{-1} u^{(k)} ) ) - w''^{(k)}\\
    & = & \tau^{-1} \bar{w}^{(k)} + ( I + LL^{\top} )^{- 1} ( 2 w''^{(k)} - \tau^{-1} \bar{w}^{(k)} ) - w''^{(k)}\nonumber\\
    & & - ( I + LL^{\top} )^{- 1} L ( \tau^{-1} \bar{u}^{(k)} - 2 \tau^{-1} u^{(k)} )\,.
  \end{eqnarray}  
  By the Woodbury identity, $(I + LL^{\top})^{- 1} = I - L (I + L^{\top} L)^{- 1} L^{\top}$ and in particular
  $(I + LL^{\top})^{-1} L = L (I + L^{\top}L)^{-1}$,
  therefore  
  \begin{eqnarray}
    \bar{v}^{(k + 1)} & = & \tau^{-1} \bar{w}^{(k)} + 2 w''^{(k)} - \tau^{-1} \bar{w}^{(k)}\nonumber\\
    & & - L ( I + L^{\top} L )^{- 1} L^{\top} ( 2 w''^{(k)} - \tau^{-1} \bar{w}^{(k)} ) - w''^{(k)}\nonumber\\
    & & - L ( I + L^{\top} L )^{- 1}  ( \tau^{-1} \bar{u}^{(k)} - 2 \tau^{-1} u^{(k)} ) \\
    & = & w''^{(k)} - L ( I + L^{\top} L )^{- 1}\nonumber\\
    & & \cdot( L^{\top} ( 2 w''^{(k)} - \tau^{-1} \bar{w}^{(k)} ) + ( \tau^{-1} \bar{u}^{(k)} - 2 \tau^{-1} u^{(k)} ) )\\
    & = & \tau^{-1} L ( I + L^{\top} L )^{- 1}  \nonumber\\
    & & \cdot( ( 2 u^{(k)} - \bar{u}^{(k)} ) + L^{\top} ( \bar{w}^{(k)} - 2 \tau w''^{(k)} ) ) + w''^{(k)}\\
    & = & \tau^{-1}(Lu'^{(k)} + \tau w''^{(k)}) = \tau^{-1} \bar{w}^{(k + 1)}\,.
  \end{eqnarray}
  By primal-dual symmetry, the same proof shows that $\bar{u}^{(k + 1)} = \tau \bar{z}^{(k + 1)}$.
    To conclude, we have shown that $w''^{(k)} = v^{(k)}$ for $\tau_D = 1/\tau$ and suitable initialization of the dual method.
  As the dual method was shown to converge, $(w''^{(k)})$ must be a maximizing sequence for $f_D$. Together
  with the convergence of $u^{(k)}$ to a minimizer of $f$, this proves the first part of the proposition.
  Equation \eqref{eq:suboptboundinprop} follows, as for any saddle-point $(u^{\ast},v^{\ast})$,
  \begin{eqnarray}
    f_D \left( w''^{(k)} \right) \leqslant f_D (v^{\ast}) & = & f (u^{\ast}) \leqslant f (u^{(k)})
  \end{eqnarray}
  holds, and therefore $f(u^{(k)})-f(u^{\ast}) \geqslant 0$ and $f(u^{\ast}) \geqslant f_D \left( w''^{(k)} \right)$.  
\end{proof}

\bibliographystyle{alpha}
\bibliography{lellmann-schnoerr11}

\end{document}